\documentclass[twoside]{article}
\usepackage[accepted]{aistats2023}
\usepackage{lipsum,booktabs}
\usepackage{amsmath,mathrsfs,amssymb,amsfonts,bm,bbm,enumitem,amsthm,xfrac, xspace}
\usepackage{rotating,verbatim,subfig,floatrow}
\usepackage{pdflscape}
\usepackage{tcolorbox,xcolor,color}
\usepackage{hyperref,url,dsfont,nicefrac,natbib, minitoc}
\usepackage{algorithmic,algorithm}
\usepackage[thinlines]{easytable}
\usepackage[export]{adjustbox}
\hypersetup{
    colorlinks,
    breaklinks,
    linkcolor = blue,
    citecolor = blue,
    urlcolor  = blue,
}

\allowdisplaybreaks

\renewcommand{\tilde}{\widetilde}
\renewcommand{\hat}{\widehat}
\def \C {\mathcal{C}}
\def \F {\mathcal{F}}
\def \N {\mathcal{N}}

\def \O {\mathcal{O}}
\def \Ot {\tilde{\O}}
\def \P {\mathbb{P}}
\def \E {\mathbb{E}}
\def \R {\mathbb{R}}

\def \T {\top}
\def \X {\mathcal{X}}
\def \Xt {\tilde{X}}

\def \x {\mathbf{x}}

\def \DReg {R}
\def \teq {\triangleq}
\def \thetah {\hat{\theta}}
\def \thetat {\tilde{\theta}}
\def \thetab {\bar{\theta}}
\def \Vt {\tilde{V}}
\def \diff {\mathrm{d}}
\def \dmu {\mu^\prime}
\def \ddmu {\mu^{\prime\prime}}
\def \Ht {\tilde{H}}

\def \Var {\textnormal{Var}}
\def \betat {\tilde{\beta}}
\def \betab {\bar{\beta}}
\def \betabr {\breve{\beta}}
\def \tb {\textbf}

\def \LBstatic {\mbox{OFUL}\xspace}
\def \LBweight {\mbox{D-LinUCB}\xspace}
\def \LBrestart {\mbox{RestartUCB}\xspace}
\def \LBwindow {\mbox{SW-UCB}\xspace}
\def \GLBstatic {\mbox{GLM-UCB}\xspace}
\def \GLBweight {\mbox{BVD-GLM-UCB}\xspace}

\def \SCBstatic {\mbox{LogUCB1}\xspace}
\def \LBweightours {\mbox{LB-WeightUCB}\xspace}
\def \GLBweightours {\mbox{GLB-WeightUCB}\xspace}
\def \GLBrestart {\mbox{GLB-RestartUCB}\xspace}
\def \SCBrestart {\mbox{SCB-RestartUCB}\xspace}
\def \SCBweightours {\mbox{SCB-WeightUCB}\xspace}
\def \SCBweightourspw {\mbox{SCB-PW-WeightUCB}\xspace}
\def \MASTER {\mbox{MASTER}\xspace}
\def \LBMASTER {\mbox{MASTER+\LBstatic}\xspace}
\def \GLBMASTER {\mbox{MASTER+\GLBstatic}\xspace}
\def \SCBMASTER {\mbox{MASTER+\SCBstatic}\xspace}

\def \bias {\mathtt{bias~part}\xspace}
\def \variance {\mathtt{variance~part}\xspace}
\newcommand\term[1]{\mathtt{term}\,(\mathtt{#1})}

\def \pathlength {path length\xspace}

\newcommand\given[1][]{\:#1\vert\:}
\newcommand\givenn[1][]{\:#1\middle\vert\:}
{\theoremstyle{definition}
\newtheorem{Assum}{Assumption}
\newtheorem{Remark}{Remark} }

\newtheorem{Lemma}{Lemma}
\newtheorem{Thm}{Theorem}

\usepackage{prettyref}
\newcommand{\pref}[1]{\prettyref{#1}}

\usepackage{mathtools}

\let\norm\undefined 
\newcommand\norm[1]{\left\| #1 \right\|} 
\newcommand\abs[1]{\left| #1 \right|} 
\newcommand\inner[2]{\langle #1, #2 \rangle} 
\newcommand\sbr[1]{\left( #1 \right)} 
\newcommand\mbr[1]{\left[ #1 \right]} 
\newcommand\bbr[1]{\left\{ #1 \right\}}
\DeclareMathOperator*{\argmax}{arg\,max}
\DeclareMathOperator*{\argmin}{arg\,min}

\usepackage{array}
\def\mystrut{\rule[-2ex]{0ex}{5ex}} 
\begin{document}
\setlength{\abovedisplayskip}{5.5pt}
\setlength{\belowdisplayskip}{5.5pt}

\doparttoc
\faketableofcontents

\twocolumn[
\aistatstitle{Revisiting Weighted Strategy for Non-stationary Parametric Bandits}
\aistatsauthor{ Jing Wang, Peng Zhao, Zhi-Hua Zhou }
\aistatsaddress{ National Key Laboratory for Novel Software Technology,\\ Nanjing University, Nanjing 210023, China\\ \{wangjing, zhaop, zhouzh\}@lamda.nju.edu.cn}]

\begin{abstract}
    \vspace{-0.2cm}
    Non-stationary parametric bandits have attracted much attention recently. There are three principled ways to deal with non-stationarity, including sliding-window, weighted, and restart strategies. As many non-stationary environments exhibit gradual drifting patterns, the weighted strategy is commonly adopted in real-world applications. However, previous theoretical studies show that its analysis is more involved and the algorithms are either computationally less efficient or statistically suboptimal. This paper revisits the weighted strategy for non-stationary parametric bandits. In linear bandits (LB), we discover that this undesirable feature is due to an inadequate regret analysis, which results in an overly complex algorithm design. We propose a \emph{refined analysis framework}, which simplifies the derivation and importantly produces a simpler weight-based algorithm that is as efficient as window/restart-based algorithms while retaining the same regret as previous studies. Furthermore, our new framework can be used to improve regret bounds of other parametric bandits, including Generalized Linear Bandits (GLB) and Self-Concordant Bandits (SCB). For example, we develop a simple weighted GLB algorithm with an $\Ot(k_\mu^{\sfrac{5}{4}} c_\mu^{-\sfrac{3}{4}} d^{\sfrac{3}{4}} P_T^{\sfrac{1}{4}}T^{\sfrac{3}{4}})$ regret, improving the $\Ot(k_\mu^{2} c_\mu^{-1}d^{\sfrac{9}{10}} P_T^{\sfrac{1}{5}}T^{\sfrac{4}{5}})$ bound in prior work, where $k_\mu$ and $c_\mu$ characterize the reward model's nonlinearity, $P_T$ measures the non-stationarity, $d$ and $T$ denote the dimension and time horizon. 
    \vspace{-0.2cm}
\end{abstract}
\section{INTRODUCTION}
\label{sec:introduction}

\begin{table*}\footnotesize
    \centering
    \caption{\small{Comparisons of our dynamic regret bounds to previous best-known results for weight-based algorithms, under different non-stationary parametric bandits. Below, $k_\mu/c_\mu$ denotes the degree of non-linearity and becomes $1$ in LB case; $d$ is the dimension, path length $P_T$ and change number $\Gamma_T$ are non-stationarity measures for drifting and piecewise-stationary cases, respectively.} }
    \label{table:results}
    \vspace{-2mm}
\renewcommand*{\arraystretch}{1.25}
    \resizebox{\textwidth}{!}{
    \begin{tabular}{c|r|r}
        \hline

        \hline
        \tb{Parametric Bandit Models}   &\multicolumn{1}{c|}{\tb{Previous Work}} &\multicolumn{1}{c}{\tb{Our Results}}\\
        \hline 
        \mystrut Drifting LB          &$\Ot\big(d^{\sfrac{7}{8}} P_T^{\sfrac{1}{4}} T^{\sfrac{3}{4}}\big)$~\citep{NIPS'19:weighted-LB}  &$\Ot\big({d^{\sfrac{3}{4}} P_T^{\sfrac{1}{4}} T^{\sfrac{3}{4}} }\big)$~[\pref{thm:LB-regret}] \\\hline
        \mystrut Drifting GLB              &$\Ot\Big(\frac{k_\mu^2}{c_\mu}d^{\sfrac{9}{10}} P_T^{\sfrac{1}{5}}T^{\sfrac{4}{5}}\Big)$~\citep{arXiv'21:faury-driftingGLB} &$\Ot\Big(\frac{k_\mu^{\sfrac{5}{4}}}{c_\mu^{\sfrac{3}{4}}}d^{\sfrac{3}{4}} P_T^{\sfrac{1}{4}}T^{\sfrac{3}{4}}\Big)$~[\pref{thm:GLB-regret}]\\\hline      
        \mystrut Drifting SCB              &$\Ot\Big(\frac{k_\mu^2}{c_\mu}d^{\sfrac{9}{10}} P_T^{\sfrac{1}{5}}T^{\sfrac{4}{5}}\Big)$~\citep{arXiv'21:faury-driftingGLB} &$\Ot\Big(\frac{k_\mu^{\sfrac{5}{4}}}{c_\mu^{\sfrac{1}{2}}}d^{\sfrac{3}{4}} P_T^{\sfrac{1}{4}}T^{\sfrac{3}{4}}\Big)$~[\pref{thm:SCB-regret}]\\\hline        
        \mystrut Piecewise Stationary SCB              &$\Ot\Big(\frac{1}{c_\mu^{\sfrac{1}{3}}}d^{\sfrac{2}{3}}\Gamma_T^{\sfrac{1}{3}}T^{\sfrac{2}{3}}\Big)$~\citep{AISTATS'21:SCB-forgetting} &$\Ot\big(d^{\sfrac{2}{3}}\Gamma_T^{\sfrac{1}{3}}T^{\sfrac{2}{3}}\big)$~[\pref{thm:SCB-PW-regret}]\\
        \hline

        \hline
    \end{tabular}}
\end{table*}
\vspace{-0.2cm}
Non-stationary parametric bandits model the sequential decision-making problems where the reward distributions of each arm are structured with an unknown \emph{time-varying} parameter, which have been extensively studied in recent years~\citep{AISTATS'19:window-LB, NIPS'19:weighted-LB, AISTATS'20:restart,arXiv'20:NS-GLB, arXiv'21:faury-driftingGLB, AISTATS'21:SCB-forgetting,COLT'21:black-box, AISTATS'22:weighted-GPB,arXiv'22:VanRoy} due to its significance in many real-world non-stationary online applications such as recommendation systems~\citep{MLJ'21:TSnonstationary,AISTATS'21:change-preference}, as well as tight connection with theoretical foundation of reinforcement learning~\citep{COLT'20:ChiJin,arXiv'20:Touati}. 

Linear Bandits (LB) is a fundamental instance of parametric bandits, where the expected reward for pulling a certain arm at time $t$ is the inner product between the arm's feature vector $X_t$ and an unknown parameter $\theta_t$, namely, $\E[r_t \given X_t] = X_t^\T \theta_t$. Moreover, Generalized Linear Bandits (GLB) is introduced as a generalization of LB to model a broader range of reward functions such as binary rewards, where the expected reward obeys a generalized linear model as $\E[r_t \given X_t] = \mu(X_t^\T \theta_t)$ with $\mu(\cdot)$ being an inverse link function. Notably, the non-stationary models allow the parameter $\theta_t$ in the above models to be time-varying. There are two typical non-stationarity measures to quantify the intensity of parameter changes: (i) in gradually drifting cases, \pathlength $P_T = \sum_{t=2}^{T} \norm{\theta_{t-1} - \theta_{t}}_2$ is used to measure the cumulative variations of the underlying parameters; and (ii) in piecewise-stationary cases, $\Gamma_T$ denotes the number of parameter changes in $T$ rounds. 

To deal with non-stationarity, there are three principled ways: sliding-window, weighted, and restart strategies. For the sliding-window strategy, the learner maintains a time window that contains the most recent observed data to discard the outdated data. For the weighted strategy, the learner puts more weight on the most recent data and less weight on the old data to gradually forget the outdated data. For the restart strategy, the learner restarts the algorithm according to a certain period to discard the outdated data. The currently best-known result for non-stationary (generalized) linear bandits is by~\citet{COLT'21:black-box}, who developed an optimal algorithm consisting of a non-stationarity detector and a base algorithm that performs well in near-stationary environments. Whenever the detector examines that the non-stationarity exceeds a certain limit, the algorithm will \emph{restart} itself to resist the non-stationarity. In this sense, their algorithm can be regarded as an \emph{adaptive restart-based algorithm}. Building on the \LBrestart algorithm~\citep{AISTATS'20:restart} and a carefully designed non-stationarity detector with multi-scale explorations, their algorithm can achieve an $\Ot(\min\{\sqrt{\Gamma_T T}, P_T^{\sfrac{1}{3}} T^{\sfrac{2}{3}}\})$ optimal dynamic regret for both LB and GLB.

In real-world scenarios, the distribution change of environments often exhibits gradually drifting patterns~\citep{COLT'10:concept-drift,COLT'13:Chiang,ACM'14:survey-concept-drift}, in such cases, a soft weighted strategy can be (empirically) more advantageous than a hard restart strategy to deal with the non-stationarity, as can be observed in bandits learning~\citep{NIPS'19:weighted-LB,AISTATS'20:restart,AISTATS'22:weighted-GPB}, classification with concept drift~\citep{SADM'12:adaptive-forgetting,TKDE'21:DFOP}, and adaptive system identification~\citep{guo1993performance, SP'17:FFRLS}. As a result, it will be highly attractive to design an \emph{adaptive weight-based algorithm} for non-stationary parametric bandits, which imposes weights to discount the importance of past data, and the weights are set adaptively according to environments. Towards this end, we examine existing methods for non-stationary parametric bandits based on the weighted strategy, and (surprisingly) find that current results exhibit \emph{unnatural} gaps compared to the other strategies, such as restart-based algorithms, as well as \emph{unnatural} regret analysis transitions from GLB to LB.

Those unnatural phenomena motivate us to revisit the algorithm design and regret analysis of the weighted strategy for non-stationary parametric bandits~\citep{NIPS'19:weighted-LB,AISTATS'21:SCB-forgetting,arXiv'21:faury-driftingGLB}. Indeed, the key ingredient is the \emph{estimation error} analysis for the weight-based estimator, which is usually decomposed into two parts --- one is the \emph{bias} part due to the parameter drift, and the other is the \emph{variance} part due to the stochastic noise. Generally, the bias part is controlled by non-stationary strategies, and the variance part is handled by carefully designed concentration.~\citet{NIPS'19:weighted-LB} provided the first analysis of the weight-based algorithm for LB, where they introduced a virtual window size in the analysis to control the bias in order to mimic the analysis of sliding-window strategy~\citep{AISTATS'19:window-LB}. Consequently, they have to use \mbox{\emph{different}} local norms to control bias and variance parts, resulting in unexpected inefficiencies and complications. For LB, this leads to an algorithm \LBweight~\citep{NIPS'19:weighted-LB} requiring to maintain an extra covariance matrix, which is less efficient than the window-based and restart-based algorithms~\citep{AISTATS'19:window-LB, AISTATS'20:restart}.

This analysis framework for weighted strategy introduces more severe issues in GLB, due to its more enriched and complicated structure. Specifically,~\citet{arXiv'21:faury-driftingGLB} studied the drifting GLB and designed a highly complex projection operation to control bias and variance parts following the way of~\citet{NIPS'19:weighted-LB} to mimic sliding-window analysis, and finally attained an $\Ot(d^{\sfrac{9}{10}} P_T^{\sfrac{1}{5}}T^{\sfrac{4}{5}})$ dynamic regret. Unfortunately, this \emph{cannot} recover the $\Ot(d^{\sfrac{7}{8}} P_T^{\sfrac{1}{4}}T^{\sfrac{3}{4}})$ bound enjoyed by the weight-based algorithm for drifting LB (a special case of GLB)~\citep{NIPS'19:weighted-LB}. Subsequently,~\citet{AISTATS'21:SCB-forgetting} investigated the non-stationary Self-Concordant Bandits (SCB), a subclass of GLB with many attractive structures. They can only conduct analysis under the piecewise-stationary setting, whereas failed in the more challenging drifting setting, due to technical difficulties in bounding bias using conventional analysis. As such, two open questions are proposed in their papers: (i) how to extend weight-based algorithms to drifting SCB; and (ii) how to replicate recent progress in stationary SCB~\citep{AISTATS'21:optimal-logistic-bandits} to improve dependence on $c_\mu$ in non-stationary SCB.

\tb{Our Results.~~} In this paper, we revisit the weighted strategy for the non-stationary parametric bandits. We discover that the earlier analysis framework may be inappropriate due to mimicking the sliding-window analysis, which demands bounding the bias and variance parts using \emph{different} local norms. We present a \emph{refined analysis framework} for the weighted strategy, in which a new analysis for the bias part is presented such that it is now allowed to use the \emph{same} local norm to analyze both bias and variance parts. This refined analysis framework not only simplifies the regret analysis but also brings many benefits in algorithm designs, including improving efficiency for LB and resolving the projection problem brought by GLB and SCB. Table~\ref{table:results} summarizes our main results compared to earlier best-known results of weight-based algorithms. Specifically, based on our refined analysis framework, we achieve the following results: (i) for LB, our approach only needs to maintain one covariance instead of two and still enjoys the same regret as~\citep{NIPS'19:weighted-LB}; (ii) for GLB, our approach enjoys an $\Ot(k_\mu^{\sfrac{5}{4}} c_\mu^{-\sfrac{3}{4}} d^{\sfrac{3}{4}} P_T^{\sfrac{1}{4}} T^{\sfrac{3}{4}})$ regret bound, whose order of $d$, $P_T$ and $T$ matches that in LB case; and (iii) for SCB, we achieve an $\Ot(k_\mu^{\sfrac{5}{4}} c_\mu^{-\sfrac{1}{2}} d^{\sfrac{3}{4}} P_T^{\sfrac{1}{4}}T^{\sfrac{3}{4}})$ regret bound. In addition, for piecewise stationary SCB, our approach achieves an $\Ot(d^{\sfrac{2}{3}}\Gamma_T^{\sfrac{1}{3}}T^{\sfrac{2}{3}})$ regret bound that can get rid of the influence of $c_\mu^{-1}$, resolving the second open problem asked by~\citet{AISTATS'21:SCB-forgetting}.

\section{RELATED WORK}
\label{sec:related-work}

\tb{Linear Bandits.}
Non-stationary LB problem was first studied by~\citet{AISTATS'19:window-LB}. They established an $\Omega(d^{\sfrac{2}{3}}P_T^{\sfrac{1}{3}}T^{\sfrac{2}{3}})$ minimax lower bound and then proposed \LBwindow algorithm based on the sliding-window strategy. Then \citet{NIPS'19:weighted-LB} proposed the \LBweight algorithm based on weighted strategy and \citet{AISTATS'20:restart} proposed the \LBrestart algorithm based on restart strategy. Note that the three works proved an $\Ot(d^{\sfrac{2}{3}} P_T^{\sfrac{1}{3}}T^{\sfrac{2}{3}})$ regret bound, but there exists a subtle technical gap in the regret analysis as identified by~\citet{arxiv'21:NSLB_revisit_note}. After fixing the technical gap, all three aforementioned algorithms achieve an $\Ot(d^{\sfrac{7}{8}} P_T^{\sfrac{1}{4}}T^{\sfrac{3}{4}})$ regret bound~\citep{arxiv'21:NSLB_revisit_note, arXiv'21:restartucb}. However, to achieve this result, all the three algorithms require the knowledge of the path length $P_T$ as an input at the beginning of algorithmic implementation, which is undesired. To address so,~\citet{AISTATS'19:window-LB} proposed the bandits-over-bandits (BOB) strategy as a meta-algorithm to learn the unknown parameter $P_T$ which can be combined with the above algorithms to remove the requirement of this prior knowledge. Afterward,~\citet{COLT'21:black-box} proposed the \MASTER algorithm with theoretically optimal~$\Ot(\min\{d\sqrt{\Gamma_T T}, d P_T^{\sfrac{1}{3}} T^{\sfrac{2}{3}}\})$ regret bound, also without requiring the non-stationarity level of environments (that is, $\Gamma_T$ and $P_T$) in advance. Most recently, there also has been some new progress in the non-stationary (linear) bandits~\citep{arXiv'22:VanRoy, colt'22:most_significant_arm, arxiv'22:new_look, arXiv'23:LB_memory}. Furthermore, \citet[Remark 4]{Manage'22:window-LB} bypassed the aforementioned technical gap by restarting (efficient) adversarial LB algorithms. However, it is important to note that this only applies to LB and still requires the prior knowledge of $P_T$.

\tb{Generalized Linear Bandits.}
GLB problem was first introduced by~\citet{NIPS10:GLM-infinite}. They proposed \GLBstatic algorithm, achieving an $\Ot(k_\mu c_\mu^{-1}d\sqrt{T})$ regret bound where $k_\mu$, $c_\mu$ are the problem-dependent constants and $k_\mu/c_\mu$ represents the nonlinearity of the generalized linear model.~\citet{arXiv'21:faury-driftingGLB} extended the stationary GLB to the drifting case, and proposed \GLBweight algorithm with $\Ot(k_\mu^{2} c_\mu^{-1}d^{\sfrac{9}{10}} P_T^{\sfrac{1}{5}}T^{\sfrac{4}{5}})$ regret bound.~\citet{ICML'20:logistic-bandits} studied a specific instance of GLB called Logistic Bandits (LogB), they first pointed out that under GLB setting, the problem-dependent constant $1/c_\mu$ could be very large in some cases like LogB, then they proposed the Logistic-UCB-1 algorithm with an $\Ot(c_\mu^{-\sfrac{1}{2}}d\sqrt{T})$ regret bound and the Logistic-UCB-2 algorithm with an $\Ot(d\sqrt{T}+c_\mu^{-1})$ regret bound. Subsequently, \citet{AISTATS'21:optimal-logistic-bandits} established an $\Omega(d\sqrt{\dmu(X_*^\T\theta_*)T})$ regret lower bound for logistic bandits and provided an optimal algorithm OFULog. \citet{AISTATS'21:SCB-forgetting} generalized the logistic bandits to self-concordant bandits and considered the piecewise-stationary case, their algorithm enjoys an $\Ot(c_\mu^{-\sfrac{1}{3}}d^{\sfrac{2}{3}} \Gamma_T^{\sfrac{1}{3}}T^{\sfrac{2}{3}})$ regret bound. To deal with $P_T$-unknown cases,~\citet{arXiv'21:faury-driftingGLB}  proposed a parameter-free algorithm by combining \GLBweight with BOB strategy, but the final result is still suboptimal. Meanwhile, the optimal black-box algorithm~\citep{COLT'21:black-box} can also adaptively restart stationary algorithm \GLBstatic~\citep{NIPS10:GLM-infinite} and achieve an optimal~$\Ot(\min\{k_\mu c_\mu^{-1}\sqrt{\Gamma_T T}, k_\mu^{4/3}c_\mu^{-1} d P_T^{\sfrac{1}{3}} T^{\sfrac{2}{3}}\})$ regret.
\section{LINEAR BANDITS}
\label{sec:LB}
In this section, we first introduce the problem setting of non-stationary LB, and then describe our \LBweightours algorithm and its theoretical guarantee. Our algorithm has the same regret bound as the best-known weight-based algorithm~\citep{NIPS'19:weighted-LB} but is more efficient.
\subsection{Problem Setting}
\label{sec:LB-problem-setting}
At each round $t$, the learner chooses an arm $X_t$ from a feasible set $\X \subseteq \R^d$ and receives a reward $r_t$ such that
\begin{equation}
  \label{eq:LB-model}
  r_t = X_t^\T\theta_t + \eta_t,
\end{equation}
where $\theta_t \in \R^d$ is the unknown time-varying parameter and $\eta_t$ is the $R$-sub-Gaussian noise. The goal of the learner is to minimize the following (pseudo) \emph{dynamic regret}:
\begin{equation}
  \label{eq:LB-regret}
  \DReg_T = \sum_{t=1}^{T} \max_{\x \in \X} \x^{\T}\theta_t - \sum_{t=1}^{T} X_t^{\T}\theta_t,
\end{equation}
which is the cumulative regret against the optimal strategy that has full information of the unknown parameter. Here we consider the drifting case where we use path length $P_T = \sum_{t=2}^{T} \norm{\theta_{t-1} - \theta_{t}}_2$ as the non-stationarity measure.

We work under the following standard boundedness assumption~\citep{NIPS'11:AY-linear-bandits,AISTATS'19:window-LB,NIPS'19:weighted-LB,AISTATS'20:restart}.

\begin{Assum} \label{ass:bounded-norm} 
  \textnormal{
  The feasible set and unknown parameters are assumed to be bounded: $\forall \x \in \X$, $\norm{\x}_2 \leq L$, and $\theta_t \in \Theta$ holds for all $t \in [T]$ where $\Theta \teq \{\theta \mid \norm{\theta}_2 \leq S\}$.}
\end{Assum}
\subsection{Algorithm and Regret Guarantee}
\label{sec:LB-algorithm}
\vspace{-3mm}
We propose the \LBweightours algorithm, which attains the same guarantees as earlier methods while having higher efficiency. We first give the employed estimator and then derive its estimation error upper bound by our refined analysis framework, which is the key for algorithm design and regret analysis. Based on the estimation error bound, we propose our selection criterion and finally give the theoretical guarantee on its dynamic regret.

\tb{Estimator.~} \label{alg:LB-estimator}
We adopt the weighted regularized least square estimator \emph{same} as \LBweight~\citep{NIPS'19:weighted-LB}, the estimator $\thetah_t$ is the solution to the following problem,
\begin{equation}
  \label{eq:LB-estimator}
    \min_{\theta}~\lambda\norm{\theta}_2^2 + \sum_{s = 1}^{t-1}\gamma^{t-s-1}\sbr{X_s^\T \theta -r_s}^2,
\end{equation}
where $\lambda > 0$ is the regularization
coefficient and $\gamma \in (0,1)$ is the discounted factor. Clearly, $\thetah_t$ admits a close-form solution $\thetah_t = V_{t-1}^{-1}(\sum_{s=1}^{t-1}\gamma^{t-s-1}r_sX_s)$, where $V_{t} = \lambda I_d + \sum_{s=1}^{t} \gamma^{t-s} X_s X_s^\T, V_0  = \lambda I_d$ is the covariance matrix. Note that this close-form solution can be further transformed into a recursive formula~\citep[Chapter~10.3]{haykin2002adaptive}. This allows it to be updated online without the need to store historical data, which is another important computational advantage of the weighted strategy compared to the sliding-window strategy.

\tb{Upper Confidence Bounds.~~}\label{alg:LB-UCB}
For estimator~\eqref{eq:LB-estimator}, we provide the following estimation error bound. Notably, this is \emph{different} from the previous result~{\citep[Appenidx B.3, second and third steps in Proof of Theorem 2]{NIPS'19:weighted-LB}}, which is the key component being a more efficient algorithm and will be explained later.
\begin{Lemma}
  \label{lemma:LB-estimation-error}
  For any $\x \in \X$, $\gamma\in(0,1)$ and $\delta \in (0,1)$, with probability at least $1-\delta$, the following holds for all $t \in [T]$
  \begin{align}
      \abs{\x^\T(\thetah_t-\theta_t)} &\leq L^2\sqrt{\frac{d}{\lambda}} \sum_{p=1}^{t-1} \gamma^{\frac{t-1}{2}} \sqrt{\frac{\gamma^{-p}-1}{1-\gamma}} \norm{\theta_p -\theta_{p+1}}_2 \notag \\
      &\qquad+ \beta_{t-1}\norm{\x}_{V_{t-1}^{-1}},\label{eq:LB-estimation-error}
  \end{align}
  where $\beta_t$ is the radius of confidence region set by
  \begin{equation}
    \label{eq:LB-confidence-radius}
    \beta_{t}=\sqrt{\lambda}S+R\sqrt{2\log\frac{1}{\delta}+d\log\sbr{1+\frac{L^2 (1-\gamma^{2t})}{\lambda d(1-\gamma^2)}}}.
  \end{equation}
  \end{Lemma}

  \begin{algorithm}[!t]
    \caption{\LBweightours}
    \label{alg:LB-WeightUCB}
  \begin{algorithmic}[1]
  \REQUIRE time horizon $T$, discounted factor $\gamma$, confidence $\delta$, regularizer $\lambda$, parameters $S$, $L$ and $R$\\
  \STATE Set $V_0 = \lambda I_d$, $\thetah_1 = \mathbf{0}$ and compute $\beta_{0}$ by~\eqref{eq:LB-confidence-radius}
  \FOR{$t = 1,2,...,T$}
    \STATE Select $X_t = \argmax_{\x \in \X} \{ \inner{\x}{\thetah_t} + \beta_{t-1} \norm{\x}_{V_{t-1}^{-1}}\}$
    \STATE Receive the reward $r_t$
    \STATE Update $V_{t} = \gamma V_{t-1} + X_t X_t^\T +(1-\gamma)\lambda I_d$ \label{LB:update}
    \STATE Compute $\thetah_{t+1}$ by~\eqref{eq:LB-estimator} and $\beta_{t}$ by~\eqref{eq:LB-confidence-radius}
  \ENDFOR
  \end{algorithmic}
  \end{algorithm}

  The proof of Lemma~\ref{lemma:LB-estimation-error} is in Appendix~\ref{sec:LB-estimation-error-proof}. Based on Lemma~\ref{lemma:LB-estimation-error}, we can specify the arm selection criterion as
\begin{equation}
  \label{eq:LB-select-criteria}
  \begin{split}
      X_t = {}&\argmax_{\x \in \X} \bbr{ \inner{\x}{\thetah_t} + \beta_{t-1}\norm{\x}_{V_{t-1}^{-1}}}.
  \end{split}
\end{equation}
The overall algorithm is summarized in Algorithm~\ref{alg:LB-WeightUCB}. From the update procedure in Line~\ref{LB:update} of Algorithm~\ref{alg:LB-WeightUCB}, we can observe that our algorithm needs to maintain a \emph{single} covariance matrix $V_{t-1} \in \R^{d\times d}$. By contrast, the selection criterion of algorithm proposed in~\citet{NIPS'19:weighted-LB} is like 
$$X_t = \argmax_{\x \in \X} \bbr{ \inner{\x}{\thetah_t} + \beta_{t-1}\norm{\x}_{V_{t-1}^{-1}\Vt_{t-1}V_{t-1}^{-1}}},$$ where $\beta_{t-1}$, $V_{t-1}^{-1}$ are identical with those in our selection criterion~\eqref{eq:LB-select-criteria} and $\Vt_{t-1} = \lambda I_d + \sum_{s=1}^{t-1} \gamma^{2(t-s-1)} X_s X_s^\T \in \R^{d\times d} $ is an extra covariance matrix. Thus, our algorithm is more efficient than their algorithm since it only needs to maintain one covariance matrix instead of two. This owes to the fact that our analysis of Lemma~\ref{lemma:LB-estimation-error} only uses $V_{t-1}^{-1}$ as the local norm to analyze both bias and variance parts, but the algorithm of~\citet{NIPS'19:weighted-LB} requires to use $l_2$-norm and $V_{t-1}^{-1}\Vt_{t-1}V_{t-1}^{-1}$-norm to control bias and variance parts, respectively. In Section~\ref{sec:analysis-framework}, we provide a sketch of the analysis framework for Lemma~\ref{lemma:LB-estimation-error} and a more detailed discussion is presented in Appendix~\ref{sec:LB-previous-work}. Furthermore, we prove that our algorithm enjoys the same (even slightly better in $d$) regret as the algorithm of~\citet{NIPS'19:weighted-LB}.
\begin{Thm}
  \label{thm:LB-regret}
  For all $\gamma\in(1/T,1)$, $\lambda = d$, the dynamic regret of \LBweightours (Algorithm~\ref{alg:LB-WeightUCB}) is bounded with probability at least $1-1/T$, by 
  \begin{equation}\nonumber
  \begin{split}
      \DReg_T\leq{}& \Ot\sbr{\frac{1}{(1-\gamma)^{\sfrac{3}{2}}}P_T + d(1-\gamma)^{\sfrac{1}{2}}T}.
  \end{split}
  \end{equation}
  Furthermore, by setting the discounted factor optimally as $\gamma = 1- \max\{1/T, \sqrt{P_T/(dT)}\}$, \LBweightours ensures
  \begin{equation}\nonumber
      \DReg_T \leq
      \begin{cases}
      \Ot\sbr{d^{ \sfrac{3}{4}} P_T^{ \sfrac{1}{4}} T^{ \sfrac{3}{4}}} & \mbox{ when } P_T \geq d/T,\vspace{2mm}\\
      \Ot\sbr{d\sqrt{T}} & \mbox{ when } P_T < d/T.
      \end{cases} 
  \end{equation}
\end{Thm}
Compared to previous works~\citep{AISTATS'19:window-LB,NIPS'19:weighted-LB,AISTATS'20:restart}, our approach improves from $\Ot(d^{ \sfrac{7}{8}} P_T^{ \sfrac{1}{4}} T^{ \sfrac{3}{4}})$ to $\Ot(d^{ \sfrac{3}{4}} P_T^{ \sfrac{1}{4}} T^{ \sfrac{3}{4}})$ when $P_T \geq d/T$. We remark that this improved dimensional dependence is simply owing to the more refined tuning of the discounted factor than the one used by~\citep{NIPS'19:weighted-LB}, who did not take the dimension into the tuning. Their algorithm and regret bound can also benefit from the refined tuning. The proof of~\pref{thm:LB-regret} is in Appendix~\ref{sec:LB-regret-proof}.

Further, notice that the optimal choice of discounted factor $\gamma$ requires knowing $P_T$ in advance. To achieve a parameter-free result for unknown $P_T$ case, our algorithm can be combined with the BOB strategy~\citep{AISTATS'19:window-LB} and achieves an $\Ot(d^{ \sfrac{3}{4}} P_T^{ \sfrac{1}{4}} T^{ \sfrac{3}{4}})$ bound. However, this bound is not optimal, and it is possible to design an adaptive weight-based algorithm based on our result, in the spirit of~\citet{COLT'21:black-box}, to further achieve an optimal dynamic regret without prior knowledge of $P_T$. This is very challenging since that at each round $t \in [T]$, we can only receive one data pair $(X_t,r_t)$, which is not adequate for the learner to real-time update the discounted factor $\gamma_{t}$. At the same time, \MASTER algorithm~\citep{COLT'21:black-box} can be considered as a special case of the adaptive weight-based algorithm since it only includes two circumstances: setting $\gamma_t = 0$ to restart at time $t$ and setting $\gamma_t = 1$ to keep going. But for the adaptive weight-based algorithm, the choice of the discounted factor $\gamma_t$ can be  continuous in $[0,1]$, which is more difficult than a binary decision. We leave this as an important open question for future study.

\section{GENERALIZED LINEAR BANDITS}
In this section, we apply the weighted strategy to drifting GLB. Compared to the best-known weight-based algorithm for drifting GLB~\citep{arXiv'21:faury-driftingGLB}, our algorithm is simpler and meanwhile has a better theoretical guarantee.
\label{sec:GLB}
\subsection{Problem Setting}
\label{sec:GLB-problem-setting}
GLB assumes an inverse link function $\mu:\R\rightarrow \R$ such that $r_t = \mu(X_t^\T \theta_t) + \eta_t$, where $\theta_t \in \R^d$ is the unknown parameter and can change over time. Similar to LB, we define \emph{dynamic regret} for GLB as follows:
\begin{equation}
  \label{eq:GLB-regret}
  \DReg_T = \sum_{t=1}^{T} \sbr{\max_{\x \in \X} \mu(\x^{\T}\theta_t) - \mu( X_t^{\T}\theta_t)}.
\end{equation}
Under the GLB setting, we make the same assumptions as those of LB, which include $R$-sub-Gaussian noise, boundedness of feasible set and unknown regression parameters (Assumption~\ref{ass:bounded-norm}). In addition, we work under the standard boundedness assumption of the inverse link function~\citep{NIPS10:GLM-infinite, ICML18:GLM-finite, arXiv'21:faury-driftingGLB}.

\begin{Assum} \label{ass:link-function} 
  The inverse link function $\mu: \R\rightarrow \R$ is $k_\mu$-Lipschitz, and continuously differentiable with $$c_\mu \triangleq \inf_{\{\theta \in \Theta, \x \in \X\}} \dmu(\theta^\T \x)>0,\quad \Theta = \{\theta \mid \norm{\theta}_2\leq S \}.$$ 
\end{Assum}
Previous works~\citep{AISTATS'20:restart,Manage'22:window-LB} define a similar parameter $\widetilde{c}_\mu \triangleq \inf_{\{\theta\in \R^d, \x \in \X\}} \dmu(\theta^\T \x)>0$ and obtain regret upper bound scaling with $1/\widetilde{c}_\mu$. Clearly, that $\widetilde{c}_\mu$ is smaller than our defined $c_\mu$ (and can be much smaller) as $c_\mu$ is defined on $\Theta$ while $\widetilde{c}_\mu$ is defined on $\R$. Therefore, $\widetilde{c}_\mu$ is less attractive to appear in the upper bound.

\begin{algorithm}[!t]
  \caption{\GLBweightours}
  \label{alg:GLB-WeightUCB}
\begin{algorithmic}[1]
\REQUIRE time horizon $T$, discounted factor $\gamma$, confidence $\delta$, regularizer $\lambda$, inverse link function $\mu$, parameters $S$, $L$ and $R$\\
\STATE Set $V_0 = \lambda I_d$, $\thetah_1 = \mathbf{0}$ and compute $k_\mu$ and $c_\mu$
\FOR{$t = 1,2,...,T$}
  \IF{$\|\thetah_t\|_2\leq S$} 
  \STATE let $\thetat_t = \thetah_t$
  \ELSE 
  \STATE Do the projection and get $\thetat_t$ by~\eqref{eq:GLB-projection}
  \ENDIF
  \STATE Compute $\betab_{t-1}$ by~\eqref{eq:GLB-confidence-radius}
  \STATE Select $X_t$ by~\eqref{eq:GLB-select-criteria}
  \STATE Receive the reward $r_t$
  \STATE Update $V_{t} = \gamma V_{t-1} + X_t X_t^\T +(1-\gamma)\lambda I_d$
  \STATE Compute $\thetah_{t+1}$ according to~\eqref{eq:GLB-estimator}
\ENDFOR
\end{algorithmic}
\end{algorithm}

\subsection{Algorithm and Regret Guarantee}
\label{sec:GLB-algorithm}
We propose \GLBweightours, which is a simpler algorithm with better theoretical guarantee compared to previous weight-based algorithm~\citep{arXiv'21:faury-driftingGLB}. The key improvement is owing to our refined analysis framework, which is compatible with a simple projection step.

\tb{Estimator.} At iteration $t$, we first adopt the quasi-maximum likelihood estimator (QMLE) without considering the projection onto the feasible domain. Specifically, the estimator $\thetah_t$ is the solution of the following weighted regularized estimation equation:
\begin{equation}
  \label{eq:GLB-estimator}
    \lambda c_\mu \theta +\sum_{s=1}^{t-1}\gamma^{t-s-1}\sbr{\mu(X_s^\T \theta) - r_s}X_s = 0.
\end{equation}
Given that $\thetah_t$ may not belong to the feasible set $\Theta$ and $c_\mu$ is defined over the parameter $\theta \in \Theta$, we need to perform the following projection step
\begin{equation}
  \label{eq:GLB-projection}
    \thetat_t = \argmin_{\theta \in \Theta}\|g_t(\thetah_t) - g_t(\theta)\|_{V_{t-1}^{-1}},
\end{equation}
where $V_{t} = \lambda I_d + \sum_{s=1}^{t} \gamma^{t-s} X_s X_s^\T$ and $g_t(\theta)$ is
\begin{equation}
  \begin{split}
      \label{eq:GLB-gt}
      g_t(\theta) \teq \lambda c_\mu \theta + \sum_{s=1}^{t-1}\gamma^{t-s-1}\mu(X_s^\T \theta)X_s.
  \end{split}
\end{equation}

However, previous work~\citep{arXiv'21:faury-driftingGLB} cannot conduct the same simple projection in the drifting case as stationary GLB or piecewise-stationary GLB, since they use different local norms to measure the bias and variance parts separately for estimation error analysis. Consequently, they have to design a complicated projection to ensure that the bias and variance parts could be measured by different local norms (see \citep[Section 4.1]{arXiv'21:faury-driftingGLB}, and  our restatements in Appendix~\ref{sec:GLB-previous-work}).

Our refined analysis framework is compatible with this projection operation, thanks to our analysis framework utilizing the same local norm for the bias and variance parts.

\tb{Upper Confidence Bounds.~~} For estimator~\eqref{eq:GLB-estimator} with projection~\eqref{eq:GLB-projection}, we construct following estimation error bound.
\begin{Lemma}
  \label{lemma:GLB-estimation-error}
  For any $\x \in \X$, $\gamma\in(0,1)$ and $\delta \in (0,1)$, with probability at least $1-\delta$, the following holds for all $t \in [T]$
  \begin{equation*}
    \begin{split}
      &\abs{\mu(\x^\T\thetat_t) - \mu(\x^\T\theta_t)} \\
      \leq& \frac{2k_\mu}{c_\mu}\sbr{\sum_{p=1}^{t-1}C(p)\norm{\theta_p - \theta_{p+1}}_2 + \betab_{t-1}\norm{\x}_{V_{t-1}^{-1}}},
    \end{split}
  \end{equation*}
where $C(p) = k_\mu L^2\sqrt{\frac{d}{\lambda}}\gamma^{\frac{t-1}{2}} \sqrt{\frac{\gamma^{-p}-1}{1-\gamma}}$, and $\betab_t$ is the radius of confidence region set by
\begin{equation}
  \label{eq:GLB-confidence-radius}
  \betab_t = \sqrt{\lambda}c_\mu S+R\sqrt{2\log\frac{1}{\delta}+d\log\sbr{1+\frac{L^2 (1-\gamma^{2t})}{\lambda d(1-\gamma^2)}}}.
\end{equation}
\end{Lemma}

The proof of Lemma~\ref{lemma:GLB-estimation-error} is in Appendix~\ref{sec:GLB-estimation-error-proof}. Then, based on Lemma~\ref{lemma:GLB-estimation-error}, we can specify the arm selection criterion as
\begin{equation}
\label{eq:GLB-select-criteria}
\begin{split}
    X_t ={}&\argmax_{\x \in \X} \bbr{ \mu(\x^\T \thetat_t)+ \frac{2k_\mu}{c_\mu}\betab_{t-1}\norm{\x}_{V_{t-1}^{-1}}}.
\end{split}
\end{equation}
The overall algorithm is summarized in Algorithm~\ref{alg:GLB-WeightUCB}. 

Notice that the estimation equation~\eqref{eq:GLB-estimator} and the confidence radius~\eqref{eq:GLB-confidence-radius} are the same as those used in Algorithm 1 of~\citet{arXiv'21:faury-driftingGLB}. But importantly, the final (projected) estimators of the two approaches are significantly different. With a simpler projection operation and our refined analysis framework, we can immediately attain an improved regret guarantee for weight-based algorithm.

\begin{Thm}
  \label{thm:GLB-regret}
  For all $\gamma\in(1/T,1)$, $\lambda = d/c_\mu^2$, the regret of \GLBweightours (Algorithm~\ref{alg:GLB-WeightUCB}) is bounded with probability at least $1-1/T$, by
  \begin{equation}\nonumber
  \begin{split}
      \DReg_T\leq{}& \Ot\sbr{k_\mu^2\frac{1}{(1-\gamma)^{\sfrac{3}{2}}}P_T + \frac{k_\mu}{c_\mu}d(1-\gamma)^{\sfrac{1}{2}}T }.
  \end{split}
  \end{equation}
  By optimally setting $\gamma = 1 - \max\{1/T, \sqrt{k_\mu c_\mu P_T/(dT)}\}$, \GLBweightours achieves the following dynamic regret,
  \begin{equation}\nonumber
      \DReg_T \leq
      \begin{cases}
      \Ot\sbr{\frac{k_\mu^{5/4}}{c_\mu^{\sfrac{3}{4}}}d^{\sfrac{3}{4}} P_T^{\sfrac{1}{4}}T^{\sfrac{3}{4}}}  \mbox{ when } P_T \geq \frac{d}{k_\mu c_\mu T},&\vspace{2mm}\\
      \Ot\sbr{\frac{k_\mu}{c_\mu}d\sqrt{T}}  \quad\quad\quad \mbox{ when } 0\leq P_T < \frac{d}{k_\mu c_\mu T}.&
      \end{cases} 
  \end{equation}
\end{Thm}
Compared to \GLBweight (the best-known weight-based algorithm for drifting GLB)~\citep{arXiv'21:faury-driftingGLB}, focusing on the dependence on $d$, $P_T$, and $T$, we can see that our approach improves the regret from $\Ot(d^{\sfrac{9}{10}} P_T^{\sfrac{1}{5}}T^{\sfrac{4}{5}})$ to $\Ot(d^{\sfrac{3}{4}} P_T^{\sfrac{1}{4}}T^{\sfrac{3}{4}})$. Furthermore, our result also improves their result upon the $c_\mu$ dependence from $c_\mu^{-1}$ to $c_\mu^{-\sfrac{3}{4}}$.
\section{SELF-CONCORDANT BANDITS}
\label{sec:SCB}
This section studies Self-Concordant Bandits (SCB), an important subclass of GLB with many attractive structures.

\subsection{Problem Setting}
\label{sec:SCB-problem-setting}
For SCB, the reward's distribution belongs to a canonical exponential family: ${\P_\theta\mbr{r \givenn \x}} = \exp(r\x^\T\theta - b(\x^\T\theta) +c(r))$ where $b(\cdot)$ is a twice continuously differentiable function and $c(\cdot)$ is a real-valued function. Owing to the benign properties of exponential families, we have $\E\mbr{r\given \x} = b^{\prime}(\x^\T\theta)$ and $\Var\mbr{r\given \x} = b^{\prime\prime}(\x^\T\theta)$ where $b^\prime$ denotes the first derivative of the function $b$, and $b^{\prime\prime}$ denotes its second derivative. Then, we can introduce the (inverse) link function $\mu(\cdot) \teq b^{\prime}(\cdot)$ such that at $t \in [T]$ the following holds
\begin{equation}
  \label{eq:SCB-model-assume}
       \E\mbr{r_t\given X_t} = \mu(X_t^\T \theta_t), \Var\mbr{r_t\given X_t} = \dmu(X_t^\T \theta_t).
\end{equation}
SCB requires the link function satisfy $|\ddmu| \leq \dmu$, usually referred to general self-concordant property. We further introduce the notation $\eta_t = r_t - \mu(X_t^\T \theta_t)$ to denote the noise. SCB successfully models many important real-world applications and captures the reward structure. For example, choosing $\mu(x) = (1+e^{-x})^{-1}$ yields the Logistic Bandits (LogB), which is often adopted to model the binary-feedback reward in recommendation system~\citep{ICML'16:ZhangYJXZ16, NIPS'17:online-GLB, COLT'19:TS_LogB}.

We make several standard assumptions same as LB and GLB, including boundedness of feasible set and unknown regression parameters (Assumption~\ref{ass:bounded-norm}), and non-linearity measure on link function (Assumption~\ref{ass:link-function}). In addition, similar to~\citet{AISTATS'21:SCB-forgetting}, we need assumptions on boundedness of reward, and for the convenience of analysis we let $L = 1$ which means $\|\x\|_2 \leq 1$ for all $\x \in \X$.
\begin{Assum} \label{ass:bounded-rewards}
  The reward received at each round satisfies $0\leq r_t \leq m$ for all $t\in [T]$ and some constant $m>0$.
\end{Assum}
\subsection{Algorithm and Regret Guarantee}
\label{sec:SCB-algorithm}
We propose the \SCBweightours algorithm. Compared to GLB, we use a new local norm for projection and regret analysis which is the key to improving the order of $c_\mu^{-1}$.

\tb{Estimator.~~} At iteration $t$, we first adopt the same maximum likelihood estimator as GLB which is defined in~\eqref{eq:GLB-estimator}. Different from GLB, here we use a new local norm to perform the projection onto the feasible set $\Theta$,
\begin{equation}
  \label{eq:SCB-projection}
    \thetat_t = \argmin_{\theta \in \Theta}\norm{g_t(\thetah_t) - g_t(\theta)}_{H_{t}^{-1}(\theta)},
\end{equation}
where $g_t(\theta)$ is the same as~\eqref{eq:GLB-gt} while $H_t(\theta)$ is defined as
\begin{equation}
  \begin{split}
      \label{eq:SCB-Ht}
      H_t(\theta) \teq \lambda c_\mu I_d + \sum_{s=1}^{t-1}\gamma^{t-s-1}\dmu(X_s^\T \theta)X_sX_s^\T.
  \end{split}
\end{equation}
Notably, compared to $V_t$, $H_t(\theta)$ depends on the function curvature along the dynamics and thus can capture more \emph{local} information. Combining this projection step with the weighted version self-normalized concentration as restated in~\pref{thm:self-normalized-weight-SCB} will remove a constant $c_\mu^{-\sfrac{1}{2}}$ in regret bound.

\tb{Upper Confidence Bound.~~}
For estimator~\eqref{eq:GLB-estimator} with projection~\eqref{eq:SCB-projection}, we construct following estimation error bound.
\begin{Lemma}
  \label{lemma:SCB-estimation-error}
  For any $\x \in \X$, $\gamma\in(0,1)$ and $\delta \in (0,1)$, with probability at least $1-\delta$, the following holds for all $t \in [T]$
  \begin{equation}\nonumber
    \begin{split}
      {}&\abs{\mu(\x^\T\thetat_t) - \mu(\x^\T\theta_t)}\\
      \leq {}&\frac{\sqrt{4+8S} k_\mu}{\sqrt{c_\mu}}\bigg(\sum_{p=1}^{t-1}C(p)\norm{\theta_p - \theta_{p+1}}_2 + \betat_{t-1}\|\x\|_{V_{t-1}^{-1}}\bigg),
    \end{split}
  \end{equation}
  where $C(p) = L^2\sqrt{\frac{d}{\lambda}} \frac{k_\mu}{\sqrt{c_\mu}}\gamma^{\frac{t-1}{2}}\sqrt{\frac{\gamma^{-p}-1}{1-\gamma}} $, and $\betat_t$ is the radius of confidence region set by
\begin{equation}
  \label{eq:SCB-confidence-radius}
  \begin{split}
    \betat_t &= \frac{\sqrt{\lambda c_\mu}}{2 m }+\frac{2 m }{\sqrt{\lambda c_\mu}}\sbr{\log\frac{1}{\delta}+d \log 2}\\
    &+\frac{d m }{\sqrt{\lambda c_\mu}} \log \sbr{1+\frac{ L^2k_\mu(1-\gamma^{2t})}{\lambda c_\mu d(1-\gamma^{2})}}+\sqrt{\lambda c_\mu} S.
  \end{split}
\end{equation}
\end{Lemma}

The proof of Lemma~\ref{lemma:SCB-estimation-error} is in Appendix~\ref{sec:SCB-estimation-error-proof}. Based on Lemma~\ref{lemma:SCB-estimation-error}, we can specify the arm selection criterion as
\begin{equation}
\label{eq:SCB-select-criteria}
    X_t = \argmax_{\x \in \X} \bbr{ \mu(\x^\T \thetat_t)+ 2\sqrt{1+2S}\frac{k_\mu}{\sqrt{c_\mu}}\betat_{t-1}\|\x\|_{V_{t-1}^{-1}}}.
\end{equation}

Our algorithm for SCB (named \SCBweightours) follows the same procedure of Algorithm~\ref{alg:GLB-WeightUCB}, and the difference is that $\thetat_t$ is computed by~\eqref{eq:SCB-projection}, $\betat_{t-1}$ is computed by~\eqref{eq:SCB-confidence-radius} and $X_t$ is computed by~\eqref{eq:SCB-select-criteria}. Further, we have the following guarantee for \SCBweightours algorithm.
\begin{Thm}
  \label{thm:SCB-regret}
  For all $\gamma\in(1/T,1)$, $\lambda = d\log(T)/c_\mu$, the dynamic regret of \SCBweightours is bounded with probability at least $1-1/T$, by
  \begin{equation}\nonumber
  \begin{split}
    \DReg_T\leq\Ot\sbr{\frac{k_\mu^2}{\sqrt{c_\mu}}\frac{1}{(1-\gamma)^{\sfrac{3}{2}}}P_T + \frac{k_\mu}{\sqrt{c_\mu}}d(1-\gamma)^{\sfrac{1}{2}}T}.
  \end{split}
  \end{equation}
  By setting $\gamma = 1-\max\{1/T, \sqrt{k_\mu P_T/(dT)}\}$, we achieve
  \begin{equation}\nonumber
      \DReg_T \leq
      \begin{cases}
      \Ot\sbr{\frac{k_\mu^{\sfrac{5}{4}}}{c_\mu^{\sfrac{1}{2}}}d^{\sfrac{3}{4}} P_T^{\sfrac{1}{4}}T^{\sfrac{3}{4}}} & \mbox{ when } P_T \geq \frac{d}{k_\mu T},\vspace{2mm}\\
      \Ot\sbr{\frac{k_\mu}{c_\mu^{\sfrac{1}{2}}}d\sqrt{T}} & \mbox{ when } 0\leq P_T < \frac{d}{k_\mu T}.
      \end{cases} 
  \end{equation}
\end{Thm}
Compared to GLB, we improve the order of $c_\mu$ from $c_\mu^{-1}$ to $c_\mu^{-\sfrac{1}{2}}$  by exploiting the self-concordant properties. At the same time, in near-stationary environments ($P_T$ is small enough), our result can recover to the performance of \SCBstatic algorithm~\citep{ICML'20:logistic-bandits}. The proof of~\pref{thm:SCB-regret} is presented in Appendix~\ref{sec:SCB-regret-proof}.

In addition, for the piecewise-stationary SCB, we propose \SCBweightourspw algorithm that gets rid of influence of $c_\mu$ and thus directly improves upon~\citep{AISTATS'21:SCB-forgetting}.
\begin{Thm}
  \label{thm:SCB-PW-regret}       
  For all $\gamma\in(1/2,1)$, $D = \log(T)/\log(1/\gamma)$ and $\lambda = d\log(T)/c_\mu$, the regret of SCB-PW-WeightUCB is bounded with probability at least $1-1/T$, by
  \begin{equation}\nonumber
  \begin{split}
    \DReg_T\leq{}& \Ot\sbr{\frac{1}{1-\gamma}\Gamma_T + \frac{1}{\sqrt{1-\gamma}} + d\sqrt{(1-\gamma)}T}.
  \end{split}
  \end{equation}
  By setting $\gamma = 1-\max\{1/T, \sbr{\Gamma_T/(dT)}^{\sfrac{2}{3}}\}$, we achieve
  \begin{equation}\nonumber
      \DReg_T \leq
      \begin{cases}
      \Ot\sbr{d^{\sfrac{2}{3}}\Gamma_T^{\sfrac{1}{3}}T^{\sfrac{2}{3}}} & \mbox{ when } \Gamma_T \geq d/\sqrt{T},\vspace{2mm}\\
      \Ot\sbr{d\sqrt{T}} & \mbox{ when } 0\leq \Gamma_T < d/\sqrt{T}.
      \end{cases} 
  \end{equation}
\end{Thm}   
The overall algorithm and analysis are in Appendix~\ref{sec:SCB-PW}.
\section{REFINED ANALYSIS FRAMEWORK}
\label{sec:analysis-framework}
This section presents a proof sketch for Lemma~\ref{lemma:LB-estimation-error} (estimation error analysis for weighted linear bandits), which also serves as a description of our proposed analysis framework.

\begin{proof}[Proof Sketch]
From the model assumption~\eqref{eq:LB-model} and the estimator~\eqref{eq:LB-estimator}, the estimation error can be split into two parts,
\begin{equation*}
  \begin{split}
  \thetah_t - \theta_t  &= \underbrace{V_{t-1}^{-1}\sbr{\sum_{s=1}^{t-1}\gamma^{t-s-1}  X_s X_s^\T \sbr{\theta_s -\theta_t}}}_{\bias} \\
  &\qquad + \underbrace{V_{t-1}^{-1}\sbr{ \sum_{s=1}^{t-1}\gamma^{t-s-1}\eta_sX_s -\lambda\theta_t}}_{\variance},
  \end{split}
\end{equation*}
where the \emph{bias} part is caused by the parameter drifting, and the \emph{variance} part is due to the stochastic noise. Then, by the Cauchy-Schwarz inequality, for any $\x\in\X$,
\begin{equation}
\begin{split}
  \label{eq:LB-decompose}
    |\x^\T (\thetah_t - \theta_t)| \leq \norm{\x}_{V_{t-1}^{-1}}(A_t+B_t),\\
\end{split}
\end{equation}
where $A_t = \|\sum_{s=1}^{t-1}\gamma^{t-s-1}  X_s X_s^\T \sbr{\theta_s -\theta_t}\|_{V_{t-1}^{-1}}$ and $B_t = \|\sum_{s=1}^{t-1}\gamma^{t-s-1}\eta_sX_s -\lambda\theta_t\|_{V_{t-1}^{-1}}$.

Choosing an appropriate local norm for~\eqref{eq:LB-decompose} is the key to simplify and improve the estimation error analysis. We note that the previous analysis~\citep{NIPS'19:weighted-LB} had to use \emph{different} local norms: using $l_2$-norm in the bias part, and $V_{t-1}^{-1}\Vt_{t-1}V_{t-1}^{-1}$-norm in the variance part, namely,
\begin{equation}
  \begin{split}
    \label{eq:previous-decompose}
      |\x^\T (\thetah_t - \theta_t)| \leq \norm{\x}_{2}A_t'+\norm{\x}_{V_{t-1}^{-1}\Vt_{t-1}V_{t-1}^{-1}}B_t',\\
  \end{split}
\end{equation}
where $A_t' = \|V_{t-1}^{-1}\sum_{s=1}^{t-1}\gamma^{t-s-1}  X_s X_s^\T \sbr{\theta_s -\theta_t}\|_{2}$ and $B_t' = \|\sum_{s=1}^{t-1}\gamma^{t-s-1}\eta_sX_s -\lambda\theta_t\|_{\Vt_{t-1}^{-1}}$ and $\Vt_{t} = \lambda I_d + \sum_{s=1}^{t} \gamma^{2(t-s)} X_s X_s^\T$. Since the need for using sliding-window analysis to analyze the bias part, they have to use $l_2$-norm to get the format of $A_t'$. For the variance part, to use weighted version of self-normalized concentration~(\pref{thm:self-normalized-weight-LB}), they use $V_{t-1}^{-1}\Vt_{t-1}V_{t-1}^{-1}$-norm to control $\x$ term so that $B_t'$ term can be normed by $\Vt_{t-1}^{-1}$. As an improvement, we can directly use the \emph{same} $V_{t-1}^{-1}$-norm to control both parts, which benefits from our new analysis for the bias part and modified analysis for the variance part. 

\tb{Bias Part Analysis.}~~ The first step is to extract the variations of underlying parameters as follows,
\begin{align*}
  A_t \leq L \sum_{p=1}^{t-1} \sum_{s=1}^{p}\gamma^{t-s-1}\norm{X_s}_{V_{t-1}^{-1}}\|\theta_p -\theta_{p+1}\|_2.
\end{align*}
Term $\sum_{s=1}^{p}\gamma^{t-s-1}\norm{X_s}_{V_{t-1}^{-1}}$ should be able to further derive an expression about discounted factor $\gamma$, which can control the variation item. After some derivation, we get
  \begin{equation}\nonumber
    \begin{split}
      A_t \leq L\sqrt{d} \sum_{p=1}^{t-1} \gamma^{\frac{t-1}{2}} \sqrt{\frac{\gamma^{-p}-1}{1-\gamma}}\norm{\theta_p -\theta_{p+1}}_2,
    \end{split}
  \end{equation}
  where the variation item is only controlled by the discounted factor $\gamma$ instead of a virtual window size.
  
  \tb{Variance Part Analysis.}~~ Based on the definition of $V_{t}$ and $\Vt_{t}$, we can find that $V_{t} \succeq \Vt_t$, so we have $B_t \leq B_t' \leq \beta_{t-1}$ where $\beta_{t-1}$ is the confidence radius~\eqref{eq:LB-confidence-radius} and the second inequality is by~\pref{thm:self-normalized-weight-LB}. So we keep the same confidence bound $\beta_{t-1}$ while only need to compute $V_{t-1}^{-1}$ instead of $V_{t-1}^{-1}\Vt_{t-1}V_{t-1}^{-1}$ when doing the arm selection.

  Combining the analysis for bias and variance parts, we can finish the proof of Lemma~\ref{lemma:LB-estimation-error}.
\end{proof}
\begin{Remark}
  The key step~\eqref{eq:LB-decompose} in our analysis framework also resolves the projection issue in GLB. Specifically, after the projection step, the bias-variance decomposition can only be performed in $V_{t-1}^{-1}$-norm. To accommodate previous analysis~\eqref{eq:previous-decompose},~\citet{arXiv'21:faury-driftingGLB} have to inject a highly complex projection operation in the algorithm, whereas our framework already satisfies this condition owing to the usage of the same $V_{t-1}^{-1}$-norm for the bias and variance parts.
\end{Remark}

\section{EXPERIMENTS}
\label{sec:experiments}
In this section, we further empirically examine the performance of our proposed algorithms. We present two synthetic experiments on drifting LB and GLB, respectively. For each experiment, we set the dimension of the feature space to $d=2$, the number of rounds to $T=6000$, and the number of arms to $n=50$. The features of each arm are sampled from the normal distribution $\mathcal{N}(0,1)$ and subsequently rescaled to satisfy $L=1$. We initialize the time-varying parameter $\theta_t$ to $[1,0]$ and rotate it uniformly counterclockwise around the unit circle, completing one full revolution from $0$ to $2\pi$ over the course of $T$ rounds and returning to the starting point $[1,0]$.

\begin{figure}[!t]
    \centering
    \subfloat[LB: Cumulative regret]{ \label{figure:LB-regret} 
        \includegraphics[width=0.49\textwidth, valign=t]{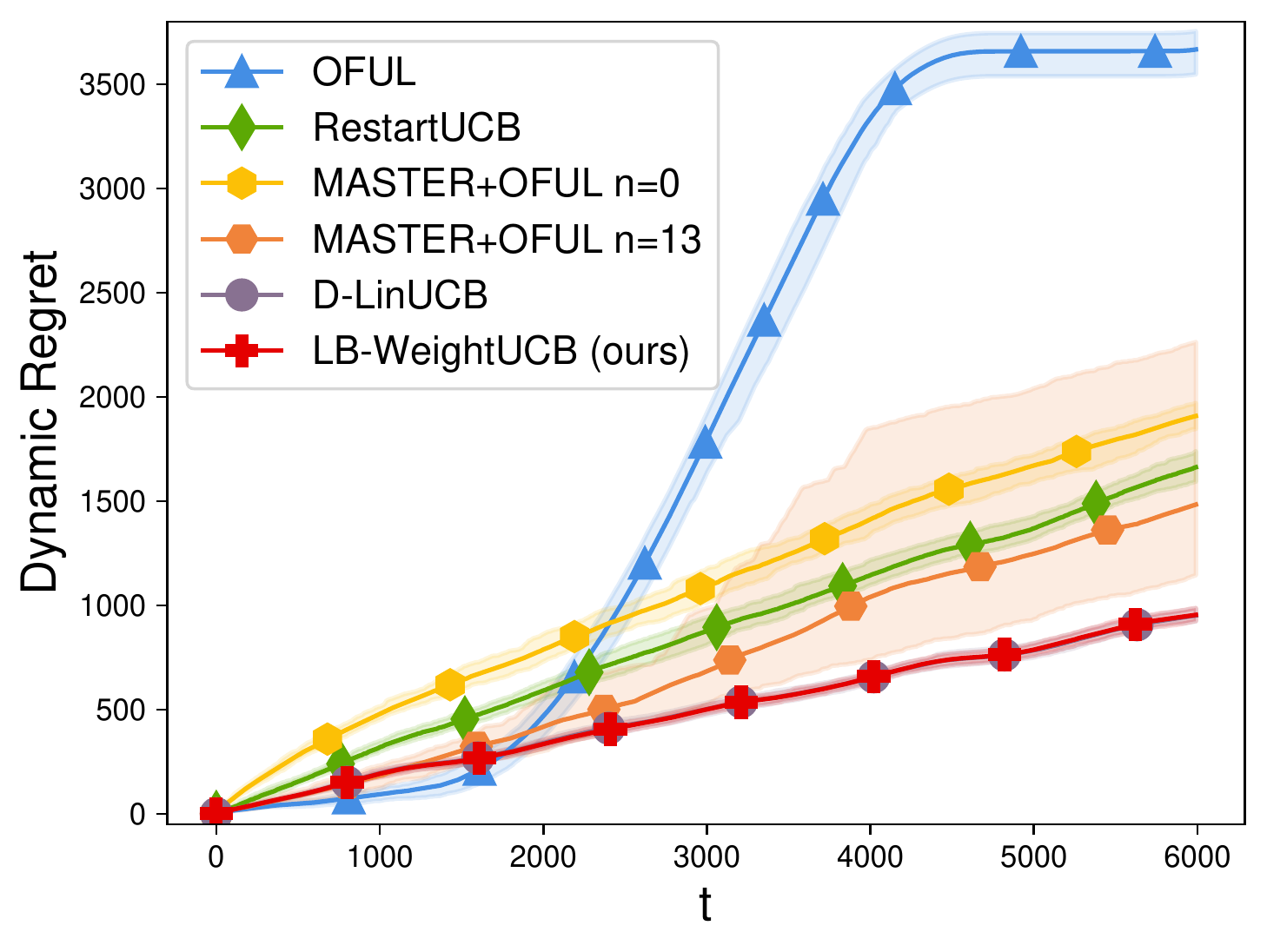}
        \vphantom{\includegraphics[width=0.49\textwidth,valign=t]{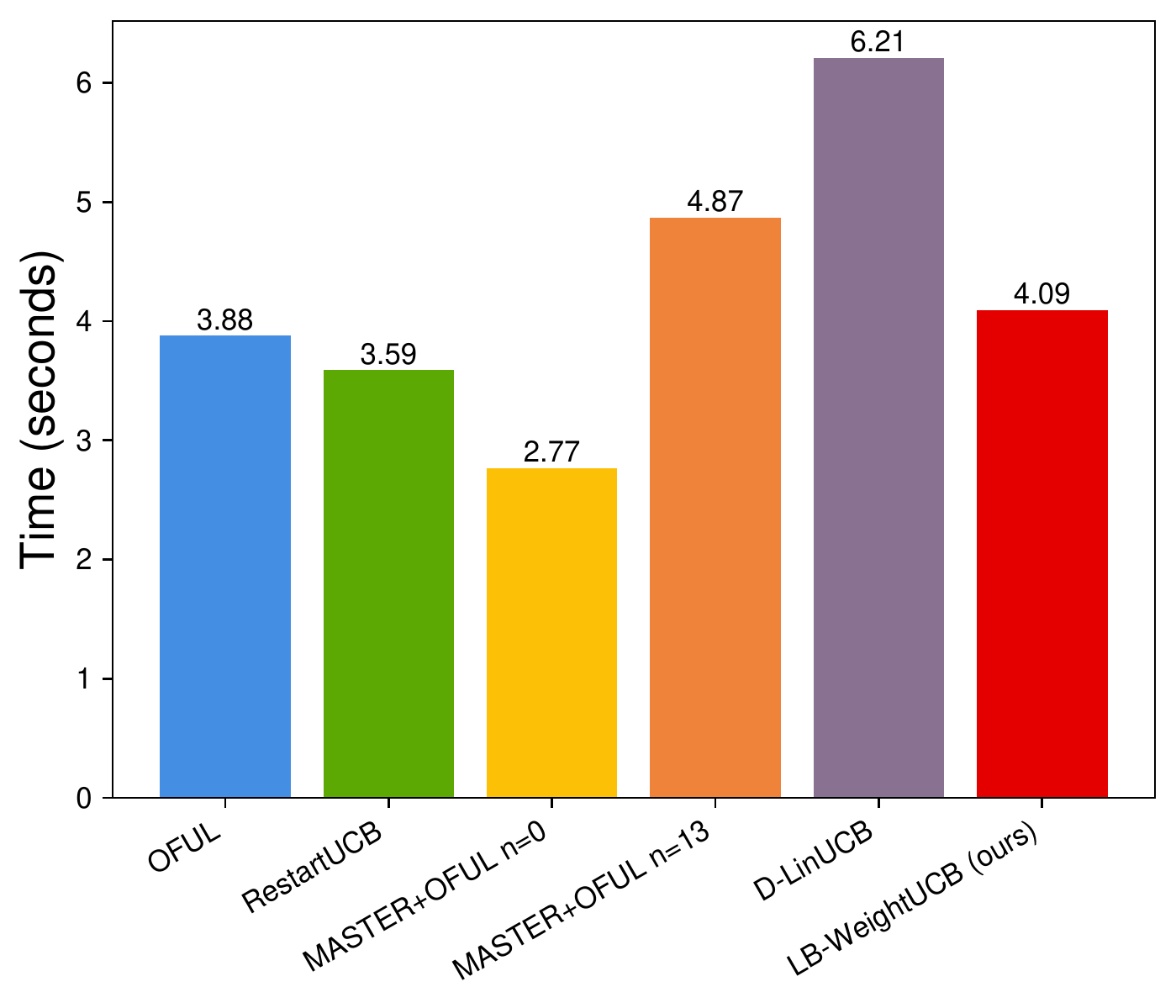}}}
    \subfloat[LB: Average running time]{ \label{figure:LB-running-time}
        \includegraphics[width=0.49\textwidth, valign=t]{figure/lb_running_time.pdf}}
    \caption{Experiments of linear bandits.}
    \label{figure:LB}
\end{figure}

\subsection{Linear Bandits}
\tb{Setting.}
We consider the linear model $r_t = X_t^\T\theta_t +\eta_t$ where the random noise $\eta_t$ is drawn from the normal distribution $\N(0,1)$ at each time $t$ independently. We compare the performance of our proposed \LBweightours algorithm to: (a) the static algorithm \LBstatic~\citep{NIPS'11:AY-linear-bandits}; (b) the restart-based algorithm \LBrestart~\citep{AISTATS'20:restart}; (c) the weight-based algorithm \LBweight~\citep{NIPS'19:weighted-LB}; and (d) the adaptive restart algorithm \LBMASTER~\citep{COLT'21:black-box}. Since $P_T$ is computable, we set the discounted factor $\gamma = 1- \max\{1/T, \sqrt{P_T/(dT)}\}$ for \LBweightours and \LBweight, and set the window size $w$ and restarting period $H$ as $w = H = d^{1/4}\sqrt{T/(1+P_T)}$. For \MASTER, there is a parameter $n$ representing the initial value of a multi-scale exploration parameter (see the input of Procedure 1 in~\citep{COLT'21:black-box}) and the origin \MASTER algorithm lets it start from $0$ (i.e., $n = 0,1,...$). However, a small initial value of $n$ will lead to high-frequent restart and thus achieve poor performance. To address this issue, we experiment with a larger initial value of $n = 13$, which leads to greatly improved performance in our case.

\tb{Results.}
The experimental results are averaged over 20 independent trials. Figure~\ref{figure:LB-regret} shows the cumulative dynamic regret performance, where the shaded area denotes the variance of the 20 independent trials of experimental results. Figure~\ref{figure:LB-running-time} reports the average time per run, with each run containing 6000 rounds. Our \LBweightours algorithm performs as well as \LBweight but significantly more efficient, with over 1.5 times speedup. Figure~\ref{figure:LB-regret} also shows that when equipped with a fine-tuned $n$, \LBMASTER ($n=13$) performs better than \LBrestart, whereas a vanilla \LBMASTER ($n=0$) performs worse due to overly active restarts at the beginning. However, a larger initial value of $n$ results in greater time overhead, since at each restart, \LBMASTER needs to do Procedure~1 once, resulting in an $\mathcal{O}(n2^n)$ time complexity. More importantly, neither adaptive restart (\LBMASTER) nor periodical restart (\LBrestart) outperforms our weighted strategy in slowly-evolving environments.

\begin{figure}[!t]
    \centering
    \subfloat[GLB Algorithms ($S=1$)]{ \label{figure:GLB1} 
        \includegraphics[width=0.49\textwidth, valign=t]{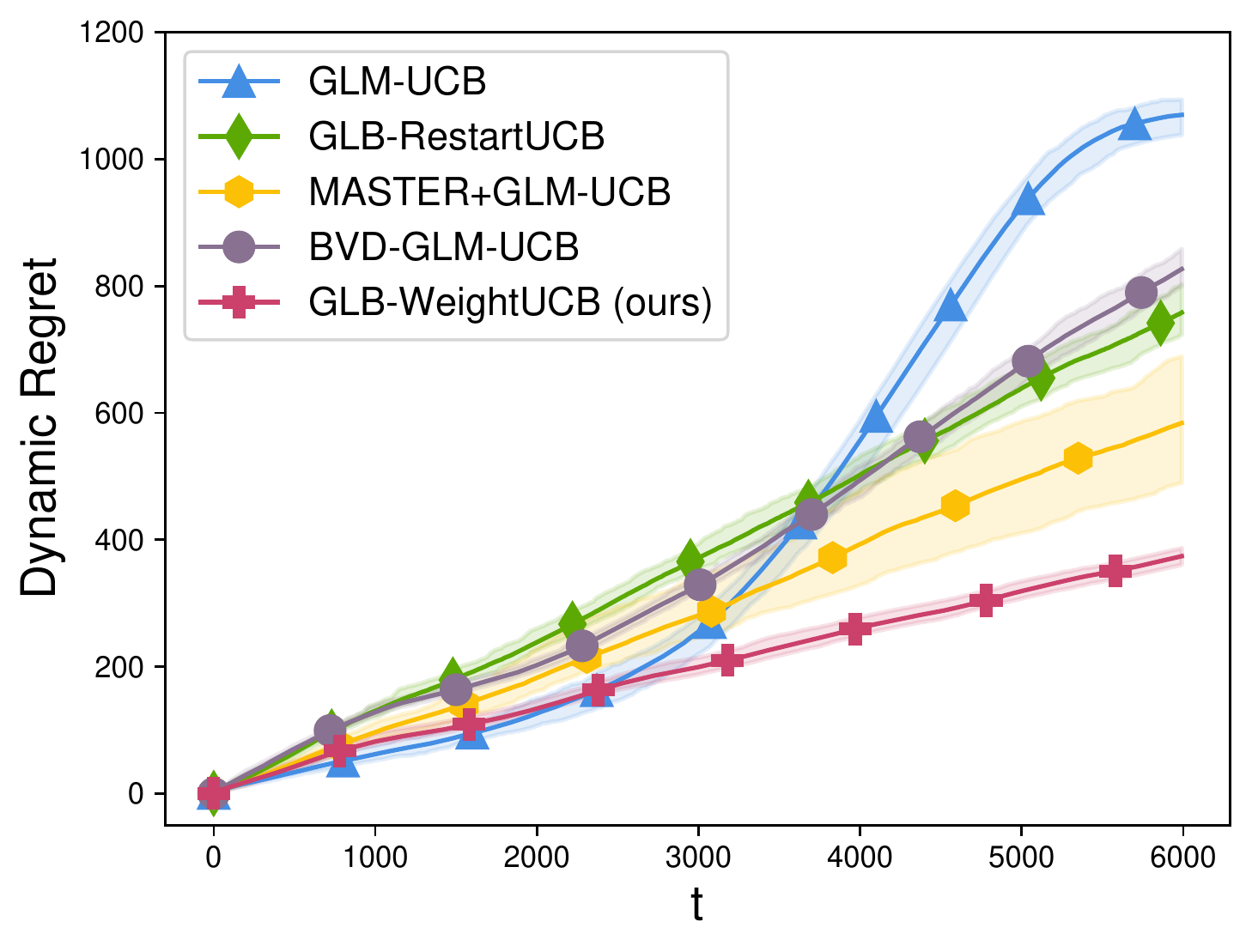}}
    \subfloat[SCB Algorithms ($S=1$)]{ \label{figure:SCB1}
        \includegraphics[width=0.49\textwidth, valign=t]{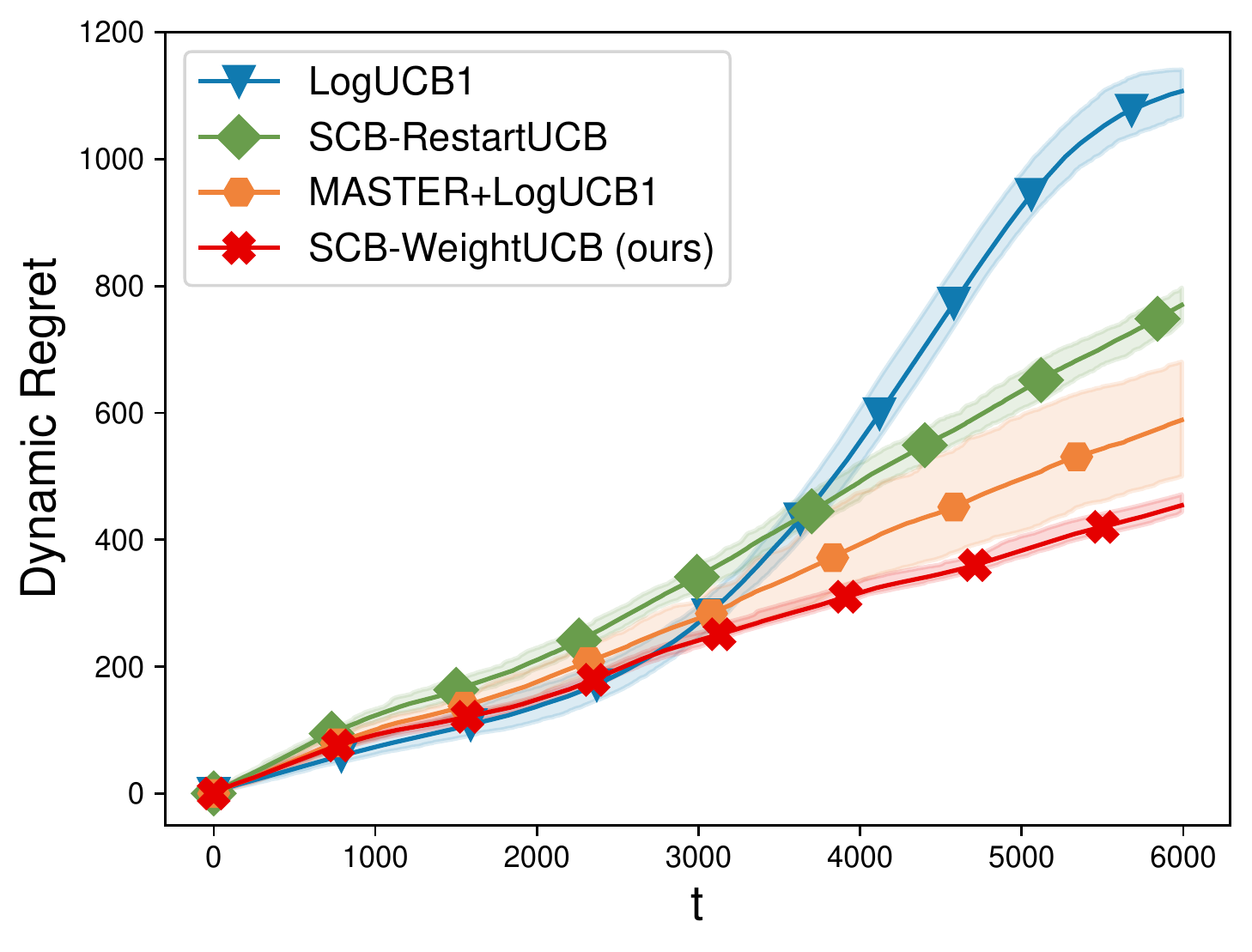}}\\
        \vspace{-2mm}
    \subfloat[GLB Algorithms ($S=5$)]{ \label{figure:GLB5} 
        \includegraphics[width=0.49\textwidth]{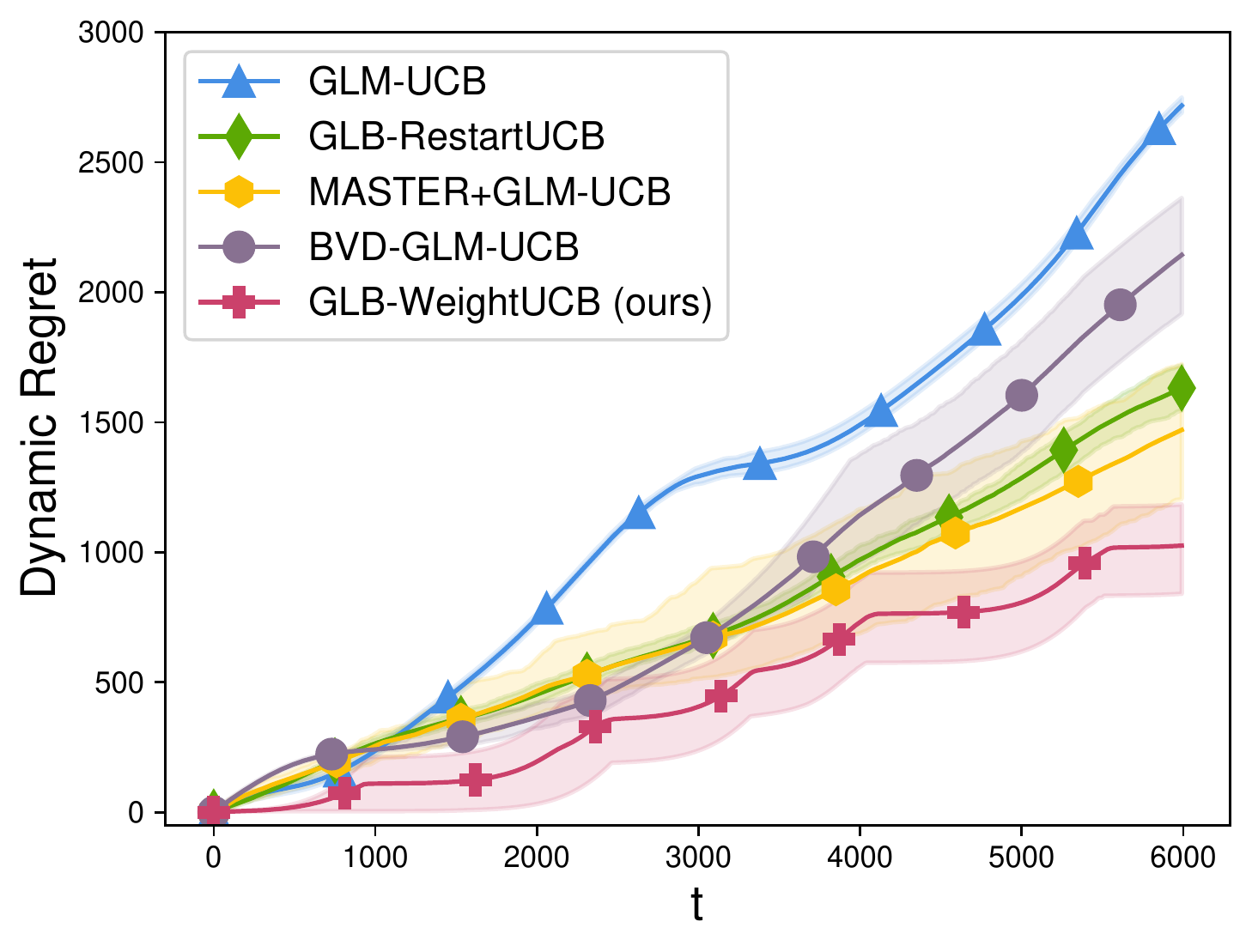}}
    \subfloat[SCB Algorithms ($S=5$)]{ \label{figure:SCB5}
        \includegraphics[width=0.49\textwidth]{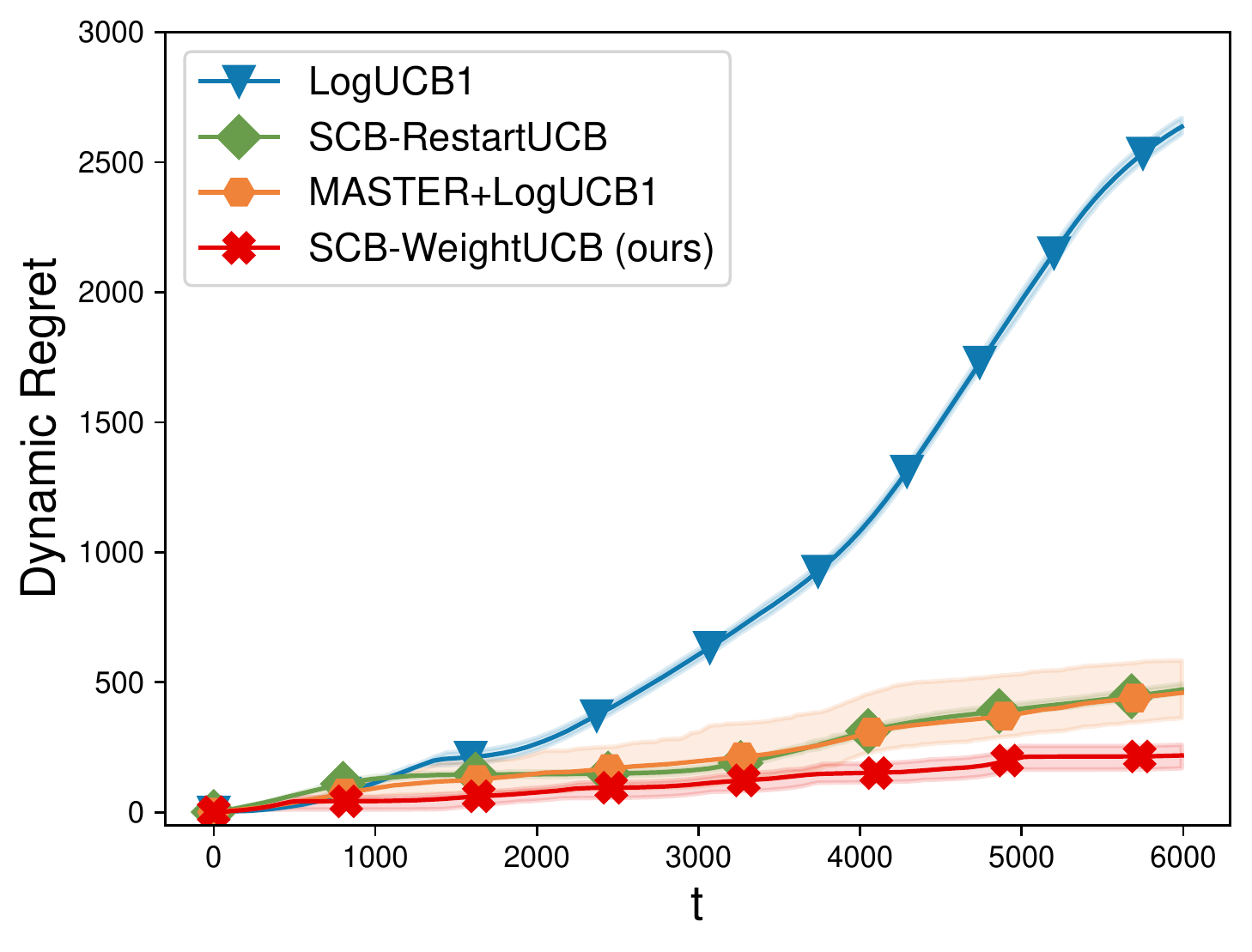}} 
    \caption{Experiments of generalized linear bandits.}
    \label{figure:GLB}
\end{figure}
\subsection{Generalized Linear Bandits}
\tb{Setting.}
We employ the logistic model in GLB experiment, i.e., the reward satisfies $r_t \sim \text{Bernoulli}(\mu(X_t^\T\theta_t))$ with logistic function $\mu(x) = (1+e^{-x})^{-1}$. We consider two cases of $S=1$ and $S=5$, respectively. We compare the performance of our proposed \GLBweightours and \SCBweightours algorithm to: (a) \GLBstatic, static algorithm for GLB~\citep{NIPS10:GLM-infinite}; (b) \SCBstatic, static algorithm for LogB~\citep{ICML'20:logistic-bandits}; (c) \GLBweight, weight-based algorithm for GLB~\citep{arXiv'21:faury-driftingGLB}; (d) \GLBrestart, restart algorithm for GLB~\citep{AISTATS'20:restart}; (e) \SCBrestart, restart algorihtm for SCB~\citep{AISTATS'20:restart}; (f) \GLBMASTER, adaptive restart algorithm for GLB~\citep{COLT'21:black-box}; and (g) \SCBMASTER, adaptive restart algorithm for LogB~\citep{COLT'21:black-box}. We set discounted factor $\gamma = 1- \max\{1/T, \sqrt{c_\mu P_T/(dT)}\}$ for \GLBweightours, $\gamma = 1- (P_T/(\sqrt{d}T))^{\sfrac{2}{5}}$ for \GLBweight and $\gamma = 1- \max\{1/T, \sqrt{P_T/(dT)}\}$ for \SCBweightours. We set restarting period $H = d^{1/4}\sqrt{T/(1+P_T)}$ for both \GLBrestart and \SCBrestart. We set regularizer $\lambda = d$ for \GLBstatic, \GLBweight, \GLBrestart and \GLBMASTER, $\lambda = d/c_\mu^2$ for \GLBweightours and $\lambda = d\log T/c_\mu$ for \SCBstatic, \SCBrestart, \SCBMASTER and \SCBweightours. Note that for LogB, $k_\mu = 1/4 < 1$, so we don't need to control the order of $k_\mu$. For the two \MASTER algorithms, we set $n = 13$.

\tb{Results.}
We present the average cumulative dynamic regret results of our experiments on $20$ independent trials in Figures~\ref{figure:GLB}. When $S$ is small ($S=1, c_\mu^{-1} \approx 5$), all of the weight-based algorithms outperform the static algorithms, and our \GLBweightours and \SCBweightours are better than \GLBweight. When $S$ is large ($S=5, c_\mu^{-1} \approx 152$), \SCBweightours significantly outperforms \GLBweightours, demonstrating the importance of considering self-concordant property (recall that LogB is an instance of SCB). In contrast, the performance of \GLBweight drops dramatically, as it does not take the $c_\mu^{-1}$ issue into account. Similar to LB, the experimental results of GLB also demonstrate the empirical advantage of the weighted strategy over (adaptive) restart strategy in slowly-evolving environments. Specifically, we observe that \GLBweightours consistently outperforms \GLBMASTER, and \SCBweightours consistently outperforms \SCBMASTER.

\section{CONCLUSION}
\label{sec:conclusion}
This paper revisits the weight-based algorithms for three non-stationary parametric bandit models (LB, GLB, SCB). We identify that the  inadequacies of the previous work are due to the inadequate analysis of the estimation error. We thus propose a refined analysis framework that enables the usage of the same local norm for both the bias and variance part in estimation error analysis. Our framework ensures more efficient algorithms for all three bandit models and improves the regret bounds for GLB and SCB settings. 

The importance of our work lies in the fact that we have now made the weight-based algorithms for non-stationary LB/GLB/SCB as competitive as the restart-based algorithms, in terms of both computational efficiency and regret guarantee. Given that the weighted strategy is particularly appealing in gradually drifting scenarios that are commonly seen in real-world applications, it is essential to further design \mbox{\emph{adaptive}} weight-based algorithms for non-stationary parametric bandits with optimal dynamic regret without requiring the knowledge of environmental non-stationarity, in the spirit of the currently best-known result achieved by adaptive restart strategy~\citep{COLT'21:black-box}. 

In this work, we employ $P_T = \sum_{t=2}^{T} \norm{\theta_{t-1} - \theta_{t}}_2$ as a measure to capture the gradually changing environment. However, this metric may not be precise enough in capturing only the gradual changes in the environment, as it can also include other types of variations, such as abrupt changes. This might be able to serve as an explanation why weight-based algorithms do not exhibit a significant theoretical advantage, yet perform remarkably well in experiments on gradually changing environments compared to restart-based algorithms. To overcome this limitation, future research could explore more refined characterizations of gradual changes, drawing inspiration from the ideas behind Sobolev or Holder classes~\citep{NIPS'19:TV_bound} or other information-theoretic tools~\citep{arXiv'22:VanRoy}.

\subsubsection*{Acknowledgements}
This research was supported by NSFC~(61921006), \mbox{JiangsuSF}~(BK20220776) and Collaborative Innovation Center of Novel Software Technology and Industrialization. The authors greatly appreciate helpful comments and suggestions from anonymous reviewers and meta-reviewer. The authors also thank Yu-Hu Yan and Yu-Jie Zhang for many helpful discussions. 

\bibliography{online_learning}
\bibliographystyle{plainnat}

\onecolumn
\appendix
\thispagestyle{empty}

\section{Analysis of \LBweightours}
\label{sec:LB-regret}
In this section, we provide the analysis for \LBweightours algorithm. In Appendix~\ref{sec:LB-review}, we review the \LBweight algorithm proposed by~\citet{NIPS'19:weighted-LB} and restate their estimation error analysis. In Appendix~\ref{sec:LB-estimation-error-proof}, we present our own estimation error analysis for the proposed \LBweightours algorithm, which is captured in Lemma~\ref{lemma:LB-estimation-error}. Finally, in Appendix~\ref{sec:LB-regret-proof}, we provide a proof for our dynamic regret bound, as stated in Theorem~\ref{thm:LB-regret}.

\subsection{Review Estimation Error Analysis of \LBweight Algorithm}
\label{sec:LB-review}
In this part, we review the previous estimation error analysis of the \LBweight algorithm~\citep{NIPS'19:weighted-LB} who has the same estimator as ours~\eqref{eq:LB-estimator}. The first step is to divide the estimation error into the bias and variance parts, where the bias part represents the error caused by parameter drift and the variance part represents the error caused by stochastic noise. Based on the reward model assumption and the estimator (same as eq~\eqref{eq:LB-model} and eq~\eqref{eq:LB-estimator}), the estimation error of \LBweight algorithm can be decomposed as
  \begin{align}
    \thetah_t - \theta_t = {}& V_{t-1}^{-1}\sbr{\sum_{s=1}^{t-1}\gamma^{t-s-1}r_sX_s}- \theta_t\nonumber\\
    = {}& V_{t-1}^{-1}\sbr{\sum_{s=1}^{t-1}\gamma^{t-s-1}\sbr{ X_s^\T \theta_s + \eta_s}X_s}- V_{t-1}^{-1}\sbr{\lambda I_d + \sum_{s=1}^{t-1} \gamma^{t-s-1} X_s X_s^\T}\theta_t\nonumber\\
    = {}& V_{t-1}^{-1}\sbr{\sum_{s=1}^{t-1}\gamma^{t-s-1}  X_s X_s^\T \theta_s + \sum_{s=1}^{t-1}\gamma^{t-s-1}\eta_sX_s }- V_{t-1}^{-1}\sbr{\lambda I_d + \sum_{s=1}^{t-1} \gamma^{t-s-1} X_s X_s^\T}\theta_t\nonumber\\
    = {}& \label{eq:estimation_error_decomposition}\underbrace{V_{t-1}^{-1}\sbr{\sum_{s=1}^{t-1}\gamma^{t-s-1}  X_s X_s^\T \sbr{\theta_s -\theta_t}}}_{\bias} + \underbrace{V_{t-1}^{-1}\sbr{ \sum_{s=1}^{t-1}\gamma^{t-s-1}\eta_sX_s -\lambda\theta_t}}_{\variance}.
  \end{align}

  Afterward,~\citet{NIPS'19:weighted-LB} use different local norms (we will explain the reason of using different local norms later) for the bias and variance parts as follows, 
  \begin{equation}
    \begin{split}
      \label{Russac:decompose}
        |\x^\T (\thetah_t - \theta_t)| \leq \norm{\x}_{2}A_t'+\norm{\x}_{V_{t-1}^{-1}\Vt_{t-1}V_{t-1}^{-1}}B_t',\\
    \end{split}
  \end{equation}  
  where $\Vt_{t} = \lambda I_d + \sum_{s=1}^{t} \gamma^{2(t-s)} X_s X_s^\T$ and 
  \begin{equation}\nonumber
      \begin{split}
        A_t' = \norm{V_{t-1}^{-1}\sum_{s=1}^{t-1}\gamma^{t-s-1}  X_s X_s^\T \sbr{\theta_s -\theta_t}}_{2}, \quad B_t' = \norm{\sum_{s=1}^{t-1}\gamma^{t-s-1}\eta_sX_s -\lambda\theta_t}_{\Vt_{t-1}^{-1}}.
      \end{split}
  \end{equation}

For the bias part,~\citet{NIPS'19:weighted-LB} divide it into two parts on the timeline by introducing a virtual window size $D$,
\begin{align*}
  A_t' \leq \underbrace{\norm{\sum_{s=t-D}^{t-1}V_{t-1}^{-1}\gamma^{t-s-1}  X_s X_s^\T \sbr{\theta_s -\theta_t}}_{2}}_{\mathtt{virtual~window}} + \underbrace{\norm{\sum_{s=1}^{t-D-1}V_{t-1}^{-1}\gamma^{t-s-1}  X_s X_s^\T \sbr{\theta_s -\theta_t}}_{2}}_{\mathtt{small~term}},
\end{align*}
The first term can be considered as a virtual window containing the most recent data obtained after time $t-D$, and can be directly analyzed by the analysis of \LBwindow~\citep{AISTATS'19:window-LB} since it corresponds to the bias part of the estimation error of window strategy and this is why they use $l_2$-norm for bias part. The second term reflects the influence formed by the outdated data obtained before time $t-D$. Since $\gamma^{t-s-1}$ will be very small when $s \leq t-D-1$, this small term is dominated by the first virtual window term which means the bias part is actually controlled by the virtual window size $D$. 

For the variance part,~{\citet{NIPS'19:weighted-LB}} extend the previous self-normalized concentration~{\citep[Theorem 1]{NIPS'11:AY-linear-bandits}} to the weighted version which is restated in Theorem~\ref{thm:self-normalized-weight-LB}. This concentration requires to use $\Vt_t$ as the local norm. To this end, ~\citet{NIPS'19:weighted-LB} split the variance part as
\begin{equation}\nonumber
  \begin{split}
      \abs{\x^\T V_{t-1}^{-1}\sbr{ \sum_{s=1}^{t-1}\gamma^{t-s-1}\eta_sX_s -\lambda\theta_t}} \leq \norm{\x}_{V_{t-1}^{-1}\Vt_{t-1}V_{t-1}^{-1}}\norm{V_{t-1}^{-1}\sbr{ \sum_{s=1}^{t-1}\gamma^{t-s-1}\eta_sX_s -\lambda\theta_t}}_{V_{t-1}\Vt_{t-1}^{-1}V_{t-1}},
  \end{split}
\end{equation} 
where
\begin{equation}\nonumber
  \begin{split}
    \norm{V_{t-1}^{-1}\sbr{ \sum_{s=1}^{t-1}\gamma^{t-s-1}\eta_sX_s -\lambda\theta_t}}_{V_{t-1}\Vt_{t-1}^{-1}V_{t-1}} = \norm{\sum_{s=1}^{t-1}\gamma^{t-s-1}\eta_sX_s -\lambda\theta_t}_{\Vt_{t-1}^{-1}} \leq \norm{\sum_{s=1}^{t-1}\gamma^{t-s-1}\eta_sX_s}_{\Vt_{t-1}^{-1}} +\sqrt{\lambda} S.
  \end{split}
\end{equation} 
Then term $\|\sum_{s=1}^{t-1}\gamma^{t-s-1}\eta_sX_s\|_{\Vt_{t-1}^{-1}}$ can be bounded by Theorem~\ref{thm:self-normalized-weight-LB}.
Finally, based on this analysis, \LBweight needs to use the following action selection criterion which only depends on the variance part since the bias part doesn't contain $\x$, 
\begin{equation}\nonumber
  \begin{split}
    X_t = \argmax_{\x \in \X} \bbr{ \inner{\x}{\thetah_t} + \beta_{t-1}\norm{\x}_{V_{t-1}^{-1}\Vt_{t-1}V_{t-1}^{-1}}},
  \end{split}
\end{equation}
where $\beta_{t-1}$ is the upper bound of $B_t'$ which is the same as~\eqref{eq:LB-confidence-radius}. From this selection criterion, it can be seen that \LBweight needs to maintain two covariance matrices, namely, $V_t$ and $\Vt_t$ at round $t$ during the algorithm running.

In the next section, we present our proof for the estimation error upper bound. The difference between our analysis and \LBweight's analysis mainly starts at step~\eqref{Russac:decompose}, which is the key step of the analysis and our new analysis framework allows us to employ the \emph{same} local norm for both bias and variance parts.

\label{sec:LB-previous-work}
\subsection{Proof of Lemma~\ref{lemma:LB-estimation-error}}
\label{sec:LB-estimation-error-proof}

\begin{proof}
  Using the same derivation in~\eqref{eq:estimation_error_decomposition}, the estimation error of \LBweightours algorithm can also be decomposed as
  \begin{align*}
    \thetah_t - \theta_t = \underbrace{V_{t-1}^{-1}\sbr{\sum_{s=1}^{t-1}\gamma^{t-s-1}  X_s X_s^\T \sbr{\theta_s -\theta_t}}}_{\bias} + \underbrace{V_{t-1}^{-1}\sbr{ \sum_{s=1}^{t-1}\gamma^{t-s-1}\eta_sX_s -\lambda\theta_t}}_{\variance}.
  \end{align*}
  Therefore, by the Cauchy-Schwarz inequality, we know that for any $\x\in\X$,
  \begin{equation}
  \begin{split}
      \label{eq:LB-bound-cauchy}
      \abs{\x^\T\sbr{\thetah_t-\theta_t}} \leq \norm{\x}_{V_{t-1}^{-1}}(A_t+B_t),\\
  \end{split}
  \end{equation}
  where
  \begin{equation}\nonumber
      \begin{split}
      A_t = \norm{\sum_{s=1}^{t-1}\gamma^{t-s-1}  X_s X_s^\T \sbr{\theta_s -\theta_t}}_{V_{t-1}^{-1}}, \quad B_t= \norm{ \sum_{s=1}^{t-1}\gamma^{t-s-1}\eta_sX_s -\lambda\theta_t}_{V_{t-1}^{-1}}.
      \end{split}
  \end{equation}
  The above two terms can be bounded separately, as summarized in the following two lemmas,
  \begin{Lemma}
    \label{lemma:LB-A_t-bound}
    For any $t \in [T]$, we have 
    \begin{equation}\nonumber
        \label{eq:LB-A_t-bound}
        \norm{\sum_{s=1}^{t-1}\gamma^{t-s-1}  X_s X_s^\T \sbr{\theta_s -\theta_t}}_{V_{t-1}^{-1}} \leq L\sqrt{d} \sum_{p=1}^{t-1} \gamma^{\frac{t-1}{2}} \sqrt{\frac{\gamma^{-p}-1}{1-\gamma}}\norm{\theta_p -\theta_{p+1}}_2.
    \end{equation}
  \end{Lemma}

  \begin{Lemma}
    \label{lemma:LB-B_t-bound}
    For any $\delta \in (0,1)$, with probability at least $1-\delta$, the following holds for all $t \in [T]$,
    \begin{equation}\nonumber
        \label{eq:LB-B_t-bound}
        \norm{ \sum_{s=1}^{t-1}\gamma^{t-s-1}\eta_sX_s -\lambda\theta_t}_{V_{t-1}^{-1}} \leq   \sqrt{\lambda}S+R\sqrt{2\log\frac{1}{\delta}+d\log\sbr{1+\frac{L^2 (1-\gamma^{2t-2})}{\lambda d(1-\gamma^2)}}},
    \end{equation}
  \end{Lemma}
  Based on the inequality~\eqref{eq:LB-bound-cauchy}, Lemma~\ref{lemma:LB-A_t-bound}, Lemma~\ref{lemma:LB-B_t-bound}, and the boundedness assumption of the feasible set, for any $\x \in \X$, $\gamma\in (0,1)$ and $\delta \in (0,1)$, with probability at least $1-\delta$, the following holds for all $t \in [T]$,
  \begin{equation}\nonumber
      |\x^\T(\thetah_t-\theta_t)|\leq L^2\sqrt{\frac{d}{\lambda}} \sum_{p=1}^{t-1} \gamma^{\frac{t-1}{2}} \sqrt{\frac{\gamma^{-p}-1}{1-\gamma}}\norm{\theta_p -\theta_{p+1}}_2 + \beta_{t-1}\|\x\|_{V_{t-1}^{-1}},
  \end{equation}
  where $\beta_t \teq \sqrt{\lambda}S+R\sqrt{2\log\frac{1}{\delta}+d\log\sbr{1+\frac{L^2 (1-\gamma^{2t})}{\lambda d(1-\gamma^2)}}}$ is the confidence radius used in \LBweightours. Hence we complete the proof.
\end{proof}

\begin{proof}[{Proof of Lemma~\ref{lemma:LB-A_t-bound}}]
  \label{sec:LB-A_t-bound-proof}
  The first step is to extract the variations of the parameter $\theta_t$ as follows, 
  \begin{align*}
      \norm{\sum_{s=1}^{t-1}\gamma^{t-s-1}  X_s X_s^\T \sbr{\theta_s -\theta_t}}_{V_{t-1}^{-1}} = {} & \norm{\sum_{s=1}^{t-1}\gamma^{t-s-1}  X_s X_s^\T \sum_{p=s}^{t-1}\sbr{\theta_p -\theta_{p+1}}}_{V_{t-1}^{-1}} \\
      = {} & \norm{\sum_{p=1}^{t-1}\sum_{s=1}^{p}\gamma^{t-s-1}  X_s X_s^\T \sbr{\theta_p -\theta_{p+1}}}_{V_{t-1}^{-1}} \\
      \leq {} & \sum_{p=1}^{t-1} \norm{\sum_{s=1}^{p}\gamma^{t-s-1}  X_s \|X_s\|_2 \|\theta_p -\theta_{p+1}\|_2}_{V_{t-1}^{-1}} \\
      \leq {} & L \sum_{p=1}^{t-1} \sum_{s=1}^{p}\gamma^{t-s-1}\norm{X_s}_{V_{t-1}^{-1}}\|\theta_p -\theta_{p+1}\|_2,
  \end{align*}
  and term $\sum_{s=1}^{p}\gamma^{t-s-1}\norm{X_s}_{V_{t-1}^{-1}}$ can be able to further derive an expression about discounted factor $\gamma$ as follows,
  \begin{equation}
    \label{eq:LB-At-gammapart}
    \sum_{s=1}^{p}\gamma^{t-s-1} \norm{X_s}_{V_{t-1}^{-1}} 
    \leq \gamma^{\frac{t-1}{2}} \sqrt{\sum_{s=1}^{p}\gamma^{-s}}\sqrt{\sum_{s=1}^{p}\gamma^{t-s-1}\norm{X_s }_{V_{t-1}^{-1}}^2}  
    \leq \sqrt{d}\gamma^{\frac{t-1}{2}}\sqrt{\frac{\gamma^{-p}-1}{1-\gamma}}.
  \end{equation} 
  In above, we use the fact that for any $\x$, $\norm{\x}_{V_{t-1}^{-1}} \leq \norm{\x}_2/\sqrt{\lambda} $ since $V_{t-1} \succeq \lambda I_d$. The second last step holds by the Cauchy-Schwarz inequality. Besides, the last step follows the fact,
  \begin{equation}
  \begin{split}
    \label{eq:Tr-d}
    \forall p \in[t-1],\quad\sum_{s=1}^{p} \gamma^{t-s-1}\norm{X_s}_{V_{t-1}^{-1}}^2 \leq d,
  \end{split}
  \end{equation} 
  which can be proven by the following argument.
  \begin{align*}
     \sum_{s=1}^{p} \gamma^{t-s-1}\norm{X_s}_{V_{t-1}^{-1}}^2 = {} &\sum_{s=1}^{p} \gamma^{t-s-1}\mathrm{Tr}(X_s^\T V_{t-1}^{-1} X_s) 
    =  \mathrm{Tr}\sbr{V_{t-1}^{-1} \sum_{s=1}^{p} \gamma^{t-s-1}X_s X_s^\T} \\
    \leq {} & \mathrm{Tr}\sbr{V_{t-1}^{-1} \sum_{s=1}^{p}\gamma^{t-s-1} X_s X_s^{\T}} + \mathrm{Tr}\sbr{V_{t-1}^{-1} \sum_{s=p+1}^{t-1}\gamma^{t-s-1} X_s X_s^{\T}} + \mathrm{Tr}\sbr{V_{t-1}^{-1} \lambda \sum_{i=1}^{d} \mathbf{e}_i \mathbf{e}_i^{\T}} \\
    = {} & \mathrm{Tr}(I_d) = d.
  \end{align*} 
  Hence, we complete the proof.
\end{proof}

\begin{proof}[{Proof of Lemma~\ref{lemma:LB-B_t-bound}}]
  \label{sec:LB-B_t-bound-proof}
  Let $\Vt_t \teq \lambda I_d + \sum_{s=1}^{t} \gamma^{2(t-s)} X_s X_s^\T $,
  \begin{equation}\nonumber
  \begin{split}
    \norm{ \sum_{s=1}^{t-1}\gamma^{t-s-1}\eta_sX_s -\lambda\theta_t}_{V_{t-1}^{-1}} \leq  \norm{ \sum_{s=1}^{t-1}\gamma^{t-s-1}\eta_sX_s }_{V_{t-1}^{-1}}+\norm{\lambda\theta_t}_{V_{t-1}^{-1}}
    \leq  \norm{ \sum_{s=1}^{t-1}\gamma^{t-s-1}\eta_sX_s }_{\Vt_{t-1}^{-1}}+\sqrt{\lambda}S.
  \end{split}
  \end{equation}
  Recall that $V_{t}  =  \lambda  I_d +\sum_{s= 1}^{t} \gamma^{t-s}X_sX_s^\T$, so the last inequality comes from
  \begin{equation}\nonumber
  \begin{split}
    V_{t} = \lambda I_d + \sum_{s=1}^{t} \gamma^{t-s} X_s X_s^\T \succeq \lambda I_d + \sum_{s=1}^{t} \gamma^{2(t-s)} X_s X_s^\T = \Vt_t.
  \end{split}
  \end{equation}
  We emphasize that the $\Vt_t$ is introduced into analysis \emph{only}, which is actually \emph{not} required in our algorithmic implementation. From the weighted version maximal deviation inequality~\citep[Theorem 1]{NIPS'19:weighted-LB}, restated in~\pref{thm:self-normalized-weight-LB}, we can get the bound for the first term $\| \sum_{s=1}^{t-1}\gamma^{t-s-1}\eta_sX_s\|_{\Vt_{t-1}^{-1}}$ as below by just let $w_s = \gamma^{t-s-1}, \mu_t= \lambda$,
  \begin{equation}\nonumber
  \begin{split}
    \norm{ \sum_{s=1}^{t-1}\gamma^{t-s-1}\eta_sX_s}_{\Vt_{t-1}^{-1}} \leq  R \sqrt{2 \log \frac{1}{\delta}+d \log \sbr{1+\frac{L^{2} \sum_{s=1}^{t-1} \gamma^{2(t-s-1)}}{\lambda d }}}
    \leq R\sqrt{2\log\frac{1}{\delta}+d\log\sbr{1+\frac{L^2 (1-\gamma^{2t-2})}{\lambda d(1-\gamma^2)}}},
  \end{split}
  \end{equation}
  which completes the proof.
\end{proof}

\subsection{Proof of~\pref{thm:LB-regret}}
\label{sec:LB-regret-proof}

\begin{proof}
  Let $X_t^* \teq \argmax_{\x \in \X} \x^\T \theta_t$. Due to Lemma~\ref{lemma:LB-estimation-error} and the fact that $X_t^*,X_t\in \X$, each of the following holds with probability at least $1-\delta$,
  \begin{equation}\nonumber
    \begin{split}
        \forall t \in [T], X_t^{*\T}\theta_t \leq{}&  X_t^{*\T}\thetah_t +L^2\sqrt{\frac{d}{\lambda}} \sum_{p=1}^{t-1} \gamma^{\frac{t-1}{2}} \sqrt{\frac{\gamma^{-p}-1}{1-\gamma}}\norm{\theta_p -\theta_{p+1}}_2 + \beta_{t-1}\|X_t^*\|_{V_{t-1}^{-1}}\\
        \forall t \in [T], X_t^{\T}\theta_t \geq{}&  X_t^{\T}\thetah_t -L^2\sqrt{\frac{d}{\lambda}} \sum_{p=1}^{t-1} \gamma^{\frac{t-1}{2}} \sqrt{\frac{\gamma^{-p}-1}{1-\gamma}}\norm{\theta_p -\theta_{p+1}}_2 - \beta_{t-1}\|X_t\|_{V_{t-1}^{-1}}.
    \end{split}
    \end{equation}
    By the union bound, the following holds with probability at least $1-2\delta$,
  \begin{equation}\nonumber
  \begin{split}
    \forall t \in [T], X_t^{*\T}\theta_t - X_t^{\T}\theta_t \leq{}&  X_t^{*\T}\thetah_t -X_t^{\T}\thetah_t +2L^2\sqrt{\frac{d}{\lambda}} \sum_{p=1}^{t-1} \gamma^{\frac{t-1}{2}} \sqrt{\frac{\gamma^{-p}-1}{1-\gamma}}\norm{\theta_p -\theta_{p+1}}_2 + \beta_{t-1}(\|X_t^*\|_{V_{t-1}^{-1}}+\|X_t\|_{V_{t-1}^{-1}})\\
      \leq{}& 2L^2\sqrt{\frac{d}{\lambda}} \sum_{p=1}^{t-1} \gamma^{\frac{t-1}{2}} \sqrt{\frac{\gamma^{-p}-1}{1-\gamma}}\norm{\theta_p -\theta_{p+1}}_2 + 2\beta_{t-1}\|X_t\|_{V_{t-1}^{-1}},
  \end{split}
  \end{equation}
  where the last step comes from the arm selection criterion~\eqref{eq:LB-select-criteria} such that $$X_t^{*\T}\thetah_t +\beta_{t-1}\|X_t^*\|_{V_{t-1}^{-1}} \leq X_t^{\T}\thetah_t+\beta_{t-1}\|X_t\|_{V_{t-1}^{-1}}.$$
  Hence, the following dynamic regret bound holds with probability at least $1-2\delta$ and can be divided into two parts,
  \begin{equation}\nonumber
  \begin{split}
      \DReg_T =  \sum_{t=1}^T \sbr{X_t^{*\T} \theta_t - X_t^\T \theta_t} \leq  \underbrace{2L^2\sqrt{\frac{d}{\lambda}}\sum_{t=1}^T\sum_{p=1}^{t-1}  \gamma^{\frac{t-1}{2}} \sqrt{\frac{\gamma^{-p}-1}{1-\gamma}}\norm{\theta_p -\theta_{p+1}}_2}_{\bias}+\underbrace{2\beta_{T}\sum_{t=1}^T\|X_t\|_{V_{t-1}^{-1}}}_{\variance},
  \end{split}
  \end{equation}
  where $\beta_T = \sqrt{\lambda}S+R\sqrt{2\log\frac{1}{\delta}+d\log\sbr{1+\frac{L^2 (1-\gamma^{2T})}{\lambda d(1-\gamma^2)}}}$ is the confidence radius.

  Now we derive the upper bound for the bias and variance parts separately.

  \tb{Bias Part.} For the bias part, we need to extract path length $P_T$ and show the control of the discounted factor $\gamma$ on $P_T$.
  \begin{align*}
    2L^2\sqrt{\frac{d}{\lambda}}\sum_{t=1}^T \sum_{p=1}^{t-1}   \gamma^{\frac{t-1}{2}}\sqrt{\frac{\gamma^{-p}-1}{1-\gamma}}\norm{\theta_p -\theta_{p+1}}_2
    = {}&2L^2\sqrt{\frac{d}{\lambda}}\sum_{p=1}^{T-1}  \sum_{t=p+1}^{T}   \gamma^{\frac{t-1}{2}}\sqrt{\frac{\gamma^{-p}-1}{1-\gamma}}\norm{\theta_p -\theta_{p+1}}_2\\
    = {}&2L^2\sqrt{\frac{d}{\lambda}}\sum_{p=1}^{T-1}  \frac{\gamma^{\frac{p}{2}}-\gamma^{ \frac{T}{2}  }}{1-\gamma^\frac{1}{2}}\sqrt{\frac{\gamma^{-p}-1}{1-\gamma}}\norm{\theta_p -\theta_{p+1}}_2\\
    \leq {}& 2L^2\sqrt{\frac{d}{\lambda}}\sum_{p=1}^{T-1} \frac{\gamma^{\frac{p}{2}}-\gamma^{ \frac{T}{2}  }}{(1-\gamma^\frac{1}{2})\frac{1+\gamma^\frac{1}{2}}{2}}\sqrt{\frac{\gamma^{-p}-1}{1-\gamma}}\norm{\theta_p -\theta_{p+1}}_2\\
    \leq {}& 4L^2\sqrt{\frac{d}{\lambda}}\sum_{p=1}^{T-1} \frac{\gamma^{\frac{p}{2}}\gamma^{-\frac{p}{2}}}{(1-\gamma)^{\sfrac{3}{2}}}\norm{\theta_p -\theta_{p+1}}_2\\
    = {}& 4L^2\sqrt{\frac{d}{\lambda}}\frac{1}{(1-\gamma)^{\sfrac{3}{2}}}P_T.\\
\end{align*}
So for the bias part, we have
\begin{equation}
  \label{eq:LB-regret-bias-bound}
  2L^2\sqrt{\frac{d}{\lambda}}\sum_{t=1}^T \sum_{p=1}^{t-1}   \gamma^{\frac{t-1}{2}}\sqrt{\frac{\gamma^{-p}-1}{1-\gamma}}\norm{\theta_p -\theta_{p+1}}_2 \leq 4L^2\sqrt{\frac{d}{\lambda}}\frac{1}{(1-\gamma)^{\sfrac{3}{2}}}P_T.\\
\end{equation}

\tb{Variance Part.}
First, use the Cauchy-Schwarz inequality, we know that
$
    2\beta_{T}\sum_{t=1}^T\|X_t\|_{V_{t-1}^{-1}}
    \leq 2\beta_{T}\sqrt{T\sum_{t=1}^T\|X_t\|_{V_{t-1}^{-1}}^2}.
$
Then by Lemma~\ref{lemma:potential-lemma} (potential lemma), we have the following upper bound:
\begin{equation}
  \label{eq:LB-regret-variance-bound}
\begin{split}
    2\beta_{T}\sum_{t=1}^T\|X_t\|_{V_{t-1}^{-1}}
    \leq {}& 2\beta_{T}\sqrt{2\max\{1,L^2/\lambda\}dT}\sqrt{T\log\frac{1}{\gamma}+\log\sbr{1+\frac{L^2}{\lambda d(1-\gamma)}}}.
\end{split}
\end{equation}
Combining the upper bounds of the bias and variance parts and with confidence level $\delta = 1/(2T)$, by union bound we have the following dynamic regret bound with probability at least $1-1/T$,
\begin{equation}\nonumber
\begin{split}
    \DReg_T\leq{}& 4L^2\sqrt{\frac{d}{\lambda}}\frac{1}{(1-\gamma)^{\sfrac{3}{2}}}P_T+ 2\beta_{T}\sqrt{2\max\{1,L^2/\lambda\}dT}\sqrt{T\log\frac{1}{\gamma}+\log\sbr{1+\frac{L^2}{\lambda d(1-\gamma)}}}.  
\end{split}
\end{equation}
where $\beta_T = \sqrt{\lambda}S+R\sqrt{2\log T+2\log 2+d\log\sbr{1+\frac{L^2 (1-\gamma^{2T})}{\lambda d(1-\gamma^2)}}}$. Since that there has a term $T \sqrt{\log (1/\gamma)}$ in the regret bound, we cannot let $\gamma$ close to $0$, so we set $\gamma \geq 1/T$ and have $\log(1/\gamma) \leq C (1-\gamma)$,
where $C = \log T/(1-1/T)$. 

Then, ignoring logarithmic factors in time horizon $T$, and let $\lambda = d$, we finally obtain
\begin{equation}\nonumber
  \begin{split}
      \DReg_T\leq{}& \Ot\sbr{\frac{1}{(1-\gamma)^{\sfrac{3}{2}}}P_T + d(1-\gamma)^{\sfrac{1}{2}}T}.
  \end{split}
  \end{equation}
  When $P_T< d/T$ (which corresponds a small amount of non-stationarity), we simply set $\gamma = 1-1/T$ and achieve an $\Ot(d\sqrt{T})$ regret bound.  Besides, when coming to the non-degenerated case ($P_T\geq d/T$), We set the discounted factor optimally as $1-\gamma = \sqrt{P_T/(dT)}$ and attain an $\Ot(d^{\sfrac{3}{4}}P_T^{\sfrac{1}{4}}T^{\sfrac{3}{4}})$ dynamic regret bound, which completes the proof.
\end{proof}
\section{Analysis of \GLBweightours}
\label{sec:GLB-regret}
In this section, we provide the analysis for \GLBweightours algorithm. In Appendix~\ref{sec:GLB-previous-work}, we review the projection issue of GLB and restate the \GLBweight algorithm of~\citet{arXiv'21:faury-driftingGLB}. In Appendix~\ref{sec:GLB-estimation-error-proof}, we present the proof of the estimation error upper bound of our \GLBweightours algorithm (namely, Lemma~\ref{lemma:GLB-estimation-error}). Finally, in Appendix~\ref{sec:GLB-regret-proof}, we provide the proof of dynamic regret upper bound as stated in Theorem~\ref{thm:GLB-regret}.

\subsection{Review Projection Step of \GLBweight Algorithm}
\label{sec:GLB-previous-work}
As mentioned in Section~\ref{sec:GLB-algorithm}, the main difficulty of GLB is that the result of MLE or QMLE estimator $\thetah_t$ may not belong to the feasible set $\Theta$ and $c_\mu$ is defined over the parameter $\theta \in \Theta$. Under stationary environments,~\citet{NIPS10:GLM-infinite} overcame this difficulty by introducing a projection step as
\begin{equation}
  \label{eq:simple-projection}
    \thetat_t = \argmin_{\theta \in \Theta}\|g_t(\thetah_t) - g_t(\theta)\|_{V_{t-1}^{-1}},
\end{equation}
where $V_{t} = \lambda I_d + \sum_{s=1}^{t}X_s X_s^\T$ and $g_t(\theta) = \lambda c_\mu \theta + \sum_{s=1}^{t-1}\mu(X_s^\T \theta)X_s$ are the static version (by setting $\gamma = 1$). Based on the QMLE, we know that 
  \begin{equation}
      g_t(\thetah_t) = \lambda c_\mu \thetah_t + \sum_{s=1}^{t-1}\mu(X_s^\T \thetah_t)X_s = \sum_{s=1}^{t-1}r_sX_s,
  \end{equation}
  and then by the mean value theorem, we know that 
  \begin{equation}
      g_t(\theta_1) - g_t(\theta_2)= G_t(\theta_1, \theta_2)(\theta_1 - \theta_2),
  \end{equation}
  where $G_t(\theta_1, \theta_2) \triangleq \int_{0}^1 \nabla g_t(s\theta_2+(1-s)\theta_1)\diff{s}\in \R^{d\times d}$. Notice that for any $\theta \in \Theta$, the gradient of $g_t$ satisfies 
  \begin{equation}\nonumber
      \nabla g_t(\theta) = \lambda c_\mu I_d + \sum_{s=1}^{t-1}\dmu(X_s^\T \theta)X_sX_s^\T \succeq c_\mu V_{t-1},
  \end{equation}
  which clearly implies $\forall \theta_1, \theta_2 \in \Theta, G_t(\theta_1,\theta_2)\succeq c_\mu V_{t-1}$.

By this projection step,~\citet{NIPS10:GLM-infinite} can analyze the estimation error like,
\begin{equation}\nonumber
  \begin{split}|\mu(\x^\T\thetat_t) - \mu(\x^\T\theta_t)| \leq {}& k_\mu |\x^\T(\thetat_t - \theta_t)|\\
      = {}& k_\mu |\x^\T G_t^{-1}(\theta_t,\thetat_t)(g_t(\thetat_t) - g_t(\theta_t))|\\
      \leq {}& k_\mu \|\x\|_{G_t^{-1}(\theta_t,\thetat_t)}\|g_t(\thetat_t) - g_t(\theta_t)\|_{G_t^{-1}(\theta_t,\thetat_t)}\\
      \leq {}& \frac{k_\mu}{c_\mu} \|\x\|_{V_{t-1}^{-1}}\|g_t(\thetat_t) - g_t(\theta_t)\|_{V_{t-1}^{-1}}\\
      \leq {}& \frac{2k_\mu}{c_\mu} \|\x\|_{V_{t-1}^{-1}}\|g_t(\thetah_t) - g_t(\theta_t)\|_{V_{t-1}^{-1}},
  \end{split}
  \end{equation}
  where the last step comes from the projection step. After doing the projection step, term $g_t(\thetah_t) - g_t(\theta_t)$ is the estimation error of the MLE without projection. 
Notice that in piecewise-stationary case,~\citet{AISTATS'21:SCB-forgetting} can also use this projection step.~\citet{arXiv'21:faury-driftingGLB} believe that these two previous works could use this projection operation mainly due to their stationary or piecewise-stationary setting. They mention that for the drifting case, the estimation error is always divided into the bias (tracking error) and variance (learning error) part, and this simple projection operation ignores the bias part which needs to be generalized to adapt to the two sources of deviation. In the analysis, the problem is that after the projection step estimation error term $g_t(\thetah_t) - g_t(\theta_t)$ need to be separate into the bias part and variance parts, and~\citet{arXiv'21:faury-driftingGLB} need to use $l_2$-norm for bias part and $V_{t-1}^{-1}$ for variance part. But the whole estimation error is already normed by $V_{t-1}^{-1}$, which means they cannot use the previous analysis of the window strategy for the bias part.

\begin{algorithm}[!t]
  \caption{\GLBweight~\citep{arXiv'21:faury-driftingGLB}}
  \label{alg:BVD-GLM-UCB}
\begin{algorithmic}[1]
\REQUIRE time horizon $T$, discounted factor $\gamma$, confidence $\delta$, regularizer $\lambda$, inverse link function $\mu$, parameters $S$, $L$ and $R$\\
\STATE Set $V_0 = \lambda I_d$, $\thetah_1 = \mathbf{0}$ and compute $k_\mu$ and $c_\mu$
\FOR{$t = 1,2,...,T$}
  \STATE Find $\theta_t^p$ by solving $ \theta_t^p \in \argmin_{\theta \in \mathbb{R}^d}\bbr{\norm{g_t(\theta)-g_t(\hat{\theta}_t)}_{V_{t}^{-2}} \text { s.t } \Theta \cap \mathcal{E}_t^\delta(\theta) \neq \emptyset}$
  \STATE Select $\thetat_t \in \Theta \cap \mathcal{E}_t^\delta(\theta_t^p)$ where $\mathcal{E}_t^\delta(\theta):=\left\{\theta^{\prime} \in \R^d \givenn \norm{g_t\left(\theta^{\prime}\right)-g_t(\theta)}_{\Vt_{t}^{-1}} \leq \betab_t(\delta)\right\}$
  \STATE Compute $\betab_{t-1}$ by $\betab_t = \sqrt{\lambda}c_\mu S+R\sqrt{2\log\frac{1}{\delta}+d\log\sbr{1+\frac{L^2 (1-\gamma^{2t})}{\lambda d(1-\gamma^2)}}}$
  \STATE Select $X_t$ by $X_t =\argmax_{\x \in \X} \bbr{ \mu(\x^\T \thetat_t)+ \frac{2k_\mu}{c_\mu}\betab_{t-1}\norm{\x}_{V_{t-1}^{-1}}}$
  \STATE Receive the reward $r_t$
  \STATE Update $V_{t} = \gamma V_{t-1} + X_t X_t^\T +(1-\gamma)\lambda I_d$, $\Vt_{t} = \gamma^2 V_{t-1} + X_t X_t^\T +(1-\gamma^2)\lambda I_d$
  \STATE Compute $\thetah_{t+1}$ according to $\lambda c_\mu \theta +\sum_{s=1}^{t}\gamma^{t-s}\sbr{\mu(X_s^\T \theta) - r_s}X_s = 0$
\ENDFOR
\end{algorithmic}
\end{algorithm}

To this end,~\citet{arXiv'21:faury-driftingGLB} propose the \GLBweight algorithm for drifting generalized linear bandits, as restated in Algorithm~\ref{alg:BVD-GLM-UCB}, where a new projection step is devised to solve this problem. Specifically, at each round $t$, the first step is to construct the confidence set $\mathcal{E}_t^\delta(\theta)$ which represents the influence of the stochastic noise.
\begin{equation}
  \label{Faury:confidence-set}
  \mathcal{E}_t^\delta(\theta):=\left\{\theta^{\prime} \in \R^d \givenn \norm{g_t\left(\theta^{\prime}\right)-g_t(\theta)}_{\Vt_{t}^{-1}} \leq \betab_t(\delta)\right\}.
\end{equation}
The second step is to find a confidence set $\mathcal{E}_t^\delta(\theta_t^p)$ that intersects with the feasible set, and the gap between $\theta_t^p$ and $\thetah_t$ represents the influence of parameter drift.
\begin{equation}
  \label{Faury:projection}
  \theta_t^p \in \argmin_{\theta \in \mathbb{R}^d}\bbr{\norm{g_t(\theta)-g_t(\hat{\theta}_t)}_{V_{t}^{-2}} \text { s.t } \Theta \cap \mathcal{E}_t^\delta(\theta) \neq \emptyset}.
\end{equation}
After obtaining the solution $\theta_t^p$ via computing the optimization problem~\eqref{Faury:projection}, the third step is to select $\thetat_t$ from $\Theta \cap \mathcal{E}_t^\delta(\theta_t^p)$. Based on this projection step,~\citet{arXiv'21:faury-driftingGLB} can separate the bias and variance parts before projection as follows,
\begin{align}\nonumber
    |\mu(\x^\T\thetat_t) - \mu(\x^\T\theta_t)| \leq {}& k_\mu |\x^\T(\thetat_t - \theta_t)|\\
      = {}& k_\mu |\x^\T G_t^{-1}(\theta_t,\thetat_t)(g_t(\thetat_t) - g_t(\theta_t))|\\
      \leq {}& k_\mu |\x^\T G_t^{-1}(\theta_t,\thetat_t)(g_t(\thetat_t) - g_t(\theta_t^p)+ g_t(\theta_t^p)- g_t(\thetah_t)+g_t(\thetah_t)-g_t(\thetab_t)+g_t(\thetab_t) - g_t(\theta_t))|\\
      \leq {}& \underbrace{k_\mu |\x^\T G_t^{-1}(\theta_t,\thetat_t)(g_t(\thetat_t) - g_t(\theta_t^p)+g_t(\thetah_t)-g_t(\thetab_t)|}_{\bias}\\
      {}&\qquad+\underbrace{k_\mu |\x^\T G_t^{-1}(\theta_t,\thetat_t)(g_t(\theta_t^p)- g_t(\thetah_t)+g_t(\thetab_t) - g_t(\theta_t))|}_{\variance}.
  \end{align}
Their bias-variance decomposition motivates the choice of \emph{different} local norms for bounding bias and variance parts in their algorithm and analysis. Notably, due to the complications of the projection step (see~\eqref{Faury:confidence-set} and~\eqref{Faury:projection}), the overall algorithm is fairly complicated and less attractive for practical implementations, and moreover, it needs to maintain two covariance matrices $V_t$ and $\Vt_t$ (due to the constructed confidence region~\eqref{Faury:confidence-set}) at each round $t$ during the algorithm running. In the next section, we will show that the simple projection used in the stationary GLB~\eqref{eq:simple-projection} can be sufficient for coping with the drifting GLB via our refined analysis framework.

\subsection{Proof of Lemma~\ref{lemma:GLB-estimation-error}}
\label{sec:GLB-estimation-error-proof}
\begin{proof}
  Base on the estimator equation~\eqref{eq:GLB-estimator}, we know that 
  \begin{equation}
  \label{eq:GLB-gt-thetah}
      g_t(\thetah_t) = \lambda c_\mu \thetah_t + \sum_{s=1}^{t-1}\gamma^{t-s-1}\mu(X_s^\T \thetah_t)X_s = \sum_{s=1}^{t-1}\gamma^{t-s-1}r_sX_s,
  \end{equation}
  and then by the mean value theorem, we know that 
  \begin{equation}
  \label{eq:GLB-gt-mvt}
      g_t(\theta_1) - g_t(\theta_2)= G_t(\theta_1, \theta_2)(\theta_1 - \theta_2),
  \end{equation}
  where $G_t(\theta_1, \theta_2) \triangleq \int_{0}^1 \nabla g_t(s\theta_2+(1-s)\theta_1)\diff{s}\in \R^{d\times d}$. Notice that for any $\theta \in \Theta$, the gradient of $g_t$ is 
  \begin{equation}\nonumber
  \label{eq:GLB-gt-gradient}
      \nabla g_t(\theta) = \lambda c_\mu I_d + \sum_{s=1}^{t-1} \gamma^{t-s-1}\dmu(X_s^\T \theta)X_sX_s^\T \succeq c_\mu V_{t-1},
  \end{equation}
  which clearly implies $\forall \theta_1, \theta_2 \in \Theta, G_t(\theta_1,\theta_2)\succeq c_\mu V_{t-1}$.

  By Assumption~\ref{ass:link-function}, the mean value theorem~\eqref{eq:GLB-gt-mvt} on $g_t$ and the projection~\eqref{eq:GLB-projection}, we have 
  \begin{equation}\nonumber
  \begin{split}|\mu(\x^\T\thetat_t) - \mu(\x^\T\theta_t)| \leq {}& k_\mu |\x^\T(\thetat_t - \theta_t)|\\
      = {}& k_\mu |\x^\T G_t^{-1}(\theta_t,\thetat_t)(g_t(\thetat_t) - g_t(\theta_t))|\\
      \leq {}& k_\mu \|\x\|_{G_t^{-1}(\theta_t,\thetat_t)}\|g_t(\thetat_t) - g_t(\theta_t)\|_{G_t^{-1}(\theta_t,\thetat_t)}\\
      \leq {}& \frac{k_\mu}{c_\mu} \|\x\|_{V_{t-1}^{-1}}\|g_t(\thetat_t) - g_t(\theta_t)\|_{V_{t-1}^{-1}}\\
      \leq {}& \frac{2k_\mu}{c_\mu} \|\x\|_{V_{t-1}^{-1}}\|g_t(\thetah_t) - g_t(\theta_t)\|_{V_{t-1}^{-1}},
  \end{split}
  \end{equation}
  then based on the model assumption, the function $g_t$~\eqref{eq:GLB-gt} and $g_t(\thetah_t)$~\eqref{eq:GLB-gt-thetah}, we have,
  \begin{align}\nonumber
   g_t(\theta_t)-g_t(\thetah_t)= {}& \lambda c_\mu\theta_t +\sum_{s=1}^{t-1}\gamma^{t-s-1}\mu(X_s^\T\theta_t)X_s-\sum_{s=1}^{t-1}\gamma^{t-s-1}r_sX_s\\
  = {}&\lambda c_\mu\theta_t +\sum_{s=1}^{t-1}\gamma^{t-s-1}\mu(X_s^\T\theta_t)X_s-\sum_{s=1}^{t-1}\gamma^{t-s-1}(\mu(X_s^\T \theta_s) + \eta_s)X_s\\
  = {}&\underbrace{\sum_{s=1}^{t-1}\gamma^{t-s-1}(\mu(X_s^\T\theta_t) - \mu(X_s^\T \theta_s) )X_s}_{\bias} + \underbrace{\lambda c_\mu \theta_t -\sum_{s=1}^{t-1}\gamma^{t-s-1}\eta_sX_s}_{\variance}.
  \end{align}
  Then, by the Cauchy-Schwarz inequality, we know that for any $\x\in\X$,
  \begin{equation}
  \begin{split}
  \label{eq:GLB-bound-cauchy}
      \abs{\mu(\x^\T\thetat_t) - \mu(\x^\T\theta_t)} \leq \frac{2k_\mu}{c_\mu}\|\x\|_{V_{t-1}^{-1}}\sbr{C_t + D_t},
  \end{split}
  \end{equation}
  where 
  \begin{equation}\nonumber
  \begin{split}
      C_t = \norm{\sum_{s=1}^{t-1}\gamma^{t-s-1}(\mu(X_s^\T\theta_t) - \mu(X_s^\T \theta_s) )X_s}_{V_{t-1}^{-1}},\quad D_t = \norm{\sum_{s=1}^{t-1}\gamma^{t-s-1}\eta_sX_s-\lambda c_\mu \theta_t}_{V_{t-1}^{-1}}.
  \end{split}
  \end{equation}
  This two terms can be bounded separately, as summarized in the following lemmas.
  \begin{Lemma}
    \label{lemma:GLB-C_t-bound}
    For any $t \in [T]$, we have 
    \begin{equation}
        \norm{\sum_{s=1}^{t-1}\gamma^{t-s-1}(\mu(X_s^\T\theta_t) - \mu(X_s^\T \theta_s) )X_s}_{V_{t-1}^{-1}}\leq Lk_\mu\sqrt{d}\sum_{p=1}^{t-1}\gamma^{\frac{t-1}{2}}\sqrt{\frac{\gamma^{-p}-1}{1-\gamma}} \norm{\theta_p - \theta_{p+1}}_2.
    \end{equation}
    \end{Lemma}
    
    \begin{Lemma}
    \label{lemma:GLB-D_t-bound}
    For any $\delta \in (0,1)$, with probability at least $1-\delta$, the following holds for all $t \in [T]$,
    \begin{equation}
        \norm{\sum_{s=1}^{t-1}\gamma^{t-s-1}\eta_sX_s-\lambda c_\mu \theta_t}_{V_{t-1}^{-1}} \leq \sqrt{\lambda}c_\mu S+R\sqrt{2\log\frac{1}{\delta}+d\log\sbr{1+\frac{L^2 (1-\gamma^{2t-2})}{\lambda d(1-\gamma^2)}}}.
    \end{equation}
    \end{Lemma}

    Based on the inequality~\eqref{eq:GLB-bound-cauchy}, Lemma~\ref{lemma:GLB-C_t-bound}, Lemma~\ref{lemma:GLB-D_t-bound}, and the boundedness assumption of the feasible set, we have for any $\x \in \X$, $\gamma \in (0,1)$, $\delta \in (0,1)$, with probability at least $1-\delta$, the following holds for all $t \in [T]$,
    \begin{equation}\nonumber
    \begin{split}
        \abs{\mu(\x^\T\thetat_t) - \mu(\x^\T\theta_t)}\leq {}& \frac{2k_\mu}{c_\mu}\|\x\|_{V_{t-1}^{-1}}\sbr{ Lk_\mu\sqrt{d}\sum_{p=1}^{t-1}\gamma^{\frac{t-1}{2}}\sqrt{\frac{\gamma^{-p}-1}{1-\gamma}} \norm{\theta_p - \theta_{p+1}}_2 + \betab_{t-1} }\\
        \leq {}& \frac{2k_\mu}{c_\mu}\sbr{ L^2k_\mu\sqrt{\frac{d}{\lambda}}\sum_{p=1}^{t-1}\gamma^{\frac{t-1}{2}}\sqrt{\frac{\gamma^{-p}-1}{1-\gamma}} \norm{\theta_p - \theta_{p+1}}_2 +\betab_{t-1} \|\x\|_{V_{t-1}^{-1}} },
    \end{split}
    \end{equation}
    where $\betab_t \triangleq \sqrt{\lambda}c_\mu S+R\sqrt{2\log\frac{1}{\delta}+d\log\sbr{1+\frac{L^2 (1-\gamma^{2t})}{\lambda d(1-\gamma^2)}}}$ is the confidence radius used in \GLBweightours. Hence we complete the proof.
\end{proof}

\begin{proof}[{Proof of Lemma~\ref{lemma:GLB-C_t-bound}}]

  Here we need to extract the variations of the time-varying parameter $\theta_t$
  \begin{align*}
  \norm{\sum_{s=1}^{t-1}\gamma^{t-s-1}(\mu(X_s^\T\theta_t) - \mu(X_s^\T \theta_s))X_s}_{V_{t-1}^{-1}}
  \leq {}& \norm{\sum_{s=1}^{t-1}\gamma^{t-s-1}\sum_{p=s}^{t-1}(\mu(X_s^\T\theta_p) - \mu(X_s^\T \theta_{p+1}) )X_s}_{V_{t-1}^{-1}}\\
  = {}& \norm{\sum_{p=1}^{t-1}\sum_{s=1}^{p}\gamma^{t-s-1}\alpha(X_s, \theta_p, \theta_{p+1})(X_s^\T\theta_p - X_s^\T \theta_{p+1})X_s}_{V_{t-1}^{-1}}\\
  = {}& \norm{\sum_{p=1}^{t-1}\sum_{s=1}^{p}\gamma^{t-s-1}\alpha(X_s, \theta_p, \theta_{p+1})X_sX_s^\T(\theta_p - \theta_{p+1})}_{V_{t-1}^{-1}}\\
  \leq {}& \sum_{p=1}^{t-1}\norm{\sum_{s=1}^{p}\gamma^{t-s-1}\alpha(X_s, \theta_p, \theta_{p+1})X_s\|X_s\|_2\|\theta_p - \theta_{p+1}\|_2}_{V_{t-1}^{-1}}\\
  \leq {}& L\sum_{p=1}^{t-1}\sum_{s=1}^{p}\gamma^{t-s-1}|\alpha(X_s, \theta_p, \theta_{p+1})| \norm{X_s}_{V_{t-1}^{-1}} \|\theta_p - \theta_{p+1}\|_2\\
  \leq {}& Lk_\mu\sum_{p=1}^{t-1}\sum_{s=1}^{p}\gamma^{t-s-1} \norm{X_s}_{V_{t-1}^{-1}} \|\theta_p - \theta_{p+1}\|_2.
  \end{align*}
  where the forth equation is due to the mean value theorem where $\alpha(\x, \theta_1, \theta_2) = \int_{0}^1 \dmu(v\x^\T\theta_2+(1-v)x^\T\theta_1)\diff{v}$:
  $$\mu(X_s^\T\theta_p) - \mu(X_s^\T \theta_{p+1}) = \alpha(X_s, \theta_p, \theta_{p+1})(X_s^\T\theta_p - X_s^\T \theta_{p+1}).$$

  Next, the derivation for the bound of term $\sum_{s=1}^{p}\gamma^{t-s-1} \norm{X_s}_{V_{t-1}^{-1}}$ is the same as the inequality~\eqref{eq:LB-At-gammapart} in \ref{sec:LB-A_t-bound-proof}, hence we complete the proof.
\end{proof}

\begin{proof}[{Proof of Lemma~\ref{lemma:GLB-D_t-bound}}]
  Same as the linear case, we need to use $\Vt_t = \lambda I_d + \sum_{s=1}^{t} \gamma^{2(t-s)} X_s X_s^\T $.
  \begin{align*}
    D_t ={}& \norm{\sum_{s=1}^{t-1}\gamma^{t-s-1}\eta_sX_s-\lambda c_\mu \theta_t}_{V_{t-1}^{-1}}\\
    \leq {}& \norm{\sum_{s=1}^{t-1}\gamma^{t-s-1}\eta_sX_s}_{V_{t-1}^{-1}}+\norm{\lambda c_\mu \theta_t}_{V_{t-1}^{-1}}\\
    \leq {}&   \norm{ \sum_{s=1}^{t-1}\gamma^{t-s-1}\eta_sX_s}_{\Vt_{t-1}^{-1}} +\sqrt{\lambda}c_\mu S\\
    \leq {}& R\sqrt{2\log\frac{1}{\delta}+d\log\sbr{1+\frac{L^2 (1-\gamma^{2t-2})}{\lambda d(1-\gamma^2)}}}+\sqrt{\lambda}c_\mu S.
  \end{align*}
  Again, we emphasize that the $\Vt_t$ is introduced into analysis \emph{only}. The proof here is the same as the proof of Lemma~\ref{lemma:LB-B_t-bound} in \ref{sec:LB-B_t-bound-proof}, the only difference is an extra $c_\mu$ in the second term.
\end{proof}
\subsection{Proof of~\pref{thm:GLB-regret}}
\label{sec:GLB-regret-proof}
\begin{proof}
  Let $X_t^* \triangleq \argmax_{\x\in \X}\mu(\x^\T \theta_t)$. Due to Lemma~\ref{lemma:GLB-estimation-error} and the fact that $X_t^*,X_t\in \X$, each of the following holds with probability at least $1-\delta$,
  \begin{equation}\nonumber
    \begin{split}
      \forall t \in [T], \mu(X_t^{*\T}\theta_t) \leq{}&  \mu(X_t^{*\T}\thetat_t) +\frac{2k_\mu}{c_\mu}\sbr{ L^2k_\mu\sqrt{\frac{d}{\lambda}}\sum_{p=1}^{t-1}\gamma^{\frac{t-1}{2}}\sqrt{\frac{\gamma^{-p}-1}{1-\gamma}} \norm{\theta_p - \theta_{p+1}}_2 + \betab_{t-1}\|X_t^*\|_{V_{t-1}^{-1}} },\\
      \forall t \in [T], \mu(X_t^{\T}\theta_t) \geq{}&  \mu(X_t^{\T}\thetat_t) -\frac{2k_\mu}{c_\mu}\sbr{ L^2k_\mu\sqrt{\frac{d}{\lambda}}\sum_{p=1}^{t-1}\gamma^{\frac{t-1}{2}}\sqrt{\frac{\gamma^{-p}-1}{1-\gamma}} \norm{\theta_p - \theta_{p+1}}_2 + \betab_{t-1}\|X_t\|_{V_{t-1}^{-1}} }.
    \end{split}
  \end{equation}
  By the union bound, the following holds with probability at least $1-2\delta$: $\forall t\in [T]$
  \begin{equation}\nonumber
  \begin{split}
      {}& \mu(X_t^{*\T}\theta_t) - \mu(X_t^{\T}\theta_t) \\
      \leq{}&  \mu(X_t^{*\T}\thetat_t) -\mu(X_t^{\T}\thetat_t) + \frac{4L^2k_\mu^2}{c_\mu}\sqrt{\frac{d}{\lambda}}\sum_{p=1}^{t-1}\gamma^{\frac{t-1}{2}}\sqrt{\frac{\gamma^{-p}-1}{1-\gamma}} \norm{\theta_p - \theta_{p+1}}_2 + \frac{2k_\mu}{c_\mu}\sbr{\betab_{t-1}\|X_t^*\|_{V_{t-1}^{-1}}+\betab_{t-1}\|X_t\|_{V_{t-1}^{-1}} }\\
      \leq{}& \frac{4L^2k_\mu^2}{c_\mu}\sqrt{\frac{d}{\lambda}}\sum_{p=1}^{t-1}\gamma^{\frac{t-1}{2}}\sqrt{\frac{\gamma^{-p}-1}{1-\gamma}} \norm{\theta_p - \theta_{p+1}}_2 + \frac{4k_\mu}{c_\mu}\betab_{t-1}\|X_t\|_{V_{t-1}^{-1}},
  \end{split}
  \end{equation}
  where the last step comes from the arm selection criterion~\eqref{eq:GLB-select-criteria} such that 
  $$\mu(X_t^{*\T} \thetat_t)+ \frac{2k_\mu}{c_\mu}\betab_{t-1}\|X_t^*\|_{V_{t-1}^{-1}} \leq \mu(X_t^{\T} \thetat_t)+ \frac{2k_\mu}{c_\mu}\betab_{t-1}\|X_t\|_{V_{t-1}^{-1}}. $$
  Hence the following dynamic regret bound holds with probability at least $1-2\delta$ and can be divided into two parts,
  \begin{align}\nonumber
      \DReg_T = {}& \sum_{t=1}^T \max_{\x\in \X}\mu(\x^\T \theta_t) - \mu(X_t^\T \theta_t)\\
      \leq {}& \underbrace{\frac{4L^2k_\mu^2}{c_\mu}\sqrt{\frac{d}{\lambda}}\sum_{t=1}^T\sum_{p=1}^{t-1}\gamma^{\frac{t-1}{2}}\sqrt{\frac{\gamma^{-p}-1}{1-\gamma}} \norm{\theta_p - \theta_{p+1}}_2}_{\bias} + \underbrace{\frac{4k_\mu}{c_\mu}\betab_T\sum_{t=1}^T\|X_t\|_{V_{t-1}^{-1}}}_{\variance} .
  \end{align}
  where $\betab_{t}= \sqrt{\lambda}c_\mu S+R\sqrt{2\log\frac{1}{\delta}+d\log\sbr{1+\frac{L^2 (1-\gamma^{2t})}{\lambda d(1-\gamma^2)}}}$ is the confidence radius.

  Now we derive the upper bound for these two parts separately.

  \tb{Bias Part.}
  Similar to the proof of inequality~\eqref{eq:LB-regret-bias-bound}, we have
  \begin{align*}
    \frac{4L^2k_\mu^2}{c_\mu}\sqrt{\frac{d}{\lambda}}\sum_{t=1}^T \sum_{p=1}^{t-1}   \gamma^{\frac{t-1}{2}}\sqrt{\frac{\gamma^{-p}-1}{1-\gamma}}\norm{\theta_p -\theta_{p+1}}_2 \leq \frac{8L^2k_\mu^2}{c_\mu}\sqrt{\frac{d}{\lambda}}\frac{1}{(1-\gamma)^{\sfrac{3}{2}}}P_T.
  \end{align*}

  \tb{Variance Part.}
  Similar to the proof of inequality~\eqref{eq:LB-regret-variance-bound}, we have
  \begin{equation}\nonumber
  \begin{split}
      \frac{4k_\mu}{c_\mu}\betab_T\sqrt{T}\sqrt{\sum_{t=1}^T\|X_t\|_{V_{t-1}^{-1}}^2} \leq {}& \frac{4k_\mu}{c_\mu}\betab_T\sqrt{2\max\{1,L^2/\lambda\}dT}\sqrt{ T\log\frac{1}{\gamma}+\log\sbr{1+ \frac{L^2}{\lambda d(1-\gamma)}}}.
  \end{split}
  \end{equation}
  Combine the upper bound for the bias and variance parts, and let $\delta = 1/(2T^2)$, we have the following regret bound with probability at least $1-1/T$,
  \begin{equation}\nonumber
      \DReg_T \leq \frac{8L^2k_\mu^2}{c_\mu}\sqrt{\frac{d}{\lambda}}\frac{1}{(1-\gamma)^{\sfrac{3}{2}}}P_T+ \frac{4k_\mu}{c_\mu}\betab_T\sqrt{2\max\{1,L^2/\lambda\}dT}\sqrt{ T\log\frac{1}{\gamma}+\log\sbr{1+ \frac{L^2}{\lambda d(1-\gamma)}}}.
  \end{equation}
  where $\betab_{t}= \sqrt{\lambda}c_\mu S+R\sqrt{4\log T+2\log 2+d\log\sbr{1+\frac{L^2 (1-\gamma^{2t})}{\lambda d(1-\gamma^2)}}}$. We set $\gamma \geq 1/T$ and $\lambda = d/c_\mu^2$, and obtain that,
  \begin{equation}\nonumber
  \begin{split}
      \DReg_T\leq{}& \Ot\sbr{k_\mu^2\frac{1}{(1-\gamma)^{\sfrac{3}{2}}}P_T + \frac{k_\mu}{c_\mu}d(1-\gamma)^{\sfrac{1}{2}}T}.
  \end{split}
  \end{equation}
  When $P_T< d/(k_\mu c_\mu T)$, we set $\gamma = 1-1/T$ and achieve an $\Ot(k_\mu c_\mu^{-1} d\sqrt{T})$ regret bound. When $P_T\geq d/(k_\mu c_\mu T)$, We set $\gamma$ optimally as $1-\gamma = \sqrt{k_\mu c_\mu P_T/(dT)}$ and attain an $\Ot(k_\mu^{\sfrac{5}{4}}c_\mu^{-\sfrac{3}{4}}d^{\sfrac{3}{4}}P_T^{\sfrac{1}{4}}T^{\sfrac{3}{4}})$ regret bound. Notice that, if $k_\mu < 1$, we just let $1-\gamma = \sqrt{c_\mu P_T/(dT)}$ and the regret bound becomes $\Ot(k_\mu^{2}c_\mu^{-\sfrac{3}{4}}d^{\sfrac{3}{4}}P_T^{\sfrac{1}{4}}T^{\sfrac{3}{4}})$.
\end{proof}
\section{Analysis of \SCBweightours}
\label{sec:SCB-regret}
In this section, we first present \SCBweightours algorithm in Algorithm~\ref{alg:SCB-WeightUCB}, Then, in Appendix~\ref{sec:SCB-estimation-error-proof} we present the proof of the estimation error upper bound of our \SCBweightours algorithm (Lemma~\ref{lemma:SCB-estimation-error}). Finally, in Appendix~\ref{sec:SCB-regret-proof}, we provide the proof of dynamic regret upper bound (Theorem~\ref{thm:SCB-regret}).
\begin{algorithm}[!t]
  \caption{\SCBweightours}
  \label{alg:SCB-WeightUCB}
\begin{algorithmic}[1]
\REQUIRE time horizon $T$, discounted factor $\gamma$, confidence $\delta$, regularizer $\lambda$, inverse link function $\mu$, parameters $S$, $L$ and $m$\\
\STATE Set $V_0 = \lambda I_d$, $\thetah_1 = \mathbf{0}$ and compute $k_\mu$ and $c_\mu$
\FOR{$t = 1,2,...,T$}
  \IF{$\|\thetah_t\|_2\leq S$} 
  \STATE let $\thetat_t = \thetah_t$
  \ELSE 
  \STATE Do the projection and get $\thetat_t$ by~\eqref{eq:SCB-projection}
  \ENDIF
  \STATE Compute $\betat_{t-1}$ by~\eqref{eq:SCB-confidence-radius}
  \STATE Select $X_t$ by~\eqref{eq:SCB-select-criteria}
  \STATE Receive the reward $r_t$
  \STATE Update $V_{t} = \gamma V_{t-1} + X_t X_t^\T +(1-\gamma)\lambda I_d$
  \STATE Compute $\thetah_{t+1}$ according to~\eqref{eq:GLB-estimator}
\ENDFOR
\end{algorithmic}
\end{algorithm}

\subsection{Proof of Lemma~\ref{lemma:SCB-estimation-error}}
\label{sec:SCB-estimation-error-proof}

\begin{proof}
  Base on the estimator equation~\eqref{eq:GLB-estimator}, we know that 
  \begin{equation}
  \label{eq:SCB-gt-thetah}
      g_t(\thetah_t) = \lambda c_\mu \thetah_t + \sum_{s=1}^{t-1}\gamma^{t-s-1}\mu(X_s^\T \thetah_t)X_s = \sum_{s=1}^{t-1}\gamma^{t-s-1}r_sX_s,
  \end{equation}
  and then by the mean value theorem, we know that 
  \begin{equation}
  \label{eq:SCB-gt-mvt}
      g_t(\theta_1) - g_t(\theta_2)= G_t(\theta_1, \theta_2)(\theta_1 - \theta_2),
  \end{equation}
  where $G_t(\theta_1, \theta_2) \triangleq \int_{0}^1 \nabla g_t(s\theta_2+(1-s)\theta_1)\diff{s}\in \R^{d\times d}$. Notice that for any $\theta \in \Theta$, the gradient of $g_t$ is 
  \begin{equation}\nonumber
  \label{eq:SCB-gt-gradient}
      \nabla g_t(\theta) = \lambda c_\mu I_d + \sum_{s=1}^{t-1} \gamma^{t-s-1}\dmu(X_s^\T \theta)X_sX_s^\T \succeq c_\mu V_{t-1},
  \end{equation}
  which clearly implies $\forall \theta_1, \theta_2 \in \Theta, G_t(\theta_1,\theta_2)\succeq c_\mu V_{t-1}$ and $\forall \theta, H_t(\theta)\succeq c_\mu V_{t-1}$, where $H_t(\theta)$ is defined as
  \begin{equation}
    \label{eq:SCB-H}
    H_t(\theta) \teq \lambda c_\mu I_d + \sum_{s=1}^{t-1}\gamma^{t-s-1}\dmu(X_s^\T \theta)X_sX_s^\T.
    \end{equation}

  By Assumption~\ref{ass:link-function}, the mean value theorem~\eqref{eq:GLB-gt-mvt} on $g_t$, the projection~\eqref{eq:SCB-projection} and Lemma~\ref{lemma:SCB-G-H}, we have 
  \begin{equation}\nonumber
  \begin{split}|\mu(\x^\T\thetat_t) - \mu(\x^\T\theta_t)| \leq {}& k_\mu |\x^\T(\thetat_t - \theta_t)|\\
      = {}& k_\mu |\x^\T G_t^{-1}(\theta_t,\thetat_t)(g_t(\thetat_t) - g_t(\theta_t))|\\
      \leq {}& k_\mu \|\x\|_{G_t^{-1}(\theta_t,\thetat_t)}\|g_t(\thetat_t) - g_t(\theta_t)\|_{G_t^{-1}(\theta_t,\thetat_t)}\\
      \leq {}& k_\mu \|\x\|_{G_t^{-1}(\theta_t,\thetat_t)}\sbr{\|g_t(\thetat_t) - g_t(\thetah_t)\|_{G_t^{-1}(\theta_t,\thetat_t)}+\|g_t(\thetah_t) - g_t(\theta_t)\|_{G_t^{-1}(\theta_t,\thetat_t)}}\\
      \leq {}& \sqrt{1+2S}k_\mu \|\x\|_{G_t^{-1}(\theta_t,\thetat_t)}\sbr{\|g_t(\thetat_t) - g_t(\thetah_t)\|_{H_t^{-1}(\thetat_t)}+\|g_t(\thetah_t) - g_t(\theta_t)\|_{H_t^{-1}(\theta_t)}}\\
      \leq {}& 2\sqrt{1+2S}\frac{k_\mu}{\sqrt{c_\mu}} \|\x\|_{V_{t-1}^{-1}}\|g_t(\thetah_t) - g_t(\theta_t)\|_{H_t^{-1}(\theta_t)},\\
  \end{split}
  \end{equation}
  then based on the model assumption~\eqref{eq:SCB-model-assume}, the function $g_t$~\eqref{eq:GLB-gt} and the $g_t(\thetah_t)$~\eqref{eq:SCB-gt-thetah}, we have,
  \begin{equation}\nonumber
  \begin{split}
   g_t(\theta_t)-g_t(\thetah_t)= {}& \lambda c_\mu\theta_t +\sum_{s=1}^{t-1}\gamma^{t-s-1}\mu(X_s^\T\theta_t)X_s-\sum_{s=1}^{t-1}\gamma^{t-s-1}r_sX_s\\
  = {}&\lambda c_\mu\theta_t +\sum_{s=1}^{t-1}\gamma^{t-s-1}\mu(X_s^\T\theta_t)X_s-\sum_{s=1}^{t-1}\gamma^{t-s-1}(\mu(X_s^\T \theta_s) + \eta_s)X_s\\
  = {}&\sum_{s=1}^{t-1}\gamma^{t-s-1}(\mu(X_s^\T\theta_t) - \mu(X_s^\T \theta_s) )X_s + \lambda c_\mu \theta_t -\sum_{s=1}^{t-1}\gamma^{t-s-1}\eta_sX_s,
  \end{split}
  \end{equation}
  then, by Cauchy-Schwarz inequality, we have 
  \begin{equation}
  \begin{split}
  \label{eq:SCB-bound-cauchy}
      \abs{\mu(\x^\T\thetat_t) - \mu(\x^\T\theta_t)} \leq 2\sqrt{1+2S}\frac{k_\mu}{\sqrt{c_\mu}} \|\x\|_{V_{t-1}^{-1}}\sbr{E_t + F_t},
  \end{split}
  \end{equation}
  where 
  \begin{equation}\nonumber
  \begin{split}
      E_t = \norm{\sum_{s=1}^{t-1}\gamma^{t-s-1}(\mu(X_s^\T\theta_t) - \mu(X_s^\T \theta_s) )X_s}_{H_t^{-1}(\theta_t)},\quad F_t = \norm{\sum_{s=1}^{t-1}\gamma^{t-s-1}\eta_sX_s-\lambda c_\mu \theta_t}_{H_t^{-1}(\theta_t)}.
  \end{split}
  \end{equation}
  This two terms can be bounded separately.

  \begin{Lemma}
    \label{lemma:SCB-E_t-bound}
    For any $t \in [T]$, we have 
    \begin{equation}
        \norm{\sum_{s=1}^{t-1}\gamma^{t-s-1}(\mu(X_s^\T\theta_t) - \mu(X_s^\T \theta_s) )X_s}_{H_t^{-1}(\theta_t)}\leq L\frac{k_\mu}{\sqrt{c_\mu}}\sqrt{d}\sum_{p=1}^{t-1}\gamma^{\frac{t-1}{2}}\sqrt{\frac{\gamma^{-p}-1}{1-\gamma}} \norm{\theta_p - \theta_{p+1}}_2.
    \end{equation}
    \end{Lemma}
    
    \begin{Lemma}
    \label{lemma:SCB-F_t-bound}
    For any  $\delta \in (0,1)$, with probability at least $1-\delta$, we have for all $t \in [T]$,
    \begin{equation}
    \begin{split}
        \norm{\sum_{s=1}^{t-1}\gamma^{t-s-1}\eta_sX_s-\lambda c_\mu \theta_t}_{H_t^{-1}(\theta_t)} &\leq \frac{\sqrt{\lambda c_\mu}}{2 m }+\frac{2 m }{\sqrt{\lambda c_\mu}}\log\frac{1}{\delta}+\frac{d m }{\sqrt{\lambda c_\mu}} \log \sbr{1+\frac{ L^2k_\mu(1-\gamma^{2t-2})}{\lambda c_\mu d(1-\gamma^{2})}}\\
        &\qquad\qquad\qquad\qquad\qquad\qquad\qquad\qquad+\frac{2 m }{\sqrt{\lambda c_\mu}} d \log (2)+\sqrt{\lambda c_\mu} S,
    \end{split}
    \end{equation}
    \end{Lemma}
    Based on the inequality~\eqref{eq:SCB-bound-cauchy}, Lemma~\ref{lemma:GLB-C_t-bound} and Lemma~\ref{lemma:GLB-D_t-bound}, and the boundedness assumption of the feasible set, we have for any $\x \in \X$,  $\delta \in (0,1)$, with probability at least $1-\delta$, we have for all $t \in [T]$,
    \begin{equation}\nonumber
    \begin{split}
        \abs{\mu(\x^\T\thetat_t) - \mu(\x^\T\theta_t)}\leq {}& 2\sqrt{1+2S}\frac{k_\mu}{\sqrt{c_\mu}}\|\x\|_{V_{t-1}^{-1}}\sbr{L\frac{k_\mu}{\sqrt{c_\mu}}\sqrt{d}\sum_{p=1}^{t-1}\gamma^{\frac{t-1}{2}}\sqrt{\frac{\gamma^{-p}-1}{1-\gamma}} \norm{\theta_p - \theta_{p+1}}_2+ \betat_{t-1} },\\
        \leq {}& 2\sqrt{1+2S}\frac{k_\mu}{\sqrt{c_\mu}}\sbr{ L^2\frac{k_\mu}{\sqrt{\lambda c_\mu}}\sqrt{d}\sum_{p=1}^{t-1}\gamma^{\frac{t-1}{2}}\sqrt{\frac{\gamma^{-p}-1}{1-\gamma}} \norm{\theta_p - \theta_{p+1}}_2+ \betat_{t-1}\|\x\|_{V_{t-1}^{-1}} },\\
    \end{split}
    \end{equation}
    where $\betat_t \teq \frac{\sqrt{\lambda c_\mu}}{2 m }+\frac{2 m }{\sqrt{\lambda c_\mu}}\log\frac{1}{\delta}+\frac{d m }{\sqrt{\lambda c_\mu}} \log \sbr{1+\frac{ L^2k_\mu(1-\gamma^{2t})}{\lambda c_\mu d(1-\gamma^{2})}}+\frac{2 m }{\sqrt{\lambda c_\mu}} d \log (2)+\sqrt{\lambda c_\mu} S$ is the confidence radius used in \SCBweightours. Hence we completes the proof.
\end{proof}

\begin{proof}[{Proof of Lemma~\ref{lemma:SCB-E_t-bound}}]
  Since $\forall \theta, H_t(\theta)\succeq c_\mu V_{t-1}$, we have 
  \begin{equation}\nonumber
    \begin{split}
        \norm{\sum_{s=1}^{t-1}\gamma^{t-s-1}(\mu(X_s^\T\theta_t) - \mu(X_s^\T \theta_s) )X_s}_{H_t^{-1}(\theta_t)} \leq \frac{1}{\sqrt{c_\mu}}\norm{\sum_{s=1}^{t-1}\gamma^{t-s-1}(\mu(X_s^\T\theta_t) - \mu(X_s^\T \theta_s) )X_s}_{V_{t-1}^{-1}}.
    \end{split}
    \end{equation}
    Then use Lemma~\ref{lemma:GLB-C_t-bound} and we complete the proof.
\end{proof}

\begin{proof}[{Proof of Lemma~\ref{lemma:SCB-F_t-bound}}]
  Let $\Ht_t(\theta) \teq \lambda c_\mu \gamma^{-2(t-1)}I_d + \sum_{s=1}^{t-1}\gamma^{-2s}\dmu(X_s^\T \theta)X_sX_s^\T$ which is \emph{only} used in the analysis.
  \begin{align*}
    F_t &= \norm{\sum_{s=1}^{t-1}\gamma^{t-s-1}\eta_sX_s-\lambda c_\mu \theta_t}_{H_t^{-1}(\theta_t)}\leq \norm{\sum_{s=1}^{t-1}\gamma^{t-s-1}\eta_sX_s}_{H_t^{-1}(\theta_t)}+\norm{\lambda c_\mu \theta_t}_{H_t^{-1}(\theta_t)}\\
    &\leq \norm{ \sum_{s=1}^{t-1}\gamma^{-s}\eta_sX_s}_{\Ht_t^{-1}(\theta_t)} +\sqrt{\lambda c_\mu} S.
  \end{align*}
  Recall that $H_t(\theta) = \lambda c_\mu I_d + \sum_{s=1}^{t-1}\gamma^{t-s-1}\dmu(X_s^\T \theta)X_sX_s^\T$, so the last inequality comes from
  \begin{equation}
    \label{eq:H-Ht}
  \begin{split}
    \gamma^{-2(t-1)}H_t(\theta) = \lambda c_\mu \gamma^{-2(t-1)}I_d + \sum_{s=1}^{t-1}\gamma^{-t-s+1}\dmu(X_s^\T \theta)X_sX_s^\T \succeq \lambda c_\mu \gamma^{-2(t-1)}I_d + \sum_{s=1}^{t-1}\gamma^{-2s}\dmu(X_s^\T \theta)X_sX_s^\T = \Ht_t^{-1}(\theta).
  \end{split}
  \end{equation}
  From the weighted version concentration inequality~\citep[Theorem 3]{AISTATS'21:SCB-forgetting}, restated in~\pref{thm:self-normalized-weight-SCB}, we can get the bound for the first term $\|\sum_{s=1}^{t-1}\gamma^{-s}\eta_sX_s\|_{\Ht_t^{-1}(\theta_t)}$. First by the model assumption~\eqref{eq:SCB-model-assume}, we know that $\sigma_t^2 = \E[\eta_t^2|\F_t] = \Var[r_t|\F_t] = \ddmu(X_t\theta_t)$, then just let $w_t = \gamma^{-t}, \lambda_{t}= \lambda c_\mu \gamma^{-2t}$ and we have,
  \begin{equation}\nonumber
    \begin{split}
      \norm{ \sum_{s=1}^{t-1}\gamma^{-s}\eta_sX_s}_{\Ht_t^{-1}(\theta_t)} \leq \frac{\sqrt{\lambda c_\mu}}{2 m }+\frac{2 m }{\sqrt{\lambda c_\mu}} \log \sbr{\frac{\det(\Ht_t)^{1 / 2}}{\delta (\lambda c_\mu \gamma^{-2(t-1)})^{d / 2}}}+\frac{2 m }{\sqrt{\lambda c_\mu}} d \log (2).
    \end{split}
  \end{equation}

  Then use Lemma~\ref{lemma:det-inequality} and let $w_{t,s} = \gamma^{-2s} \dmu(X_s^\T\theta_t), \lambda_t=  \lambda c_\mu \gamma^{-2(t-1)}$, we get the upper bound for $\det(\Ht_t)$,
  \begin{equation}\nonumber
    \begin{split}
      \det(\Ht_t)  &\leq \sbr{\lambda c_\mu \gamma^{-2(t-1)} + \frac{L^2k_\mu\sum_{s=1}^{t-1} \gamma^{-2s} }{d}}^d,
    \end{split}
  \end{equation}
  then,
  \begin{equation}\nonumber
    \begin{split}
      \norm{ \sum_{s=1}^{t-1}\gamma^{-s}\eta_sX_s}_{\Ht_t^{-1}(\theta_t)} &\leq \frac{\sqrt{\lambda c_\mu}}{2 m }+\frac{2 m }{\sqrt{\lambda c_\mu}}\log\frac{1}{\delta}+\frac{d m }{\sqrt{\lambda c_\mu}} \log \sbr{\frac{\lambda c_\mu \gamma^{-2(t-1)} + \frac{L^2k_\mu\sum_{s=1}^{t-1} \gamma^{-2s} }{d}}{\lambda c_\mu \gamma^{-2(t-1)}}}+\frac{2 m }{\sqrt{\lambda c_\mu}} d \log (2)\\
      &\leq \frac{\sqrt{\lambda c_\mu}}{2 m }+\frac{2 m }{\sqrt{\lambda c_\mu}}\log\frac{1}{\delta}+\frac{d m }{\sqrt{\lambda c_\mu}} \log \sbr{1+\frac{ L^2k_\mu(1-\gamma^{2t-2})}{\lambda c_\mu d(1-\gamma^{2})}}+\frac{2 m }{\sqrt{\lambda c_\mu}} d \log (2).
    \end{split}
  \end{equation}
  Therefore, we get the upper bound for $F_t$ term.
\end{proof}

\subsection{Proof of~\pref{thm:SCB-regret}}
\label{sec:SCB-regret-proof}

\begin{proof}
  Let $X_t^* \triangleq \argmax_{\x\in \X}\mu(\x^\T \theta_t)$. Due to Lemma~\ref{lemma:GLB-estimation-error} and the fact that $X_t^*,X_t\in \X$, each of the following holds with probability at least $1-\delta$,
  \begin{equation}\nonumber
    \begin{split}
        \forall t\in [T], \mu(X_t^{*\T}\theta_t) \leq{}&  \mu(X_t^{*\T}\thetat_t) +2\sqrt{1+2S}\frac{k_\mu}{\sqrt{c_\mu}}\sbr{ L^2\frac{k_\mu}{\sqrt{\lambda c_\mu}}\sqrt{d}\sum_{p=1}^{t-1}\gamma^{\frac{t-1}{2}}\sqrt{\frac{\gamma^{-p}-1}{1-\gamma}} \norm{\theta_p - \theta_{p+1}}_2+ \betat_{t-1}\|X_t^*\|_{V_{t-1}^{-1}} },\\
        \forall t\in [T], \mu(X_t^{\T}\theta_t) \geq{}&  \mu(X_t^{\T}\thetat_t) -2\sqrt{1+2S}\frac{k_\mu}{\sqrt{c_\mu}}\sbr{ L^2\frac{k_\mu}{\sqrt{\lambda c_\mu}}\sqrt{d}\sum_{p=1}^{t-1}\gamma^{\frac{t-1}{2}}\sqrt{\frac{\gamma^{-p}-1}{1-\gamma}} \norm{\theta_p - \theta_{p+1}}_2+ \betat_{t-1}\|X_t\|_{V_{t-1}^{-1}} }.
    \end{split}
  \end{equation}
  By the union bound, the following holds with probability at least $1-2\delta$: $\forall t\in [T]$
  \begin{equation}\nonumber
  \begin{split}
      {}&\mu(X_t^{*\T}\theta_t) - \mu(X_t^{\T}\theta_t) \\
      \leq{}&  \mu(X_t^{*\T}\thetat_t) -\mu(X_t^{\T}\thetat_t) + 2\sqrt{1+2S}\bigg(\frac{2L^2k_\mu^2}{c_\mu}\sqrt{\frac{d}{\lambda}}\sum_{p=1}^{t-1}\gamma^{\frac{t-1}{2}}\sqrt{\frac{\gamma^{-p}-1}{1-\gamma}} \norm{\theta_p - \theta_{p+1}}_2 \\
      &\qquad\qquad\qquad\qquad\qquad\qquad\qquad\qquad\qquad\qquad+ \frac{k_\mu}{\sqrt{c_\mu}}\sbr{\betat_{t-1}\|X_t^*\|_{V_{t-1}^{-1}}+\betat_{t-1}\|X_t\|_{V_{t-1}^{-1}} }\bigg)\\
      \leq{}& \frac{4\sqrt{1+2S}L^2k_\mu^2}{c_\mu}\sqrt{\frac{d}{\lambda}}\sum_{p=1}^{t-1}\gamma^{\frac{t-1}{2}}\sqrt{\frac{\gamma^{-p}-1}{1-\gamma}} \norm{\theta_p - \theta_{p+1}}_2 + \frac{4\sqrt{1+2S}k_\mu}{\sqrt{c_\mu}}\betat_{t-1}\|X_t\|_{V_{t-1}^{-1}},
  \end{split}
  \end{equation}
  where the last step comes from the arm selection criterion~\eqref{eq:SCB-select-criteria} such that 
  $$\mu(X_t^{*\T} \thetat_t)+ 2\sqrt{1+2S}\frac{k_\mu}{\sqrt{c_\mu}}\betat_{t-1}\|X_t^*\|_{V_{t-1}^{-1}} \leq \mu(X_t^{\T} \thetat_t)+ 2\sqrt{1+2S}\frac{k_\mu}{\sqrt{c_\mu}}\betat_{t-1}\|X_t\|_{V_{t-1}^{-1}}. $$
  Hence, the following dynamic regret bound holds with probability at least $1-2\delta$ and can be divided into two parts,
  \begin{equation}\nonumber
  \begin{split}
      \DReg_T = {}& \sum_{t=1}^T \mu(X_t^{*\T} \theta_t) - \mu(X_t^\T \theta_t)\\
      \leq {}& \underbrace{\frac{4\sqrt{1+2S}L^2k_\mu^2}{c_\mu}\sqrt{\frac{d}{\lambda}}\sum_{t=1}^T\sum_{p=1}^{t-1}\gamma^{\frac{t-1}{2}}\sqrt{\frac{\gamma^{-p}-1}{1-\gamma}} \norm{\theta_p - \theta_{p+1}}_2}_{\bias} + \underbrace{\frac{4\sqrt{1+2S}k_\mu}{\sqrt{c_\mu}}\betat_T\sum_{t=1}^T\|X_t\|_{V_{t-1}^{-1}}}_{\variance} .\\
  \end{split}
  \end{equation}
  where $\betat_t = \frac{\sqrt{\lambda c_\mu}}{2 m }+\frac{2 m }{\sqrt{\lambda c_\mu}}\log\frac{1}{\delta}+\frac{d m }{\sqrt{\lambda c_\mu}} \log \sbr{1+\frac{ L^2k_\mu(1-\gamma^{2t})}{\lambda c_\mu d(1-\gamma^{2})}}+\frac{2 m }{\sqrt{\lambda c_\mu}} d \log (2)+\sqrt{\lambda c_\mu} S$ is the confidence radius.

  Now we derive the upper bound for these two parts separately.

  \tb{Bias Part.} Similar to the proof of inequality~\eqref{eq:LB-regret-bias-bound}, we have 
  \begin{align*}
    \frac{4\sqrt{1+2S}L^2k_\mu^2}{c_\mu}\sqrt{\frac{d}{\lambda}}\sum_{t=1}^T \sum_{p=1}^{t-1}   \gamma^{\frac{t-1}{2}}\sqrt{\frac{\gamma^{-p}-1}{1-\gamma}}\norm{\theta_p -\theta_{p+1}}_2 \leq \frac{8\sqrt{1+2S}L^2k_\mu^2}{c_\mu}\sqrt{\frac{d}{\lambda}}\frac{1}{(1-\gamma)^{\sfrac{3}{2}}}P_T.
  \end{align*}

  \tb{Variance Part.}
  First use the Cauchy-Schwarz inequality, we know that
  \begin{equation}\nonumber
  \begin{split}
    \frac{4\sqrt{1+2S}k_\mu}{\sqrt{c_\mu}}\betat_T\sum_{t=1}^T\|X_t\|_{V_{t-1}^{-1}}
      \leq {}& \frac{4\sqrt{1+2S}k_\mu}{\sqrt{c_\mu}}\betat_T\sqrt{T}\sqrt{\sum_{t=1}^T\|X_t\|_{V_{t-1}^{-1}}^2}.
  \end{split}
  \end{equation}
  Then for the term $\sqrt{\sum_{t=1}^T\|X_t\|_{V_{t-1}^{-1}}^2}$, we can directly use the Lemma~\ref{lemma:potential-lemma} to bound it,
  \begin{equation}\nonumber
  \begin{split}
    \frac{4\sqrt{1+2S}k_\mu}{\sqrt{c_\mu}}\betat_T\sqrt{T}\sqrt{\sum_{t=1}^T\|X_t\|_{V_{t-1}^{-1}}^2} \leq {}& \frac{4\sqrt{1+2S}k_\mu}{\sqrt{c_\mu}}\betat_T\sqrt{2\max\{1,L^2/\lambda\}dT}\sqrt{ T\log\frac{1}{\gamma}+\log\sbr{1+ \frac{L^2}{\lambda d(1-\gamma)}}}.
  \end{split}
  \end{equation}
  Combining the upper bound for the bias and variance parts, and letting $\delta = 1/(2T)$, we have the following regret bound with probability at least $1-1/T$,
  \begin{equation}\nonumber
      \DReg_T \leq \frac{8\sqrt{1+2S}L^2k_\mu^2}{c_\mu}\sqrt{\frac{d}{\lambda}}\frac{1}{(1-\gamma)^{\sfrac{3}{2}}}P_T+ \frac{4\sqrt{1+2S}k_\mu}{\sqrt{c_\mu}}\betat_T\sqrt{2\max\{1,L^2/\lambda\}dT}\sqrt{ T\log\frac{1}{\gamma}+\log\sbr{1+ \frac{L^2}{\lambda d(1-\gamma)}}}.
  \end{equation}
  where $\betat_t = \frac{\sqrt{\lambda c_\mu}}{2 m }+\frac{2 m }{\sqrt{\lambda c_\mu}}\log\sbr{2T}+\frac{d m }{\sqrt{\lambda c_\mu}} \log \sbr{1+\frac{ L^2k_\mu(1-\gamma^{2t})}{\lambda c_\mu d(1-\gamma^{2})}}+\frac{2 m }{\sqrt{\lambda c_\mu}} d \log (2)+\sqrt{\lambda c_\mu} S$. Since that there is a $T \sqrt{\log (1/\gamma)}$ term in the regret bound, which means that we cannot let $\gamma$ close to $0$, so we set $\gamma \geq 1/T$, then we have $\log(1/\gamma) \leq C (1-\gamma)$, where $C = \log T/(1-1/T)$. Then, ignoring logarithmic factors in time horizon $T$, and let $\lambda = d\log(T)/c_\mu$, we finally obtain that,
  \begin{equation}\nonumber
  \begin{split}
      \DReg_T\leq{}& \Ot\sbr{\frac{k_\mu^2}{\sqrt{c_\mu}}\frac{1}{(1-\gamma)^{\sfrac{3}{2}}}P_T + \frac{k_\mu}{\sqrt{c_\mu}}d(1-\gamma)^{\sfrac{1}{2}}T}.
  \end{split}
  \end{equation}
  When $P_T< d/(k_\mu T)$ (which corresponds a small amount of non-stationarity), we simply set $\gamma = 1-1/T$ and achieve an $\Ot(k_\mu c_\mu^{-\sfrac{1}{2}}d\sqrt{T})$ regret bound.  Besides, when coming to the non-degenerated case of $P_T\geq d/(k_\mu T)$, We set the discounted factor optimally as $1-\gamma = \sqrt{k_\mu P_T/(dT)}$ and attain an $\Ot(k_\mu^{\sfrac{5}{4}}c_\mu^{-\sfrac{1}{2}}d^{\sfrac{3}{4}}P_T^{\sfrac{1}{4}}T^{\sfrac{3}{4}})$ regret bound, which completes the proof.
\end{proof}
\section{Piecewise-Stationary SCB}
\label{sec:SCB-PW}
In this section, we study SCB under piecewise-stationary environment and our work is a direct improvement over~\citep{AISTATS'21:SCB-forgetting}. Next, we will first propose our \SCBweightourspw algorithm, and then, present the analysis of the confidence set. Finally, we give the proof of the dynamic regret upper bound.

\subsection{SCB-PW-WeightUCB Algorithm}
\label{sec:SCB-PW-algorithm}
Inspired by \citet{AISTATS'21:optimal-logistic-bandits}, we make a direct improvement over \citet{AISTATS'21:SCB-forgetting}. Just like~\citet{AISTATS'21:SCB-forgetting}, for $D\geq 1$, define $\mathcal{T}(D) = \{1\leq t\leq T, \text{~such that~} \theta_s = \theta_t \text{~for~} t-D\leq s\leq t-1\}$. $t\in\mathcal{T}(D)$ when $t$ is at least $D$ steps away from the previous closest changing point. But the difference is that~\citet{AISTATS'21:SCB-forgetting} considers $D$ as an analysis parameter, and we treat $D$ as a tunable algorithm parameter. Notice that, the $D$ here is \emph{not} a virtual window size, but the algorithm's estimate of how durable the environment is stationary.

\tb{Estimator.} At iteration $t$, we adopt the same maximum likelihood estimator as in the drifting case as defined in \eqref{eq:GLB-estimator}. 

\tb{Confidence Set.} We further construct confidence set for the real $\theta_t$. For $\delta \in (0,1)$, we define,
\begin{equation*}
  \label{eq:SCB-PW-confidence-set}
  \begin{split}
    \C_t(\delta) \teq \bbr{\theta \in \Theta \givenn \|g_t(\theta) - g_t(\thetah_t)\|_{H_t^{-1}(\theta)}\leq \rho_t},
  \end{split}
\end{equation*}
where $\rho_t = \frac{2 L^2S k_\mu}{\sqrt{\lambda c_\mu}} \frac{\gamma^D}{1-\gamma} + \frac{Lm}{\sqrt{\lambda c_\mu}}\frac{\gamma^D}{1-\gamma} + \betabr_t$ and $\betabr_t = \frac{d m }{\sqrt{\lambda c_\mu}}\log \sbr{1 + \frac{ L^2k_\mu(1-\gamma^{2D})}{\lambda c_\mu d(1-\gamma)}}+\frac{\sqrt{\lambda c_\mu}}{2 m }+\frac{2 m }{\sqrt{\lambda c_\mu}}\log\frac{1}{\delta} +\frac{2 m }{\sqrt{\lambda c_\mu}} d \log (2) + \sqrt{\lambda c_\mu}S$.
\begin{Lemma}
  \label{lemma:SCB-PW-confidence-set}
  For any $\delta \in (0,1)$, with probability at least $1-\delta$, we have $\forall t\in \mathcal{T}(D), \theta_t \in \C_t(\delta)$.
  \begin{equation*}
    \begin{split}
      \C_t(\delta) = \bbr{\theta \in \Theta \given \|g_t(\theta) - g_t(\thetah_t)\|_{H_t^{-1}(\theta)}\leq \frac{2 L^2S k_\mu}{\sqrt{\lambda c_\mu}} \frac{\gamma^D}{1-\gamma} + \frac{Lm}{\sqrt{\lambda c_\mu}}\frac{\gamma^D}{1-\gamma} + \betabr_t },
    \end{split}
  \end{equation*}
  where $\betabr_t = \frac{\sqrt{\lambda c_\mu}}{2 m }+\frac{2 m }{\sqrt{\lambda c_\mu}}\log\frac{1}{\delta} +\frac{d m }{\sqrt{\lambda c_\mu}}\log \sbr{1 + \frac{ L^2k_\mu(1-\gamma^{2D})}{\lambda c_\mu d(1-\gamma)}}+\frac{2 m }{\sqrt{\lambda c_\mu}} d \log (2) + \sqrt{\lambda c_\mu}S$.
\end{Lemma}
The proof of Lemma~\ref{lemma:SCB-PW-confidence-set} is presented in Appendix~\ref{sec:SCB-PW-estimation-error-proof}.

\tb{Selection Criteria.}
Algorithms discussed earlier for drifting cases are using bonus-based selection criteria. But here we use a parameter-based selection criterion as follows,
\begin{equation}
\label{eq:SCB-PW-select-criteria}
\begin{split}
  (X_t,\thetat_t) = \argmax_{\x\in\X, \theta\in \C_t(\delta)}\mu(\x^\T\theta).
\end{split}
\end{equation}
The main difference between parameter-based and bonus-based selection criteria is discussed in Section 3.2 of~\citet{AISTATS'21:optimal-logistic-bandits}. The overall algorithm is summarized in Algorithm~\ref{alg:SCB-PW-WeightUCB}.

\begin{algorithm}[!t]
  \caption{SCB-PW-WeightUCB}
  \label{alg:SCB-PW-WeightUCB}
\begin{algorithmic}[1]
\REQUIRE time horizon $T$, discounted factor $\gamma$, confidence $\delta$, regularizer $\lambda$, inverse link function $\mu$, parameters $S$, $L$ and $m$, changing confidence $D$\\
\STATE Set $\thetah_0 = \mathbf{0}$ and compute $k_\mu$ and $c_\mu$
\FOR{$t = 1,2,3,...,T$}
  \STATE Compute $(X_t,\thetat_t) = \argmax_{\x\in\X, \theta\in \C_t(\delta)}\mu(\x^\T\theta)$ 
  \STATE Select $X_t$ and receive the reward $r_t$
  \STATE Compute $\thetah_{t+1}$ according to~\eqref{eq:GLB-estimator}
\ENDFOR
\end{algorithmic}
\end{algorithm}

\subsection{Proof of Lemma~\ref{lemma:SCB-PW-confidence-set}}
\label{sec:SCB-PW-estimation-error-proof}

\begin{proof}
  Based on the model assumption~\eqref{eq:SCB-model-assume}, the function $g_t$~\eqref{eq:GLB-gt} and the $g_t(\thetah_t)$~\eqref{eq:SCB-gt-thetah}, we have,
  \begin{align*}
    g_t(\theta_t)-g_t(\thetah_t)= {}& \lambda c_\mu\theta_t +\sum_{s=1}^{t-1}\gamma^{t-s-1}\mu(X_s^\T\theta_t)X_s-\sum_{s=1}^{t-1}\gamma^{t-s-1}r_sX_s\\
    = {}&\lambda c_\mu\theta_t +\sum_{s=1}^{t-1}\gamma^{t-s-1}\mu(X_s^\T\theta_t)X_s-\sum_{s=1}^{t-1}\gamma^{t-s-1}(\mu(X_s^\T \theta_s) + \eta_s)X_s\\
    = {}&\sum_{s=1}^{t-1}\gamma^{t-s-1}(\mu(X_s^\T\theta_t) - \mu(X_s^\T \theta_s) )X_s + \lambda c_\mu \theta_t -\sum_{s=1}^{t-1}\gamma^{t-s-1}\eta_sX_s.
  \end{align*}
  Then,
  \begin{align*}
    \|g_t(\theta_t) - g_t(\thetah_t)\|_{H_t^{-1}(\theta_t)} &= \norm{\sum_{s=1}^{t-1}\gamma^{t-s-1}(\mu(X_s^\T\theta_t) - \mu(X_s^\T \theta_s) )X_s + \lambda c_\mu \theta_t -\sum_{s=1}^{t-1}\gamma^{t-s-1}\eta_sX_s}_{H_t^{-1}(\theta_t)} \\
    &\leq \norm{\sum_{s=1}^{t-1}\gamma^{t-s-1}(\mu(X_s^\T\theta_t) - \mu(X_s^\T \theta_s) )X_s}_{H_t^{-1}(\theta_t)} + \norm{\lambda c_\mu \theta_t -\sum_{s=1}^{t-1}\gamma^{t-s-1}\eta_sX_s}_{H_t^{-1}(\theta_t)}\\
    &\leq \underbrace{\norm{\sum_{s=1}^{t-1}\gamma^{t-s-1}(\mu(X_s^\T\theta_t) - \mu(X_s^\T \theta_s) )X_s}_{H_t^{-1}(\theta_t)}}_{\term{a}} + \underbrace{\norm{\sum_{s=1}^{t-D-1}\gamma^{t-s-1}\eta_sX_s}_{H_t^{-1}(\theta_t)}}_{\term{b}}\\
    &+ \underbrace{\norm{\sum_{s=t-D}^{t-1}\gamma^{t-s-1}\eta_sX_s - \lambda c_\mu \theta_t}_{H_t^{-1}(\theta_t)}}_{\term{c}}.
  \end{align*}
  \tb{Term~(a).}
  Since $t \in \mathcal{T}(D)$, we have
  \begin{align*}
    \norm{\sum_{s=1}^{t-1}\gamma^{t-s-1}(\mu(X_s^\T\theta_t) - \mu(X_s^\T \theta_s) )X_s}_{H_t^{-1}(\theta_t)} &= \norm{\sum_{s=1}^{t-D-1}\gamma^{t-s-1}(\mu(X_s^\T\theta_t) - \mu(X_s^\T \theta_s) )X_s}_{H_t^{-1}(\theta_t)}\\
    &\leq \norm{\sum_{s=1}^{t-D-1}\gamma^{t-s-1}k_\mu X_s^\T(\theta_t-\theta_s)X_s}_{H_t^{-1}(\theta_t)}\\
    &\leq \sum_{s=1}^{t-D-1}\gamma^{t-s-1}k_\mu \|X_s\|_2\|(\theta_t-\theta_s)\|_2\norm{X_s}_{H_t^{-1}(\theta_t)}\\
    &\leq \frac{2 L^2S k_\mu}{\sqrt{\lambda c_\mu}} \frac{\gamma^D}{1-\gamma}. 
  \end{align*}

  \tb{Term~(b).}
  \begin{equation}\nonumber
    \begin{split}
      \norm{\sum_{s=1}^{t-D-1}\gamma^{t-s-1}\eta_sX_s}_{H_t^{-1}(\theta_t)} \leq \sum_{s=1}^{t-D-1}\gamma^{t-s-1} m \norm{X_s}_{H_t^{-1}(\theta_t)}\leq \frac{Lm}{\sqrt{\lambda c_\mu}}\sum_{s=1}^{t-D-1}\gamma^{t-s-1}\leq \frac{Lm}{\sqrt{\lambda c_\mu}}\frac{\gamma^D}{1-\gamma}. \\
    \end{split}
  \end{equation}

  \tb{Term~(c).}
  Let $\Ht_{t-D:t}(\theta) = \lambda c_\mu \gamma^{-2(t-1)}I_d + \sum_{s=t-D}^{t-1}\gamma^{-2s}\dmu(X_s^\T \theta)X_sX_s^\T$
  \begin{equation}\nonumber
    \begin{split}
      \norm{\sum_{s=t-D}^{t-1}\gamma^{t-s-1}\eta_sX_s - \lambda c_\mu \theta_t}_{H_t^{-1}(\theta_t)} &\leq \norm{\sum_{s=t-D}^{t-1}\gamma^{t-s-1}\eta_sX_s}_{H_t^{-1}(\theta_t)} + \sqrt{\lambda c_\mu} S\\
      &\leq \norm{\sum_{s=t-D}^{t-1}\gamma^{-s}\eta_sX_s}_{\Ht_t^{-1}(\theta_t)} + \sqrt{\lambda c_\mu} S\\
      &\leq \norm{\sum_{s=t-D}^{t-1}\gamma^{-s}\eta_sX_s}_{\Ht_{t-D:t}^{-1}(\theta_t)} + \sqrt{\lambda c_\mu} S.
    \end{split}
  \end{equation}
  We already proof that $\gamma^{-2(t-1)}H_t(\theta) \succeq  \Ht_t^{-1}(\theta)$ in~\eqref{eq:H-Ht}, and obviously $\Ht_{t}(\theta) \succeq \Ht_{t-D:t}(\theta)$. Next, we need to bound the term $\|\sum_{s=t-D}^{t-1}\gamma^{-s}\eta_sX_s\|_{\Ht_{t-D:t}^{-1}(\theta_t)} $ using self-normalization bound~\citep[Theorem 3]{AISTATS'21:SCB-forgetting}, restated in Theorem~\ref{thm:self-normalized-weight-SCB} by let  $w_t = \gamma^{-t}, \lambda_{t}= \lambda c_\mu \gamma^{-2t}$, then we have
  \begin{align*}
      &\norm{ \sum_{s=t-D}^{t-1}\gamma^{-s}\eta_sX_s}_{\Ht_{t-D:t}^{-1}(\theta_t)}\\
       \leq& \frac{\sqrt{\lambda c_\mu}}{2 m }+\frac{2 m }{\sqrt{\lambda c_\mu}} \log \sbr{\frac{\det\sbr{\Ht_{t-D:t}}^{1 / 2}}{\delta (\lambda c_\mu \gamma^{-2(t-1)})^{d / 2}}}+\frac{2 m }{\sqrt{\lambda c_\mu}} d \log (2)\\
       \leq& \frac{\sqrt{\lambda c_\mu}}{2 m }+\frac{2 m }{\sqrt{\lambda c_\mu}}\log\frac{1}{\delta} +\frac{d m }{\sqrt{\lambda c_\mu}}\log \sbr{\frac{\lambda c_\mu \gamma^{-2(t-1)} + \frac{L^2k_\mu\sum_{s=t-D}^t \gamma^{-2s}}{d}}{\lambda c_\mu \gamma^{-2(t-1)}}}+\frac{2 m }{\sqrt{\lambda c_\mu}} d \log (2)\\
       \leq& \frac{\sqrt{\lambda c_\mu}}{2 m }+\frac{2 m }{\sqrt{\lambda c_\mu}}\log\frac{1}{\delta} +\frac{d m }{\sqrt{\lambda c_\mu}}\log \sbr{1 + \frac{ L^2k_\mu(1-\gamma^{2D})}{\lambda c_\mu d(1-\gamma)}}+\frac{2 m }{\sqrt{\lambda c_\mu}} d \log (2).
    \end{align*} 

  Let $\betabr_t \teq \frac{\sqrt{\lambda c_\mu}}{2 m }+\frac{2 m }{\sqrt{\lambda c_\mu}}\log\frac{1}{\delta} +\frac{d m }{\sqrt{\lambda c_\mu}}\log \sbr{1 + \frac{ L^2k_\mu(1-\gamma^{2D})}{\lambda c_\mu d(1-\gamma)}}+\frac{2 m }{\sqrt{\lambda c_\mu}} d \log (2) + \sqrt{\lambda c_\mu}S$
  , finally we have,
  \begin{equation}\nonumber
    \begin{split}
      \|g_t(\theta_t) - g_t(\thetah_t)\|_{H_t^{-1}(\theta_t)} \leq \frac{2 L^2S k_\mu}{\sqrt{\lambda c_\mu}} \frac{\gamma^D}{1-\gamma} + \frac{Lm}{\sqrt{\lambda c_\mu}}\frac{\gamma^D}{1-\gamma} + \betabr_t,
    \end{split}
  \end{equation}
  which completes the proof. 
\end{proof}

\subsection{Proof of~\pref{thm:SCB-PW-regret}}
\label{sec:SCB-PW-regret-proof}

\begin{proof}
  Let $R_t = \mu(X_t^{*\T} \theta_t) - \mu(X_t^\T \theta_t)$
  \begin{equation}\nonumber
    \begin{split}
        \DReg_T = \sum_{t=1}^T R_t = \sum_{t\notin \mathcal{T}(D)} R_t + \sum_{t\in \mathcal{T}(D)} R_t = \Gamma_T D + \sum_{t\in \mathcal{T}(D)} R_t.\\
    \end{split}
  \end{equation}
  For $t\in \mathcal{T}(D)$, by selection criterion~\eqref{eq:SCB-PW-select-criteria}, 
  \begin{equation}\nonumber
  \begin{split}
    R_t = {}&\mu(X_t^{*\T}\theta_t) - \mu(X_t^{\T}\theta_t) \\
    \leq {}&\mu(X_t^{\T}\thetat_t) - \mu(X_t^{\T}\thetah_t) + \mu(X_t^{\T}\thetah_t) -\mu(X_t^{\T}\theta_t) \\
    \leq {}&\alpha(X_t,\thetat_t,\thetah_t)\abs{X_t^\T\sbr{\thetat_t - \thetah_t}} + \alpha(X_t,\theta_t,\thetah_t)\abs{X_t^\T\sbr{\theta_t - \thetah_t}}\\
    \leq {}&\sqrt{1+2S}\bigg(\alpha(X_t,\thetat_t,\thetah_t)\norm{X_t}_{G_t^{-1}(\thetat_t,\thetah_t)} \norm{g_t(\thetat_t) - g_t(\thetah_t)}_{H_t^{-1}(\thetat_t)} \\
    & \qquad\qquad\qquad\qquad\qquad\qquad\qquad\qquad+\alpha(X_t,\theta_t,\thetah_t)\norm{X_t}_{G_t^{-1}(\theta_t,\thetah_t)} \norm{g_t(\theta_t) - g_t(\thetah_t)}_{H_t^{-1}(\theta_t)}\bigg).
  \end{split}
  \end{equation}
  where $\alpha(\x, \theta_1, \theta_2) = \int_{0}^1 \dmu(v\x^\T\theta_2+(1-v)x^\T\theta_1)\diff{v}$, and the last second inequality comes from the mean value theorem $
  \mu(\x^\T\theta_1) - \mu(\x^\T \theta_2) = \alpha(\x, \theta_1, \theta_2)(\x^\T\theta_1 - \x^\T \theta_2)$. Since that $\thetat_t \in \C_t(\delta)$ and with probability at least $1-\delta$, $\forall t\in [T], \theta_t \in \C_t(\delta)$, and by union bound, the following dynamic regret bound hold with probability at least $1-\delta$,
  \begin{equation}\nonumber
    \begin{split}
      \sum_{t\in \mathcal{T}(D)} R_t 
      \leq {}&\sum_{t\in \mathcal{T}(D)}\sqrt{1+2S} \sbr{\alpha(X_t,\thetat_t,\thetah_t)\norm{X_t}_{G_t^{-1}(\thetat_t,\thetah_t)}\rho_t + \alpha(X_t,\theta_t,\thetah_t)\norm{X_t}_{G_t^{-1}(\theta_t,\thetah_t)} \rho_t}\\
      \leq {}&\sqrt{1+2S} \rho_T  \sbr{\sum_{t\in \mathcal{T}(D)}\alpha(X_t,\thetat_t,\thetah_t)\norm{X_t}_{G_t^{-1}(\thetat_t,\thetah_t)}+ \sum_{t\in \mathcal{T}(D)}\alpha(X_t,\theta_t,\thetah_t)\norm{X_t}_{G_t^{-1}(\theta_t,\thetah_t)} }.\\
    \end{split}
  \end{equation}
    Now we try to derive the upper bound for term $\sum_{t\in \mathcal{T}(D)}\alpha(X_t,\thetat_t,\thetah_t)\norm{X_t}_{G_t^{-1}(\thetat_t,\thetah_t)}$. 
    
  Based on the definition of $g_t$~\eqref{eq:GLB-gt}, we have
  \begin{equation}
    \begin{split}
        g_t(\theta_1) - g_t(\theta_2) &=\lambda c_\mu (\theta_1 - \theta_2) + \sum_{s=1}^{t-1}\gamma^{t-s-1}(\mu(X_s^\T \theta_1)-\mu(X_s^\T \theta_2))X_s\\
        &= \lambda c_\mu (\theta_1 - \theta_2) + \sum_{s=1}^{t-1}\gamma^{t-s-1}\alpha(X_s, \theta_1, \theta_2) X_s^\T X_s (\theta_1-\theta_2)\\
        &= \sbr{\lambda c_\mu + \sum_{s=1}^{t-1}\gamma^{t-s-1}\alpha(X_s, \theta_1, \theta_2) X_s^\T X_s}(\theta_1-\theta_2).
    \end{split}
  \end{equation}
  Then based on the definition of $G_t$~\eqref{eq:SCB-gt-mvt}, we know $G_t(\theta_1, \theta_2) = \lambda c_\mu + \sum_{s=1}^{t-1}\gamma^{t-s-1}\alpha(X_s, \theta_1, \theta_2) X_s^\T X_s.$ which means $G_t(\thetat_t,\thetah_t) = \lambda c_\mu I_d + \sum_{s=1}^{t-1}\gamma^{t-s-1}\alpha(X_s,\thetat_t,\thetah_t)X_sX_s^\T$, if we let $\Xt_s = \sqrt{\alpha(X_s,\thetat_t,\thetah_t)}X_s$, then
    \begin{align}\nonumber
        \sum_{t\in \mathcal{T}(D)} \alpha(X_t,\thetat_t,\thetah_t)\norm{X_t}_{G_t^{-1}(\thetat_t,\thetah_t)} &\leq \sqrt{\sum_{t=1}^T\alpha(X_t,\thetat_t,\thetah_t)}\sqrt{\sum_{t=1}^T \alpha(X_t,\thetat_t,\thetah_t) \norm{X_t}_{G_t^{-1}(\thetat_t,\thetah_t)}^2 } \\
        &\leq \sqrt{k_\mu T}\sqrt{\sum_{t=1}^T \norm{\Xt_t}_{G_t^{-1}(\thetat_t,\thetah_t)}^2 }.
      \end{align}

  Then for the term $\sqrt{\sum_{t=1}^T \|\Xt_t\|_{G_t^{-1}(\thetat_t,\thetah_t)}^2 }$, we can directly use the Lemma~\ref{lemma:potential-lemma} to bound it,
  \begin{equation}\nonumber
  \begin{split}
    \sqrt{k_\mu T}\sqrt{\sum_{t=1}^T \norm{\Xt_t}_{G_t^{-1}(\thetat_t,\thetah_t)}^2 }\leq {}& \sqrt{2k_\mu\max\{1,L^2k_\mu/(\lambda c_\mu)\}dT}\sqrt{ T\log\frac{1}{\gamma}+\log\sbr{1+ \frac{L^2k_\mu}{\lambda c_\mu d(1-\gamma)}}}.
  \end{split}
  \end{equation}
  We can bound term $\sum_{t\in \mathcal{T}(D)}\alpha(X_t,\theta_t,\thetah_t)\norm{X_t}_{G_t^{-1}(\theta_t,\thetah_t)} $ in the same way and get, 
  \begin{equation}\nonumber
    \begin{split}
      \sum_{t\in \mathcal{T}(D)}\alpha(X_t,\theta_t,\thetah_t)\norm{X_t}_{G_t^{-1}(\theta_t,\thetah_t)} \leq {}& \sqrt{2k_\mu\max\{1,L^2k_\mu/(\lambda c_\mu)\}dT}\sqrt{ T\log\frac{1}{\gamma}+\log\sbr{1+ \frac{L^2k_\mu}{\lambda c_\mu d(1-\gamma)}}}.
    \end{split}
    \end{equation}

  Combine these two bound and let $\delta = 1/T$, we have the following regret bound with probability at least $1-1/T$,
  \begin{equation}\nonumber
      \DReg_T \leq \Gamma_T D + 2\sqrt{1+2S} \rho_T \sqrt{2k_\mu\max\{1,L^2k_\mu/(\lambda c_\mu)\}dT}\sqrt{ T\log\frac{1}{\gamma}+\log\sbr{1+ \frac{L^2k_\mu}{\lambda c_\mu d(1-\gamma)}}},
  \end{equation}
  where $\rho_t = \frac{2 L^2S k_\mu}{\sqrt{\lambda c_\mu}} \frac{\gamma^D}{1-\gamma} + \frac{Lm}{\sqrt{\lambda c_\mu}}\frac{\gamma^D}{1-\gamma} + \betabr_t$ and $\betabr_t = \frac{d m }{\sqrt{\lambda c_\mu}}\log \sbr{1 + \frac{ L^2k_\mu(1-\gamma^{2D})}{\lambda c_\mu d(1-\gamma)}}+\frac{\sqrt{\lambda c_\mu}}{2 m }+\frac{2 m }{\sqrt{\lambda c_\mu}}\log\sbr{T} +\frac{2 m }{\sqrt{\lambda c_\mu}} d \log (2) + \sqrt{\lambda c_\mu}S$.  Since that there is a $T \sqrt{\log (1/\gamma)}$ term in the regret bound, which means that we cannot let $\gamma$ close to $0$, so we set $\gamma \geq 1/2$, then we have $\log(1/\gamma) \leq 2\log(2) (1-\gamma)$. Then, we set $D = \log(T)/\log(1/\gamma)$, noticing that $0< 1/\gamma -1<1$ and using $\log(1+x)\geq x/2$ for $0<x<1$, we have 
  \begin{equation}\nonumber
    \log\frac{1}{\gamma} = \log(1+1/\gamma -1) \geq \frac{1-\gamma}{2\gamma}.
\end{equation}
  Therefore, we have $D\leq \frac{2\gamma\log(T)}{1-\gamma}$. Then, ignoring logarithmic factors in time horizon $T$, and let $\lambda = d\log(T)/c_\mu$, we finally obtain that,
  \begin{equation}\nonumber
  \begin{split}
      \DReg_T\leq{}& \Ot\sbr{\frac{1}{1-\gamma}\Gamma_T + \sbr{\frac{1}{\sqrt{d}}\frac{1}{1-\gamma}\frac{1}{T}+\sqrt{d}}\sqrt{d(1-\gamma)}T}\\
      \leq{}& \Ot\sbr{\frac{1}{1-\gamma}\Gamma_T + \frac{1}{\sqrt{1-\gamma}} + d\sqrt{(1-\gamma)}T}.
  \end{split}
  \end{equation}
  When $\Gamma_T< d/\sqrt{T}$ (which corresponds a small amount of non-stationarity), we simply set $\gamma = 1-1/T$ and achieve an $\Ot(d\sqrt{T})$ regret bound.  Besides, when coming to the non-degenerated case of $\Gamma_T > d/\sqrt{T}$, We set the discounted factor optimally as $1-\gamma = \sbr{\Gamma_T/(dT)}^{\sfrac{2}{3}}$ and attain an $\Ot(d^{\sfrac{2}{3}}\Gamma_T^{\sfrac{1}{3}}T^{\sfrac{2}{3}})$ dynamic regret bound, which completes the proof.
\end{proof}
\section{Thechnical Lemmas}
\label{sec:thechnical-lemmas}
In this section, we provide several useful lemmas, mainly about weighted version self-normalized concentration, weighted version potential lemma and some derivatives of self-concordant property.

\subsection{Weighted Version Self-normalized Concentration}

\begin{Thm}[Weighted Version Self-Normalized Bound for Vector-Valued Martingales~{\citep[Theorem 1]{NIPS'19:weighted-LB}}]
  \label{thm:self-normalized-weight-LB}
  Let $\{\F_t\}_{t=0}^\infty$ be a filtration, $\{\eta_t\}_{t=0}^\infty$ be a real-valued stochastic process such that $\eta_t$ is $\F_t$-measurable and $\eta_t$ is conditionally $R$-sub-Gaussian for some $R\geq 0$, such that $$\forall\lambda \in \R, \E\mbr{\exp(\lambda \eta_t) \given X_{1:t}, \eta_{1:t-1}}\leq \exp\sbr{\frac{\lambda^2 R^2}{2}}.$$
  Let $\{X_t\}_{t=1}^\infty$ be an $\R^d$-valued stochastic process such that $X_t$ is $\F_{t-1}$-measurable. For any  $t\geq 0$, define
  $$\Vt_t = \mu_t I_d +\sum_{s=1}^t w_s^2 X_sX_s^\T,\quad\quad S_t = \sum_{s=1}^t w_s \eta_sX_s.$$
   where $\forall s \geq 0, t\geq 0, w_s, \mu_t > 0$. Then, for any $\delta > 0$, with probability at least $1-\delta$, we have
  $$\forall t\geq 0, \|S_t\|_{\Vt_{t}^{-1}} \leq R \sqrt{2 \log \frac{1}{\delta}+d \log \sbr{1+\frac{L^{2} \sum_{s=1}^{t} w_{s}^{2}}{d \mu_{t}}}}.$$
\end{Thm}

\begin{Thm}[{{Theorem 3 of \citet{AISTATS'21:SCB-forgetting}}}]
  \label{thm:self-normalized-weight-SCB}
  Let $t$ be a fixed time instant. Let $\bbr{\F_t}_{t=0}^\infty$ be a filtration. Let $\bbr{X_t}_{t=0}^\infty$ be a stochastic process on $\R^d$ such that $X_t$ is $\F_{t-1}$ measurable and $\norm{X_t}_2 \leq 1$. Let $\bbr{\eta_t}_{t=0}^\infty$ be a martingale difference sequence such that $\eta_{t}$ is $\F_{t}$ measurable. Assume that the weights are non-decreasing, strictly positive and the time horizon is known. Furthermore, assume that conditionally on $\F_t$ we have $\abs{\eta_{t}} \leq m$ a.s. Let $\bbr{\lambda_t}_{u=0}^\infty$ be a deterministic sequence of regularization terms and denote $\sigma_t^2=\mathbb{E}\mbr{\eta_{t}^2 \given \F_t}$. Let $\Ht_t=\sum_{s=1}^{t-1} w_s^2 \sigma_s^2 X_s X_s^{\top}+\lambda_{t-1} \mathbf{I}_d$ and $S_t=\sum_{s=1}^{t-1} w_s \eta_{s} X_s$, then for any $\delta \in(0,1]$, with probability at least $1-\delta$,
\begin{equation}\nonumber
\forall t \geq 0, \norm{S_t}_{\Ht_t^{-1}} \leq \frac{\sqrt{\lambda_{t-1}}}{2 m w_{t-1}}+\frac{2 m w_{t-1}}{\sqrt{\lambda_{t-1}}} \log \sbr{\frac{\operatorname{det}\sbr{\Ht_t}^{1 / 2}}{\delta \lambda_{t-1}^{d / 2}}}+\frac{2 m w_{t-1}}{\sqrt{\lambda_{t-1}}} d \log (2).
\end{equation}

\end{Thm}

\subsection{Weighted Version Potential Lemma}

\begin{Lemma}[Weighted Version Potential Lemma~{\citep[Lemma 8]{arXiv'21:faury-driftingGLB}}]
  \label{lemma:potential-lemma} 
  Suppose $V_{t} = \sum_{s=1}^{t}\gamma^{t-s} X_s X_s^\T + \lambda I_d, V_{0} = \lambda I_d, \gamma \in(0,1]$ and $\norm{X_t}_2 \leq L$ for all $t \geq 1$, then the following inequality holds,
  \begin{equation*}
  \begin{split}
      \sum_{t=1}^T\|X_t\|_{V_{t-1}^{-1}}^2\leq {}& 2\max\{1,L^2/\lambda\}d\sbr{ T\log\frac{1}{\gamma}+\log\sbr{1+ \frac{L^2}{\lambda d(1-\gamma)}}}.
  \end{split}
  \end{equation*}
\end{Lemma}

\begin{Lemma}[Determinant inequality]
  \label{lemma:det-inequality} Let $V_{t} = \sum_{s=1}^{t}w_{t,s} X_s X_s^\T + \lambda_t I_d, V_0 = \lambda_0 I_d$. Assume $\|\x\|_2\leq L$ and we have,
  \begin{equation}\nonumber
  \begin{split}
      \det(V_{t})  \leq \sbr{\lambda_t + \frac{L^2\sum_{s=1}^t w_{t,s}}{d}}^d.
  \end{split} 
  \end{equation}
\end{Lemma}

\begin{proof}
  Now we have $V_{t} = \sum_{s=1}^{t}w_{t,s} X_s X_s^\T + \lambda_t I_d$, take the trace on both sides, and get the upper bound of $\mathrm{Tr}(V_{t})$
  \begin{equation}
  \begin{split}
  \label{eq:bound-trace-LB}
      \mathrm{Tr}(V_{t}) =  \mathrm{Tr}( \lambda_t I_d) + \sum_{s=1}^{t}w_{t,s}\mathrm{Tr}\sbr{ X_s X_s^\T} = \lambda_t d + \sum_{s=1}^{t}w_{t,s}\|X_s\|_2^2 \leq \lambda_t  d + L^2\sum_{s=1}^{t}w_{t,s}.
  \end{split} 
  \end{equation}
  Base on the definition of determinant and the upper bound of $\mathrm{Tr}(V_{t})$~\eqref{eq:bound-trace-LB}, we can get the upper bound for $\det(V_{t})$,
  \begin{equation}\nonumber
  \begin{split}
      \det(V_{t}) = {}& \prod_{i =1}^d \lambda_i \leq \sbr{\frac{\sum_{i=1}^d \lambda_i}{d}}^d = \sbr{\frac{\mathrm{Tr}(V_{t})}{d}}^d \leq \sbr{\lambda_t + \frac{L^2\sum_{s=1}^t w_{t,s}}{d}}^d.
  \end{split} 
  \end{equation}
\end{proof}
\subsection{Self-Concordant Properties}
Based on the generalized self-concordant property of the (inverse) link function $\mu(\cdot)$, we have the following lemma, which will be later used to derive Lemma~\ref{lemma:SCB-G-H}.
\begin{Lemma}[{Lemma 9 of \citet{ICML'20:logistic-bandits}}]
  \label{lemma:self-concordant-properties}
  For any $z_1, z_2\in \R$, we have the following inequality:
  \begin{equation*}
    \begin{split}
      \dmu(z_1)\frac{1-\exp(-|z_1-z_2|)}{|z_1-z_2|} \leq \int_{0}^1 \dmu(z_1+v(z_2-z_1))\diff{v} \leq \dmu(z_1)\frac{\exp(|z_1-z_2|)-1}{|z_1-z_2|}.
    \end{split}
  \end{equation*}
  Furthermore, $\int_{0}^1 \dmu(z_1+v(z_2-z_1))\diff{v} \geq \dmu(z_1)(1+|z_1-z_2|)^{-1}$.
\end{Lemma}
The following lemma provides a weighted version of Lemma 10 of~{\citet{ICML'20:logistic-bandits}} which can be easily proven.
\begin{Lemma}
  \label{lemma:SCB-G-H}
  With $G_t$ defined in~\eqref{eq:SCB-gt-mvt} and $H_t$ defined in~\eqref{eq:SCB-H}, the following inequalities hold
  \begin{equation*}
    \begin{split}
      \forall \theta_1,\theta_2 \in \Theta,\quad G_t(\theta_1,\theta_2) \geq (1+2S)^{-1}H_t(\theta_1),\quad G_t(\theta_1,\theta_2) \geq (1+2S)^{-1}H_t(\theta_2).
    \end{split}
  \end{equation*}
\end{Lemma}

\vfill
\end{document}